\documentclass{article}
\PassOptionsToPackage{authoryear, round, sort, compress}{natbib}

\usepackage[final]{neurips_2023}

\usepackage[utf8]{inputenc} %
\usepackage[T1]{fontenc}    %
\usepackage[colorlinks=true, linkcolor=red, urlcolor=blue, citecolor=Blue, anchorcolor=blue, backref=page]{hyperref}       %
\renewcommand*{\backref}[1]{}
\renewcommand*{\backrefalt}[4]{%
    \ifcase #1%
          \or [Cited on page~#2.]%
          \else [Cited on pages~#2.]%
    \fi%
    }
\usepackage{booktabs}       %
\usepackage{amsfonts}       %
\usepackage{nicefrac}       %
\usepackage{microtype}      %
\usepackage[dvipsnames]{xcolor}         %
\usepackage{amssymb}
\usepackage{bm}
\usepackage{comment}

\usepackage{natbib} %
\usepackage{mathtools} 

\usepackage{tikz}
\usepackage{subfig}
\usepackage{amsthm}
\usepackage{thmtools, thm-restate}
\usepackage{cleveref}
\usepackage{url}
\usepackage{wrapfig}
\usepackage{mdframed}
\usepackage{notation1}
\usepackage[shortlabels]{enumitem}

\makeatletter
\newcommand\footnoteref[1]{\protected@xdef\@thefnmark{\ref{#1}}\@footnotemark}
\makeatother

\crefname{figure}{Fig.}{Figs.}
\crefname{definition}{Defn.}{Defns.}
\crefname{corollary}{Corollary}{Corollaries}
\crefname{lemma}{Lemma}{Lemmas}
\crefname{proposition}{Prop.}{Props.}
\crefname{theorem}{Thm.}{Thms.}
\crefname{remark}{Remark}{Remarks}
\crefname{principle}{Principle}{Principles}
\crefname{lemma}{Lemma}{Lemmas}
\crefname{table}{Tab.}{Tabs.}
\crefname{section}{\S}{\S\S}
\crefname{subsection}{\S}{\S\S}
\crefname{subsubsection}{\S}{\S\S}
\crefname{assumption}{Asm.}{Asms.}
\crefname{appendix}{App.}{Apps.}
\title{Extending Independent Mechanism Analysis to Manifolds}
\title{Independent Mechanism Analysis on Manifolds}
\title{Independent Mechanism Analysis \\ and the Manifold Hypothesis%
}

\author{
Shubhangi Ghosh\,$^{1,}$\thanks{S.G. mostly worked on this project while at ETH, Z\"urich \& the MPI for Intelligent Systems, T\"ubingen.}
\And Luigi Gresele\,$^{2}$
\And Julius von K\"ugelgen\,$^{2,3}$
\AND Michel Besserve\,$^{2}$ \quad\quad
Bernhard Sch\"olkopf\,$^{2}$\\[.75em] %
$^1$ Columbia University, USA\\
\quad 
$^2$ Max Planck Institute for Intelligent Systems, T\"ubingen, Germany\\
$^3$ University of Cambridge, United Kingdom \\[.25em]
\texttt{shubhangi.ghosh@columbia.edu}\\
\texttt{\{luigi.gresele,jvk,besserve,bs\}@tue.mpg.de}
}

\begin{document}

\maketitle

\begin{abstract}  
    \looseness-1 
    Independent Mechanism Analysis (IMA) seeks to address non-identifiability in nonlinear Independent Component Analysis (ICA) by assuming that the Jacobian of the mixing function has orthogonal columns. 
    As typical in ICA, previous work focused on the case with an equal number of latent components and observed mixtures. %
    Here, we extend IMA to settings with a larger number of mixtures that reside on a manifold embedded in a
    higher-dimensional 
    space---in line with the {\em manifold hypothesis} in representation learning.
    For this setting, we show that IMA still circumvents several non-identifiability issues, suggesting that it can also be a beneficial principle for higher-dimensional observations when the manifold hypothesis holds.
  Further, we prove that
  the IMA principle is approximately satisfied  with high probability (increasing with the number of observed mixtures) when the directions along which the latent components influence the observations are chosen independently at random.
  This provides a new and rigorous statistical interpretation of IMA.
\end{abstract}

\section{Introduction}
\looseness-1 
Nonlinear Independent Component Analysis (ICA) 
provides a principled approach to representation learning~\citep{uni, gresele2020incomplete, khemakhem2020variational}.
It %
postulates that the observed variables are 
nonlinear mixtures of independent latent components, 
and focuses on whether it is possible to reconstruct the latent components from
the mixtures---formalized through the notion of {\em identifiability}. 
When the mixing is nonlinear, the model is provably non-identifiable 
without additional assumptions~\citep{hyvarinen1999nonlinear},
i.e., the latent variables cannot be recovered.
Independent Mechanism Analysis%
~\citep[IMA;][]{Greseleetal21} seeks to address this problem by restricting the class of considered mixing functions. Specifically, IMA postulates that the columns of the Jacobian of the mixing function, 
which describe how each latent component {\em influences} the observed mixtures,
are orthogonal. This can be viewed as %
encoding a non-statistical notion of {independence} among these influences which is
inspired by the principle of Independent Causal Mechanisms (ICM)~\citep{peters2017elements,janzing2010causal}. 

While identifiability of IMA remains an open question, 
\citet{Greseleetal21} showed that IMA can circumvent certain non-identifiability issues arising in nonlinear ICA %
by 
ruling out well-known counterexamples or {\em spurious solutions}~\citep{hyvarinen1999nonlinear, darmois1951construction, locatello2019challenging}. \citet{buchholz2022function} then proved that IMA is, in fact, {\em locally identifiable}. Further, \citet{reizinger2022embrace} showed that IMA may also provide a way to explain the empirical success of Variational Autoencoders~\citep[VAEs;][]{kingma2014autoencoding} in disentangled representation learning.
\looseness-1 
However, all the aforementioned 
works on IMA consider a setting with an equal number of latent components and observed mixtures, the one most typically studied in ICA~\citep{ICAbook, hyvarinen2023identifiability}.\footnote{For {linear} ICA, identifiability has also been established for the case of more mixtures than sources~\citep{eriksson2004identifiability}. For  nonlinear ICA,  \citet{khemakhem2020variational} extended existing identifiability results relying on additional supervision in the form of an auxiliary variable~\citep{hyvarinen2016unsupervised, hyvarinen2017nonlinear, hyvarinen2019nonlinear, gresele2020incomplete, halva2020hidden} to the high-dimensional
observation setting.%
}
As a result, they do not directly apply to
cases in which the observed data is high-dimensional and the latents low-dimensional, as is often the case in representation learning---e.g., for images or biomedical data. 
In this work, we address this shortcoming of previous theory and generalize IMA to higher-dimensional observations.
In particular, we adopt the {\em manifold hypothesis}~\citep{becker1992self, vincent2002manifold, bengio2013representation} which posits that {many high-dimensional data sets that occur in the real world actually lie along low-dimensional manifolds inside that high-dimensional space}~\citep{cayton2005algorithms, fefferman2016testing}.\footnote{In this work, we mostly assume that observations lie exactly {\em on} a low-dimensional manifold, not {\em close to it}.} 
In this spirit, we extend the analysis of IMA to the setting in which observations lie on a low-dimensional Riemannian manifold, with dimension equal to that of the latent space, embedded in a higher-dimensional observation space.

We show that IMA still helps circumvent non-identifiability issues in this scenario, in the sense that it rules out several kinds of spurious solutions when the generative model satisfies the manifold hypothesis. This suggests that IMA may also be
useful for more realistic representation learning settings involving dimensionality reduction. This insight is consistent with work by
\citet{cunningham2022principal} which also provides empirical evidence 
illustrating
the benefits of an
orthogonality constraint akin to IMA
for unsupervised representation learning with dimensionality reduction---albeit from a different perspective than the one based on nonlinear ICA and identifiability which we adopt here.

\looseness-1
According to~\citet{Greseleetal21}, IMA
can intuitively be  interpreted %
as 
{%
``Nature choosing the direction of the
influence of each source component in the observation space independently and from an isotropic
prior''}. 
Based on the manifold hypothesis, we provide a quantitative argument that formalizes
this statement: when the observations lie on a low-dimensional manifold in the higher-dimensional ambient space, we show that the IMA principle is approximately satisfied with high probability
if the influence directions are sampled independently and isotropically in the high-dimensional space,
with increasing probability as its dimensionality grows. The argument is based on a concentration inequality---{\em Levy's Lemma}~\citep[see, e.g.,][Lemma 1]{janzing2010telling}---and relies on a construction which generates 
smoothened piecewise-affine functions. These functions also play an important role in the theoretical analysis of deep neural networks \citep[e.g.,][]{montufar204number}.
Our work thus shows that, under the manifold hypothesis, the IMA principle can be considered the consequence of a {\em genericity} assumption on the data generating process~\citep{freeman1994generic, janzing2010telling,besserve2018group}. 

\textbf{Structure and Main Contributions:}
\begin{itemize}[itemsep=0em,topsep=0em,leftmargin=0.75em]
    \item \cref{sec:background} briefly reviews independent component analysis and Independent Mechanism Analysis (IMA).\looseness=-1
    \item In~\cref{sec:man_hyp}, we introduce an extension of the
    IMA principle under the {\em manifold hypothesis}.
    \item In~\cref{sec:spurious} we then show that certain common counterexamples to identifiability %
    are ruled out by our extension of IMA to manifolds. 
    \item In~\cref{sec:genericity}, we show that,
     when the manifold hypothesis holds,
    the IMA principle follows from a {\em genericity} assumption on the data-generating process.
\end{itemize}

\section{Background}
\label{sec:background}
Independent Comonent Analysis (ICA)~\citep{comon1994independent, ICAbook} %
assumes a data-generating process
where {\em  observed mixtures}
\(\mathbf{x} \in \mathbb{R}^d\) are generated by a smooth and invertible {\em mixing function} \(\mathbf{f}: \mathbb{R}^d \to \mathbb{R}^d\) belonging to a function class $\Fcal$, which takes as input a vector \(\mathbf{s} \in \mathbb{R}^d\) sampled from a distribution with independent components, i.e., %
\begin{equation}
    \label{eq:dgp}
    \mathbf{x} = \mathbf{f}(\mathbf{s}), \quad \quad \mathbf{s} \sim p(\mathbf{s})=\prod^d_{i=1} p_i(s_i)\,.
\end{equation}

\looseness-1 
{\em Can we recover the independent components based only on the observed mixtures?} Unfortunately, when \(\mathbf{f}\) is nonlinear, the model is non-identifiable without additional constraints~\citep{darmois1951construction, hyvarinen1999nonlinear}---i.e., the independent components cannot be unambiguously recovered from the observed mixtures. This can be shown through suitable {\em spurious solutions}, which transform the  mixtures $\xb$ into independent components which {\em may themselves be mixtures of the true ones}.
\begin{definition}[Darmois construction~\citep{darmois1951construction, hyvarinen1999nonlinear}]
\label{def:darmois}
The Darmois construction $\mathbf{g^D}: \mathbb{R}^{\textrm{d}} \rightarrow (0,1)^{d}$ transforms a given distribution to the uniform distribution by recursively applying the conditional Cumulative Distribution Function (CDF) transform,
\begin{equation}
    \label{eq:darmois}
    g_i^{\textrm{D}}(\mathbf{x}_{1:i}) = \int_{-\infty}^{x_i} p(x'_i|\mathbf{x}_{1:i-1})dx'_i\,.
\end{equation}
\end{definition}

\begin{definition}[Rotated-Gaussian measure preserving Automorphism (MPA)~\citep{locatello2019challenging, khemakhem2020variational}]
\label{def:rgmpa}
The rotated-Gaussian MPA transforms a given density into a Gaussian density by the CDF-transform, applies as orthonormal rotation, and inverts the preceding transformation. Let \(\Fb_\sb(\sb)\) and \(\bm\Phi(\zb)\) denote the CDFs of the latent source distribution and the multivariate Gaussian distribution respectively. For an orthogonal matrix, $\Rb\in O(d)$, the ``rotated-Gaussian'' MPA $\ab^{\Rb}(p_\sb)$ is defined as,
\begin{equation}
\label{eq:measure_preserving_automorphism_Gaussian}
\textstyle
    \ab^{\Rb}(p_\sb) =\Fb_\sb^{-1} \circ \bm\Phi \circ \Rb \circ \bm\Phi^{-1} \circ \Fb_\sb\,.
\end{equation}%
\end{definition}

Towards the goal of achieving identifiability, Independent Mechanism Analysis (IMA)~\citep{Greseleetal21} 
constraints the mixing function class $\Fcal$. IMA postulates that the influence directions of individual latent sources in the mixing process, given by \(\frac{\partial \mathbf{f}}{\partial s_i}\), are orthogonal---i.e.,
\begin{equation}
    \label{eq:ima-old}
     \log |\mathbf{J_{f}}(\mathbf{s})| = \sum_{i=1}^{d} \log \norm{\frac{\partial \mathbf{f}}{\partial s_i} (\mathbf{s})}\,.
\end{equation}   
While identifiability of IMA is still an open question, it provably rules out both the Darmois construction %
and the rotated-Gaussian MPA%
~\citep{Greseleetal21}; empirically, IMA allows recovery of the ground truth sources when the data-generating process satisfies the IMA principle~\citep{Greseleetal21}, as well as under mild model misspecification~\citep{sliwa2022probing}. Moreover, IMA was proved to be {\em locally identifiable}~\citep{buchholz2022function}.

\section{IMA under the Manifold Hypothesis}
\label{sec:man_hyp}
\looseness-1 In the following, we revisit the generative process in~\eqref{eq:dgp} as follows: the observed mixtures lie on a \(d\)-dimensional Riemannian manifold (a smooth manifold with a \(d\)-dimensional tangent space), \(\XX \subseteq \RR^m, m \ge d\), which is embedded in the \(m\)-dimensional Euclidean space. 
The mixing function \({\fb: \RR^d \to \XX}\) is a \emph{diffeomorphism}\footnote{A diffeomorphism is an invertible function between two manifolds such that both the function and its inverse are continuously differentiable.} from the latent space to the observation manifold:
the observations therefore lie on a low-dimensional manifold within the high-dimensional ambient space, in line with the {\em manifold hypothesis}~\citep{becker1992self, vincent2002manifold, bengio2013representation}. To study IMA in this setting, the main definitions in~\citep{Greseleetal21} need to be adapted,
since they were originally tailored to the setting where latent and observation spaces have the same dimension.

We start by extending the IMA principle~\citep[Principle 4.1]{Greseleetal21}. We say that when the manifold hypothesis holds, the IMA principle implies the following equality:
\begin{align}
    \label{ima:condition}
    \sum^d_{i=1} \text{log}\left \|\frac{\partial \fb}{\partial s_i}(\mathbf{s}) \right \| = \frac{1}{2}\text{log}\left| \Jb^{\top}_\fb(\mathbf{s})\Jb_\fb(\mathbf{s})\right| \quad \forall \;\sb \in \RR^d,
\end{align}
where \(\mathbf{J_{f}}(\mathbf{s})\) is the Jacobian of \(\fb\) at \(\sb\). Equation~\eqref{ima:condition} therefore states that the area element on the Riemannian manifold \(\XX\) at \(\xb = \fb(\sb)\), given by \(\sqrt{\left |\Jb^{\top}_{\fb}(\sb)\Jb_{\fb}(\sb) \right |}\), equals the product of the norms of the {\em influences} \(\|\frac{\partial \fb}{\partial s_i}\|\) that span that element. It is therefore an orthogonality condition, similar to the one expressed in~\cref{eq:ima-old} for $m=d$: note however that~\cref{ima:condition} is meaningful for any $m \geq d$. %

Based on~\eqref{ima:condition}, we redefine the {\em local IMA contrast}~\citep[Def. 4.2]{Greseleetal21}, which quantifies the violation of the IMA principle at a given point \(\xb = \fb(\sb)\), and state two of its useful properties.
\begin{definition}[Local IMA contrast] 
The local IMA contrast, \(c_{\textsc{ima}}(\fb, \mathbf{s})\), of \(\fb\) at point \(\mathbf{s}\), is defined as
\begin{align}
    c_{\textsc{ima}}(\fb, \mathbf{s}) = \sum^d_{i=1} \log\left \|\frac{\partial \fb}{\partial s_i}(\mathbf{s}) \right \| - \frac{1}{2}\log\left| \Jb^{\top}_\fb(\mathbf{s})\Jb_\fb(\mathbf{s})\right|. \label{contrast}
\end{align}
\end{definition}

\begin{restatable}{proposition}{cima}
  The local IMA contrast satisfies the following properties:
\begin{enumerate}[(i)]
    \item \(c_{\textsc{ima}}(\fb, \mathbf{s}) \ge 0\) with equality iff. all columns \(\frac{\partial \fb}{\partial u_i}(\mathbf{s})\) of \(\Jb_\fb(\mathbf{s})\) are orthogonal.
    \item \(c_{\textsc{ima}}(\fb, \mathbf{s})\) is invariant to left multiplication of \(\Jb_\fb(\mathbf{s})\) by an orthogonal matrix and to right multiplication by permutation and diagonal matrices.
\end{enumerate}\label{prop:cima}
\end{restatable}

Property {\em (i)} is a geometric condition: given the vectors that span a parallelepiped, the largest volume is obtained when the spanning vectors are orthogonal. Property {\em (ii)} states that permutation and rescaling of the latent factors, or any orthonormal basis transformation applied to the columns of the Jacobian \(\Jb_\fb(\sb)\), do not affect the value of $c_{\textsc{ima}}$. 
Next, we redefine the {\em global IMA contrast}~\citep[Def. 4.5]{Greseleetal21}.

\begin{definition}[Global IMA contrast]
The global IMA contrast, \(C_{\textsc{ima}}\), is defined as the expected value of the local IMA contrast with respect to the source distribution, \(p_\sb\).
\begin{equation}
    C_{\textsc{ima}}(\fb, p_{\mathbf{s}}) = \mathbb{E}_{\mathbf{s} \sim p_{\mathbf{s}}}[c_{\textsc{ima}}(\fb, \mathbf{s})] = \int c_{\textsc{ima}}(\fb, \mathbf{s})p_{\mathbf{s}}(\mathbf{s})d\mathbf{s}. \label{globalima}
\end{equation}    
\end{definition}

\begin{restatable}{proposition}{Cima}
\label{global-ima-prop}
The global IMA contrast (\(C_{\textsc{ima}}(\fb, p_{\mathbf{s}})\) satisfies:
\begin{enumerate}
    \item \(C_{\textsc{ima}}(\fb, p_{\mathbf{s}}) \ge 0\) with equality iff. \(\Jb_{\fb}(\mathbf{s})\) has orthogonal columns almost surely wrt \(\pb_{\mathbf{s}}\).
    \item \(C_{\textsc{ima}}(\fb, p_{\mathbf{s}}) = C_{\textsc{ima}}(\Tilde{\fb}, p_{\mathbf{\Tilde{s}}})\) for any \(\Tilde{\fb} = \fb \circ \hb^{-1} \circ \Pb^{-1}\) and \(\mathbf{\Tilde{s}} = \Pb\hb(\mathbf{s})\) where \(\Pb \in \mathbb{R}^{d \times d}\) is a permutation matrix and \(\hb(\mathbf{s}) = (h_1(s_1), h_2(s_2), ..., h_d(s_d))\) is an invertible element-wise function.     
\end{enumerate}
\end{restatable}

Property {\em (i)} states that $C_\textsc{ima}$ can be used to verify whether the IMA condition holds almost surely with respect to the latent distribution, \(p_\sb\). Property {\em (ii)} states that the IMA constrast for high-dimensional observations is blind to permutation and element-wise transformation of the sources.

\section{Ruling out ``spurious solutions'' under the Manifold Hypothesis}
\label{sec:spurious}
Measure preserving automorphisms (MPAs) applied in the latent space can be used to construct spurious solutions, by composition with the true mixing function~\citep{xi2023indeterminacy}. Below, we show that the global IMA contrast defined above %
rules out spurious solutions based on different MPAs.
All proofs of the results in this section can be found in~\cref{app:b}.

\paragraph{Gaussian-rotated MPA,~\cref{def:rgmpa}.} We prove that IMA rules out the MPA in~\cref{def:rgmpa} for the case in which the mixing function is a \emph{conformal}\footnote{For a formal definition of a \emph{conformal} map, refer to~\cref{app:b}} (angle-preserving) map---a special case of the IMA function class which, for the $m=d$ case, was proved to be identifiable in~\citep{buchholz2022function}.

\begin{restatable}{theorem}{rgmpa}
\label{rgmpa-theorem}
Consider \((\fb, p_{\sb})\) such that \(C_\IMA(\fb, p_\sb) = 0\), and moreover \(\fb: \RR^d \to \XX\) is a conformal map. Given \(\Rb \in \Ob(n)\), assume additionally that \(\exists\) at least one non-Gaussian \(s_i\) whose associated canonical basis vector \(\eb_i\) is not transformed by \(\Rb^{-1} = \Rb^{\top}\) into another canonical basis vector \(\eb_j\). Then, \(C_{\IMA}(\fb \circ a^{\Rb}(p_{\sb}), p_{\sb}) > 0\). 
\end{restatable}

\paragraph{Measure preserving automorphism based on the Darmois construction} The Darmois construction (\cref{def:darmois}) does not directly yield a spurious solution when the dimension of the observed space does not match one of the latent space.\footnote{This is because the CDF transform cannot define a map between a distribution on a higher-dimensional ambient space (observations) to the uniform distribution on a lower-dimensional space (latent sources).} Instead, we construct a spurious solution by applying the Darmois construction to an orthonormal rotation of the latent distribution. We show that the IMA contrast defined in our work can distinguish between such a counterexample and the ground truth (up to tolerable ambiguities like permutation and element-wise transformations). 
\begin{restatable}{theorem}{darmois}
\label{thm:darmois}
Let \((\fb, p_{\sb}) \in \Mcal_{\IMA}\) where \(\fb: \RR^d \to \XX\) is a bijective map and the sources \(\sb\) are such that at most one factor \(s_i\) is Gaussian. Further, we assume that \(\fb\) is a conformal map. Consider an orthonormal transformation \(\Ob \in \RR^{d \times d}\) applied on \(\sb\), \(\Tilde{\xb} = \Ob\sb \in \RR^d\). We further consider that the orthonormal transfomation given by \(\Ob\) is not trivial, i.e. it does not correspond to a permutation or an element-wise scaling. The observations \(\xb \in \XX\) and the transformed variables \(\Tilde{\xb}\) have a bijective relationship. Then any Darmois solution \((\Tilde{\fb}^D, p_\ub)\) based on applying \(\gb^D\) to \(\Tilde{\xb}\) satisfies \(C_{\IMA}(\Tilde{\fb}^D, p_\ub) > 0\). Here, \(\Tilde{\fb}^D = \fb \circ \Ob^{\top} \circ {\gb^D}^{-1}\). Thus a solution satisfying \(C_{\IMA}(\fb, p_{\sb})\) can be distinguished from \((\Tilde{\fb}^D, p_\ub)\) based on the contrast \(C_{\IMA}\).%
\end{restatable}
\cref{rgmpa-theorem} and~\cref{thm:darmois} therefore suggest that IMA may be beneficial for identifiable representation learning even when the manifold hypothesis holds, extending previous findings for $m=d$.

\section{Genericity of IMA under the Manifold Hypothesis}
\label{sec:genericity}

In this section, we
provide a formal interpretation and justification of IMA as the consequence of a {\em genericity} assumption---i.e., that the IMA principle is typically satisfied when {``Nature [chooses] the direction of the
influence of each source component in the observation space independently''}~\citep{Greseleetal21}.
We do so by defining a process to sample mixing functions from a lower-dimensional source space to a higher-dimensional observation space, and show that the IMA principle in equation~\eqref{ima:condition} is typically approximately satisfied if the influences are sampled independently from an isotropic prior. While this may not be the only way to sample mixing functions which are typically close to the IMA function class, our proposed construction provides the first rigorous statistical argument that justifies the non-statistical notion of independence expressed in~\eqref{ima:condition}.

We construct mixing functions \(\fb: \RR^d \to \RR^m\) with \(m \gg d\) such that, locally, the Jacobian of $\fb$ has columns  $\Jb_{\fb,i}(\sb) \coloneqq \Jb_{\fb}(\sb)[:,i]$ that are sampled independently and isotropically, i.e., %
from a spherically invariant distribution~\(p_\rb\) over \(\RR^m\):%
\begin{equation*}
 {\Jb_{\fb,1}(\sb), \Jb_{\fb,2}(\sb), \ldots, \Jb_{\fb,d}(\sb) \overset{\text{i.i.d}}{\sim} p_\rb}\,.
\end{equation*}
\looseness-1  The \(i\)-th column of the Jacobian, $\Jb_{\fb,i}(\sb)$,
represents the influence of the \(i\)-th source on the observations: this sampling procedure formalizes the intuition that every source %
influences the mixtures independently. %
We then proceed to show that typical samples from this process satisfy the IMA principle with high probability. Note that orthogonality of the Jacobian columns is not enforced from the outset: rather, it emerges as a property typically (approximately) satisfied by samples from this process.

In~\cref{subsec:local}, we prove an upper bound on  
the global IMA contrast \(C_{\textsc{ima}}(\fb, p_\sb)\), %
satisfied with high probability %
by linear maps. %
Next, in~\cref{subsec:global}, we show how to generate
nonlinear maps with locally independent influences,
 and prove a bound 
 for \(C_{\textsc{ima}}(\fb, p_\sb)\) of maps sampled from this procedure.
 For detailed proofs and additional technical details on the results in this section, see~\cref{app:c}.

\subsection{Bound on the local IMA contrast}
\label{subsec:local}
\begin{restatable}{theorem}{localimabound}
\label{local:ima:thm}
Consider linear maps, \(\fb(\sb)=\Jb\sb\), where the columns of \({\Jb \in \RR^{m \times d}}\) are sampled from a spherically symmetric distribution \(p_\rb\) over \(\RR^m\); \(\Jb_{1}, \Jb_{2}, ..., \Jb_{d} \overset{i.i.d}{\sim} p_\rb\). \looseness=-1 For such maps, the IMA contrast satisfies for \(m \gg d\) and \(\delta > 0\):
\begin{equation*}
    \Pr\left[C_{\textsc{ima}}(\fb, p_\sb) \le \delta\right] \geq 1 - \min \left\{1, \exp\left(2\log d - \kappa(m -1)\frac{\delta^2}{d^2}\right)\right \} \,.
\end{equation*}
\end{restatable}

This theorem is based on a concentration result for isotropic priors given by {\em Levy's lemma}~\citep[Lemma 1]{janzing2010telling}. Levy's lemma shows that a smooth function of an isotropically sampled direction concentrated around its mean with probability growing exponentially in the dimension of the sample space. We observe that in our sampling process, each sampled \emph{influence direction} is orthogonal to the other sampled directions in expectation, i.e. the pairwise inner products of the sampled directions is equal to zero in expectation. We employ Levy's lemma to derive a concentration result on the inner products to obtain the result in~\cref{local:ima:thm}.

\subsection{Bound on the global IMA contrast}
\label{subsec:global}
We now describe a sampling procedure, and derive bounds for \(C_{\textsc{ima}}(\fb, p_\sb)\), for
non-linear maps \(\fb: \RR^d \to \RR^m, m \gg d\). We retain the principle that locally the columns of the Jacobian of \(\fb\) are sampled from a spherically invariant distribution. %
Our procedure therefore samples piece-wise linear functions. We restrict the latent domain to be bounded, in particular the \(d\)-dimensional unit cube, \([0, 1]^d\), and consider a grid-like partition on the same. On each grid unit, we sample a linear function by the previously described sampling process in~\cref{subsec:local}. %

\subsubsection{Defining non-linear functions as composition of affine
functions}

\begin{definition}
    \label{def:sampling}
    On the source domain \(\sb \in [0, 1]^d\), define an axis-aligned square grid partition, with width \(\delta \in \RR\); the number of grid parts along a dimension, \(k \in [d]\) is therefore equal to \(p = \ceil{\frac{1}{\delta}}{} + 1\).\footnote{\(\ceil{.}{}\) is the ceiling function.} Consider matrices \(\Jb^{(1)}, \Jb^{(2)}, ..., \Jb^{(p)} \in \RR^{m \times d}\) with columns sampled from a spherically symmetric distribution, \(p_\rb\); \; \(\Jb^{(i)}_{1}, \Jb^{(i)}_{2}, ..., \Jb^{(i)}_{d} \overset{i.i.d}{\sim} p_\rb \forall i \in [p]\). The sampled function, \(\fb: [0, 1]^d \to \RR^m\), is specified as as a sum of \emph{coordinate-wise functions} \(\fb(\sb) = \sum^d_{k=1}\fb_k(s_k)\), where
    \begin{align}
            \label{coord:wise:grid}
            \fb_k(s_k) \coloneqq  \sum^p_{t=1}(\Jb^{(t)}_{:, k}(s_k - (t-1)\delta) + \sum^{t-1}_{i=1}\Jb^{(i)}_{:,k}\delta)1_{s_k \in ((t-1)\delta, t\delta]}\,.
        \end{align}
\end{definition}

Observe that the Jacobian of the sampled function, \(\Jb_\fb(\sb)\), has independent and isotropic Jacobian columns almost everywhere by construction; therefore we expect that the local IMA contrast, \(c_{\textsc{ima}}(\fb, \sb)\) is small almost everywhere. To derive a bound on the global IMA contrast, \(C_{\textsc{ima}}(\fb, p_\sb)\), we require the Jacobian, \(\Jb_\fb(\sb)\), to be defined everywhere---i.e., \(\fb\) should be continuous, injective and smooth. We briefly explain that the resulting \(\fb\)
is indeed continuous and injective, see~\cref{app:c} for details. We then consider a smooth approximation of the sampled function. %

\paragraph{Continuity of sampled functions.} 
Note that coordinate-wise functions, \(f_k: [0,1]\to\RR^m, k \in [d]; \fb(\sb) = \sum^d_{k=1}\fb_k(s_k)\) are piece-wise linear and continuous by definition. The sampled function \(\fb\) is therefore continuous as it is a sum of continuous functions.\looseness=-1

\paragraph{Injectivity of sampled functions.}
We show that the subspaces spanned by the images of coordinate-wise functions \(f_i: [0, 1] \to \RR^m\) of \(\fb\) are linearly independent, said to be in \emph{direct sum}, and that the coordinate-wise functions are injective. We then show that the injectivity of coordinate-wise functions that are in direct sum entails the injectivity of \(\fb(\sb) = \sum^d_{k=1}\fb_k(s_k)\).

\paragraph{Smooth approximation of sampled functions.} 
We discuss a smooth approximation to the sampled functions from~\cref{def:sampling} by means of a sinusoidal approximation to the step function.

\begin{definition}[Smooth step function]
\label{smooth:step}
We define the smooth step function as \(\tilde{1}_{\epsilon}: \RR \to \RR\) as
\begin{equation*}
    \tilde{1}_{\epsilon}(s) = \begin{cases} 0\,, & s \le -\epsilon\,,\\
\frac{1}{2}\sin \left (\frac{\pi s}{2\epsilon} \right )  + \frac{1}{2}\,, & -\epsilon < s \le \epsilon\,,\\
1\,, & s > \epsilon\,.
\end{cases}
\end{equation*}
\end{definition}

\begin{definition}[Smooth approximation to grid-wise linear functions]
 \label{def:smooth_grid}
    We define the smooth approximation of \(\fb: [0, 1]^d \to \RR^m\) as \(\tilde{\fb}_\epsilon(\sb) = \sum_{k=1}^d\tilde{\fb}_{\epsilon, k}(s_k)\) for \(0 < \epsilon \ll \delta\) where
    \begin{align*}
        \tilde{\fb}_{\epsilon, k}(s_k) \coloneqq \sum^p_{t=1}\left (\Jb^{(t)}_{:, k}(s_k - (t-1)\delta) +  \sum^{t-1}_{i=1}\Jb^{(i)}_{:,k}\delta \right )\,.(\tilde{1}_\epsilon(s_k - (t-1)\delta ) - \tilde{1}_\epsilon(s_k - t\delta))
    \end{align*}
\end{definition}

We show that \(\tilde{\fb}_\epsilon(\sb)\) obtained in~\cref{def:smooth_grid} is continuous and injective. For a detailed exposition on this section, refer to~\cref{app:c}.
\begin{theorem}[Properties of smoothened functions]
\label{thm:smooth}
Functions \(\tilde{\fb}_{\epsilon}: [0, 1]^d \to \RR^m\) defined in~\cref{def:smooth_grid} are continuously differentiable in \(\RR^d\), in addition to being continuous and injective, are continuously differentiable for \(0 < \epsilon \ll \delta\) arbitrarily small.
\end{theorem}

We can now prove the following bound on the global IMA contrast \(C_{\textsc{ima}}(\fb, p_\sb)\) for the class of nonlinear functions specified in~\cref{def:smooth_grid}.

\begin{theorem}
\label{global:ima:thm_grid}
Consider the map \(\tilde{\fb}_\epsilon: [0, 1]^d \to \RR^m\) sampled randomly from the procedure ~\ref{def:smooth_grid}. 
Then the map \(\tilde{\fb}_\epsilon: \RR^d \to \RR^m\) for \(0 < \epsilon \ll \delta\) and any finite probability density \(p_\sb\) defined over \([0, 1]^d\) satisfies the following bound on the global IMA contrast \(C_{\textsc{ima}}(\tilde{\fb}_\epsilon, p_\sb)\) for \(m \gg d\) and \(\delta > 0\):
\begin{equation*}
    \Pr\left[C_{\textsc{ima}}(\tilde{\fb}_\epsilon, p_\sb) \le \delta\right] \ge 1 - \min \left\{1, \exp\left(2\log d - \kappa(m -1)\frac{\delta^2}{d^2}\right)\right\}
\end{equation*}
\end{theorem}

\Cref{global:ima:thm_grid} shows that for smooth piecewise-linear functions \(\tilde{\fb}_\epsilon: [0, 1]^d \to \RR^m\) sampled according to~\cref{def:smooth_grid}, the IMA principle~(\cref{ima:condition}) is {\em typically approximately satisfied}: i.e., the probability of the columns of \(\Jb_{\tilde{\fb}_\epsilon}(\sb)\) being close to orthogonal increases as the dimension of the observation space grows. 
We achieve this result by using the previously derived bound on the IMA contrast %
for linear functions (\cref{local:ima:thm}), and applying it locally on the interior of the grid. We then show that the local IMA contrast, \(c_{\textsc{ima}}(\fb, \mathbf{s})\), is finite on the grid boundary and can be neglected since the volume of the boundary is small.

\Cref{global:ima:thm_grid} thus enables us to view the IMA principle as the consequence of a genericity assumption for the sampling process in~\cref{def:smooth_grid}. This is because the functions sampled in accordance with~\cref{def:smooth_grid} do not satisfy the IMA principle by construction: instead, it is the typical draws from the sampling procedure that approximately satisfy the principle when the observation space is high-dimensional.

\section{Conclusion}

We extended IMA theory under the {\em manifold hypothesis},
revisiting the definitions from previous works, and show that  
IMA provably circumvents non-identifiability issues even in this setting. 
Our results pave the way for an application of IMA to realistic representation learning involving dimensionality reduction.
We also showed that the IMA principle can be seen as the consequence of a {\em genericity} assumption when the manifold hypothesis holds: this clarifies the interpretation of IMA. In particular, when the latent data-generating factors influence the observed mixture \emph{independently}, the orthogonality condition given by IMA \emph{typically} holds.\looseness=-1

\section*{Acknowledgments}

We thank Georgios Arvanitis, Dominik Janzing, Simon Buchholz, Patrik Reizinger and Joanna Sliwa for interesting discussions, and the anonymous reviewers for useful feedback.

This work was supported by the T\"ubingen AI Center and by the Deutsche Forschungsgemeinschaft (DFG, German Research Foundation) under Germany’s Excellence Strategy, EXC number 2064/1, Project number 390727645. 
L.G.\ was supported by the VideoPredict project, FKZ: 01IS21088.
\bibliography{refs}

\clearpage
\appendix
\begin{center}
{\centering \LARGE APPENDIX}
\vspace{0.8cm}
\sloppy
\end{center}

\section*{Overview}
\begin{itemize}
    \item Appendix~\ref{app:b} provides proofs of technical results in Section~\ref{sec:man_hyp}.
    \item Appendix~\ref{app:c} provides a detailed exposition on Section~\ref{sec:genericity}.
\end{itemize}

\section{Proof of technical results in Section~\ref{sec:man_hyp}}
\label{app:b}

\subsection{Preliminaries}
\label{app:a}
We provide some preliminary definitions that we refer to in the main text and the remainder of the Appendix.
\begin{definition}[Diffeomorphism, Chapter 1~\citep{milnor1997topology}]
\label{def:diff}
A diffeomorphism is an invertible function, \(\fb\), which maps one differentiable manifold onto another such that both the function and its inverse are smooth.
\end{definition}
\begin{definition}[Spherically symmetric distribution]
\label{sphdist}
A distribution \(p_\rb\) on the \(m\)-dimensional Lebesgue measure is said to be spherically symmetric if \(\forall \xb \in \RR^m, \xb \overset{p_\rb}{\sim} \Ob\xb\), where \(\Ob \in \RR^{m \times m}\) is an orthonormal matrix.
\end{definition}

\subsection{Properties of the local IMA contrast under the Manifold Hypothesis}

\cima*
\begin{proof}\(\newline\)
\begin{enumerate}[(i)]
    \item  \(c_{\IMA}(\fb, \mathbf{s}) = \frac{1}{2} D^{l}_{KL}(\Jb^{\top}_\fb(\mathbf{s})\Jb_\fb(\mathbf{s})) \ge 0\) with equality iff. \(\Jb^{\top}_\fb(\mathbf{s})\Jb_\fb(\mathbf{s})\) is a diagonal matrix. \(\Jb^{\top}_\fb(\mathbf{s})\Jb_\fb(\mathbf{s})\) is diagonal iff. 
    the columns of \(\Jb_\fb(\mathbf{s}), \frac{\partial \fb}{\partial u_i}(\mathbf{s})\) are orthogonal.\\
    
    Thus, we have shown that \(c_{\IMA}(\fb, \mathbf{s}) \ge 0\) with equality iff all columns \(\frac{\partial \fb}{\partial u_i}(\mathbf{s})\) of \(\Jb_\fb(\mathbf{s})\) are orthogonal. \\   
    \textbf{Remark:} 
    To show that \(D^{l}_{KL}(\Jb^{\top}_\fb(\mathbf{s})\Jb_\fb(\mathbf{s})) \ge 0\), Hadamard's determinant inequality is used in a different form than in \citep{Greseleetal21}. In particular, for positive definite matrices, in our case \(\Wb = \Jb^{\top}_\fb(\mathbf{s})\Jb_\fb(\mathbf{s})\), the determinant is upper-bounded by the product of its diagonal entries.    
    \begin{equation}
        \label{hadamard-rect}
        |\text{det}(\Wb)| \le \prod^d_{i=1}w_{ii}
    \end{equation}
    with equality iff. \(\Wb\) is diagonal.
    This is obtained by considering the Cholesky decomposition of \(\Wb = \Nb\Nb^{\top}\) which uniquely exists for any real positive definite matrix, where \(\Nb \in \mathbb{R}^{d \times d}\)
    \begin{equation*}
        |\text{det}(\Wb)| = |\text{det}(\Nb)|^2 \le \prod^d_{i=1}\|\nb_i\|^2 = \prod^d_{i=1}w_{ii}
    \end{equation*}
    where \(\nb_i\) are the columns of \(\Nb\).
    Equality holds iff. the columns of \(\Nb\) are orthogonal i.e. \(\Wb\) is diagonal.
    
   \item \textit{Invariance to left multiplication by orthogonal matrix}\\
    \(\Wb = \Jb_\fb(\mathbf{s}) \in \mathbb{R}^{m \times d}\), \(\Ob \in \mathbb{R}^{m \times m}\) is an orthogonal matrix. \(\Tilde{\Wb} = \Ob\Wb\).
\begin{align*}
    \frac{1}{2}D^{l}_{KL}(\Tilde{\Wb}^{\top}\Tilde{\Wb}) &= \frac{1}{2}D^{l}_{KL}(\Wb^{\top}\Ob^{\top}\Ob\Wb)\\
    &= \frac{1}{2}D^{l}_{KL}(\Wb^{\top}\Ib_m\Wb)\\
    &= \frac{1}{2}D^{l}_{KL}(\Wb^{\top}\Wb)
\end{align*}

This corresponds to a change of basis in the observation space.

\textit{Invariance to right multiplication by a permutation matrix}\\
Let \(\Tilde{\Wb} = \Wb\Pb\) where \(\Pb \in \mathbb{R}^{d \times d}\) is a permutation matrix. Then \(\Tilde{\Wb}\) is just \(\Wb\) with permuted columns. Clearly, the sum of log-column-norms does not change with the order of the summands. Further, log\(|\Tilde{\Wb}^{\top}\Tilde{\Wb}| = \text{log}|\Pb^{\top}\Wb^{\top}\Wb\Pb| = \text{log}|\Pb^{\top}| + \text{log}|\Wb^{\top}\Wb| +\text{log}|\Pb| = \text{log}|\Wb^{\top}\Wb|\) because the absolute value of the determinant of the permutation matrix is one.

\textit{Invariance to right multiplication by a diagonal matrix}\\
Let \(\Tilde{\Wb} = \Wb\Db\) where \(\Db \in \mathbb{R}^{d \times d}\) is a diagonal matrix.\\
For the first term, we know that the columns of \(\Tilde{\Wb}\) are scaled versions of the columns of \(\Wb\), i.e. \(\Tilde{\wb}_i = \wb_i, \|\Tilde{\wb}_i\| = |d_i|\|\wb_i\|\). For the second term, we use the decomposition of the determinant:
\begin{align*}
    \log |\Tilde{\Wb}^{\top}\Tilde{\Wb}| &= \log |\Db^{\top}\Wb^{\top}\Wb\Db|\\
    &= 2\log|\Db| + \log|\Wb^{\top}\Wb| \\
    &= \log|\Wb^{\top}\Wb| + 2\sum^d_{i=1}\log|d_i|\
\end{align*}
Taken together, we obtain:
\begin{align*}
    \sum^d_{i=1}\log\|\Tilde{\wb}_i\| - \frac{1}{2}\log|\Tilde{\Wb}^{\top}\Tilde{\Wb}| &= \sum^d_{i=1}\log|d_i|\|\wb_i\| - \frac{1}{2}\left ( \log|\Wb^{\top}\Wb| + 2\sum^d_{i=1}\log|d_i|\right)\\
    &= \sum^d_{i=1}\log\|\wb_i\| + \sum^d_{i=1}\log|d_i| - \frac{1}{2} \log|\Wb^{\top}\Wb|\\
    & \quad - \sum^d_{i=1}\log|d_i|\\
    &= \sum^d_{i=1}\log\|\wb_i\| - \frac{1}{2} \log|\Wb^{\top}\Wb|
\end{align*}
\end{enumerate}
\end{proof}

\subsection{Properties of the global IMA contrast under the Manifold Hypothesis}
\Cima*
In \citep{Greseleetal21}, for property (i) \(\Jb_{\fb}(\mathbf{s})\) can be expressed as \(\Ob(\mathbf{s})\Db(\mathbf{s})\) where \(\Ob(\mathbf{s}), \Db(\mathbf{s})\) are orthogonal and diagonal matrices respectively in the condition for equality. This is no longer possible in the case for high dimensional observations because the Jacobian \(\Jb_{\fb}(\mathbf{s})\) is no longer a square matrix.
\begin{proof}\(\newline\)
\begin{enumerate}[(i)]

    \item From property (i) of Proposition 4.4, we know that \(c_{\IMA}(\fb, \mathbf{s}) \ge 0\). Hence, \(C_{\IMA}(\fb, p_{\mathbf{s}}) \ge 0\) follows as a direct consequence of integrating the non-negative quantity \(c_{\IMA}(\fb, \mathbf{s}) \ge 0\).
    
    Equality is attained iff.\ \(c_{\IMA}(\fb, \mathbf{s}) = 0\) almost surely wrt \(p_{\mathbf{s}}\) which holds when \(\Jb_{\fb}(\mathbf{s})\) has orthogonal columns almost surely wrt \(p_{\mathbf{s}}\).
    
    \item  \(\Tilde{\fb} = \fb \circ \hb^{-1} \circ \Pb^{-1}\) and \(\mathbf{\Tilde{\sb}} = \Pb\hb(\mathbf{s})\) where \(\Pb \in \mathbb{R}^{d \times d}\) is a permutation matrix and \(\hb(\mathbf{s}) = (h_1(s_1), h_2(s_2), ..., h_d(s_d))\) is an invertible element-wise function. Consider
    \begin{equation}
        C_{\IMA}(\Tilde{\fb}, p_{\mathbf{\Tilde{s}}}) = \int c_{\IMA}(\Tilde{\fb}, \Tilde{\mathbf{s}})p_{\mathbf{\Tilde{s}}}(\mathbf{\Tilde{s}})d\mathbf{\Tilde{s}} = \int c_{\IMA}(\Tilde{\fb}, \Tilde{\mathbf{s}})p_{\mathbf{s}}(\mathbf{s})d\mathbf{s} \label{global}
    \end{equation}
    where for the second equality we have used \(p_{\mathbf{\Tilde{s}}}(\mathbf{\Tilde{s}})d\mathbf{\Tilde{s}} = p_{\mathbf{s}}(\mathbf{s})d\mathbf{s}\) since \(P \circ h\) is an invertible transformation. It thus suffices to show that 
    \begin{equation}
        c_{\IMA}(\Tilde{\fb}, \Tilde{\mathbf{s}}) = c_{\IMA}(\fb, \mathbf{s}) \label{conteq}
    \end{equation}
    at any point \(\Tilde{\mathbf{s}} = \Pb \circ \hb (\mathbf{s})\). To show this we write
    \begin{align*}
        \Jb_{\Tilde{\fb}}(\mathbf{\Tilde{s}}) &= \Jb_{\fb \circ \hb^{-1} \circ \Pb^{-1}}(\Pb\hb (\mathbf{s}))\\
        &= \Jb_{\fb \circ \hb^{-1}}(\Pb^{-1} \circ \Pb\hb (\mathbf{s}))\Jb_{\Pb^{-1}}(\Pb\hb (\mathbf{s}))\\
        &= \Jb_{\fb \circ \hb^{-1}}(\hb (\mathbf{s}))\Jb_{\Pb^{-1}}(\Pb\hb (\mathbf{s}))\\
        &= \Jb_\fb(\hb^{-1} \circ \hb (\mathbf{s}))\Jb_{\hb^{-1}}(\hb (\mathbf{s}))\Jb_{\Pb^{-1}}(\Pb\hb (\mathbf{s}))\\
        &= \Jb_\fb(\mathbf{s})\Db(\mathbf{s})\Pb^{-1}
    \end{align*}
    
    where we have used the chain rule for differentiability, \(\Jb_{\hb^{-1}}(\hb (\mathbf{s}))\) is a diagonal matrix \(\Db(\mathbf{s})\) and \(\Jb_{\Pb^{-1}} = \Pb^{-1}\) for any \(\mathbf{s}\). Note that \(\Pb^{-1}\) is also a permutation matrix.
    
    The equality in (\ref{conteq}) follows from applying (ii) from Proposition 4.4. Substituting (\ref{conteq}) into (\ref{global}), we finally obtain 
    \begin{equation*}
         C_{\IMA}(\Tilde{\fb}, \pb_{\mathbf{\Tilde{s}}}) = C_{\IMA}(\fb, \pb_{\mathbf{s}})
    \end{equation*}
\end{enumerate}
\end{proof}
\subsection{Ruling out ``spurious solutions" under the manifold hypothesis}
``Spurious solutions" to nonlinear ICA on an observation manifold are constructed by composing \emph{conformal} maps -- a subclass of the IMA function class -- with measure preserving automorphisms on the latent space. We define conformal maps between Riemannian manifolds below, and comment on the Jacobian of conformal maps.

\begin{definition}[Conformal map between Riemannian manifolds~\citep{ishii1957conharmonic, stepanov2017theorems}]

\label{conformal}

A diffeomorphism \(\fb: (\MM, \gb) \to (\bar{\MM}, \bar{\gb})\) between two Riemannian manifolds, \(\MM\) and \(\bar{\MM}\), equipped with the Riemannian metric tensors, \(\gb\) and \(\bar{\gb}\), is called \textit{conformal} if it preserves the angles between any pair curves. In this case, the metric tensors \(\gb\) and \(\bar{\gb}\) are related as \(\bar{\gb} = e^{2\sigma}\gb\) for some scalar function \(\sigma \in C^2\MM\). 

\end{definition}

We make an observation on conformal maps from the \(d\)-dimensional Euclidean space to a \(d\)-dimensional Riemannian manifold living in a higher \(m\)-dimensional Euclidean space (\(m \ge d\)). In overview, this observation derives from the definition of conformal maps on manifolds~\ref{conformal} that the columns of the Jacobian \(\Jb_\fb(\sb)\) are equal in norm for all values in the domain of \(\fb\), here \(\sb \in \RR^d\), which is equivalent to the condition that \(\Jb^{\top}_\fb(\sb)\Jb_\fb(\sb)\) is a scalar multiple of the identity matrix.

In our scenario, we consider a map between the Riemannian manifolds, \(\MM \equiv \RR^d\) to \(\bar{\MM} \equiv \Xcal \subset \RR^m\), where \(m \gg d\). Note that the \(d\)-dimensional Euclidean space is also a Riemannian manifold.

The tangent space of \(\MM \equiv \RR^d\) is set of the canonical basis vectors, \(e_1, e_2, ...,e_d\) at all points in \(\RR^d\). Hence, the metric tensor associated with \(\MM \equiv \RR^d\) is the identity matrix \(\gb \equiv \II_d\) at all points in \(\RR^d\). 

The metric tensor associated with the Riemannian manifold, \(\bar{\MM} \equiv \Xcal\) is written as \(\bar{\gb} \equiv \Jb^{\top}_\fb(\mathbf{s})\Jb_\fb(\mathbf{s})\) at the point \(\fb(\mathbf{s}) \in \Xcal\), where \(\sb \in \RR^d\) upon which the bijective map \(\fb: \RR^d \to \Xcal\) acts.

For the map, \(\fb\) to be conformal \ref{conformal}, we require that \(\bar{\gb} = e^{2\sigma}\gb\) for some scalar function \(\sigma \in C^2\MM\). This is equivalent to the condition that \(\Jb^{\top}_\fb(\sb)\Jb_\fb(\sb) = t(\sb)\II_d\), where \(t: \RR^d \to \RR^+\) is a positive scalar function. This is further equivalent to \(\Jb_\fb(\sb)\) having orthogonal columns with the same norm.

We define \(\lambda(\sb) = \sqrt{t(\sb)}\) as the conformal factor (equal to the norm of the column vectors of \(\Jb_\fb(\sb)\)).

\paragraph{Rotated-Gaussian measure preserving automorphism}

We now present the theorem which shows that the IMA contrast rules out rotated-Gaussian measure preserving automorphism solutions.

\rgmpa*

\begin{proof}
Recall the definition
\begin{equation*}
    a^\Rb(p_\mathbf{s}) = \fb_{\sb}^{-1} \circ \Phi \circ \Rb \circ \Phi^{-1} \circ \fb_{\sb}
\end{equation*}
For notational convenience, we denote \(\sigma = \Phi^{-1} \circ \fb_{\sb}\) and write
\begin{equation*}
    a^\Rb(p_\mathbf{s}) = \sigma^{-1} \circ \Rb \circ \sigma
\end{equation*}

Note that since both \(\fb_{\sb}\) and \(\Phi\) are element-wise transformations so is \(\sigma\).

First by using property \textit{(ii}) of Prop. \ref{global-ima-prop} (invariance of \(C_{\IMA}\) to element-wise transformations), we obtain
\begin{equation*}
    C_{\IMA}(\fb \circ a^{\Rb}(p_{\sb}), p_{\sb}) = C_{\IMA}(\fb \circ \sigma^{-1} \circ \Rb \circ \sigma, p_{\sb}) = C_{\IMA}(\fb \circ \sigma^{-1} \circ \Rb, p_{\zb}) 
\end{equation*}
with \(\zb = \sigma(\sb)\) such that \(p_{\zb}\) is an isotropic Gaussian distribution.

Suppose \textit{for a contradiction} that \(C_{\IMA}(\fb \circ \sigma^{-1} \circ \Rb, p_{\zb}) = 0\).

According to property \textit{(i)} of Prop. \ref{global-ima-prop}, this entails that the matrix 
\begin{equation}
    \label{chain-rule}
    \Jb_{\fb \circ \sigma^{-1} \circ \Rb}(\zb)^{\top}\Jb_{\fb \circ \sigma^{-1} \circ \Rb}(\zb) = \Rb^{\top}J_{\sigma^{-1}}(\zb)^{\top}\Jb_\fb(\sigma^{-1}(\zb))^{\top}\Jb_\fb(\sigma^{-1}(\zb))\Jb_{\sigma^{-1}}(\zb)\Rb
\end{equation}
is diagonal almost surely w.r.t \(p_\zb\). Moreover, smoothness of \(p_\sb\) and \(\fb\) implies the matrix expression of \ref{chain-rule} is a continuous function of \(\zb\). Thus \(\Jb_{\fb \circ \sigma^{-1} \circ \Rb}(\zb)^{\top}\Jb_{\fb \circ \sigma^{-1} \circ \Rb}(\zb)\) needs to be diagonal for all \(\zb \in \RR^d\).

Since \(\fb\) is a conformal map, 
\begin{equation*}
    \Jb_\fb(\sigma^{-1}(\zb))^{\top}\Jb_\fb(\sigma^{-1}(\zb))
\end{equation*}
is diagonal. Moreover, since \(\sigma\) is an element-wise transformation, \(\Jb_{\sigma^{-1}}(\zb)^{\top}\) and \(\Jb_{\sigma^{-1}}(\zb)\) are also diagonal. Taken together, this implies that 
\begin{equation}
    \label{diag}
    \Jb_{\sigma^{-1}}(\zb)^{\top}\Jb_\fb(\sigma^{-1}(\zb))^{\top}\Jb_\fb(\sigma^{-1}(\zb))J_{\sigma^{-1}}(\zb)
\end{equation}
is diagonal (i.e. \ref{chain-rule} is of the form \(\Rb^{\top}\Db(\zb)\Rb\) for some diagonal matrix \(\Db(\zb))\).

Without loss of generality, we assume the first dimension \(s_1\) of \(\sb\) is non-Gaussian and satisfies the assumptions relative to \(\Rb\) (axis not invariant nor sent to another canonical axis).

Now, since both the Gaussian CDF \(\Phi\) and the CDF \(\fb_\sb\) are smooth (the latter by the assumption that \(p_\sb\) is a smooth density), \(\sigma\) is a smooth function and thus has continuous partial derivatives.

By continuity of the partial derivative and the non-Gaussianity of \(s_1\), the first diagonal element \(\frac{\partial \sigma_1^{-1}}{\partial z_1}\) of \(\Jb_{\sigma^{-1}}\) must be strictly monotonic in a neighborhood of some \(z^0_1\). 

On the other hand, our assumptions related to \(\Rb\) entail that there are at least two non-vanishing coefficients in the first row of \(\Rb\). Let us call \(i \ne j\) such pairs of coordinates, i.e. \(r_{1i} \ne 0\) and \(r_{1j} \ne 0\).

Now consider the off-diagonal term \((i, j)\) of \(\Jb_{\fb \circ \sigma^{-1} \circ \Rb}(\zb)^{\top}\Jb_{\fb \circ \sigma^{-1} \circ \Rb}(\zb)\) which we assumed must be 0. Since the term in \ref{diag} is diagonal, this off-diagonal term is given by 
\begin{equation*}
    \sum^n_{k=1}\left ( \frac{\partial \sigma_k^{-1}}{d z_k}(z_k) \right)^2\left \|\frac{\partial \fb}{d s_k} \circ \sigma^{-1}(\zb)\right\|^2r_{ki}r_{kj} = \sum^n_{k=1}\left ( \frac{\partial \sigma_k^{-1}}{d z_k}(z_k) \right)^2\lambda( \sigma^{-1}(\zb)) ^2r_{ki}r_{kj} = 0 
\end{equation*}
By definition of conformal map between Riemannian manifolds \ref{conformal}, the square of the conformal factor is a strictly positive function.
\begin{equation*}
    \lambda( \sigma^{-1}(\zb)) ^2 > 0 \forall \zb
\end{equation*}
Thus, for all \(\zb\) we must have 
\begin{equation}
    \label{summation}
    \sum^n_{k=1}\left ( \frac{\partial \sigma_k^{-1}}{d z_k}(z_k) \right)^2r_{ki}r_{kj} = 0 
\end{equation}
Now consider the first term \(\left ( \frac{\partial \sigma_1^{-1}}{d z_1}(z_1) \right)^2r_{1i}r_{1j}\) in the sum.

Recall that \(r_{1i}r_{1j} \ne 0\), and that \(\frac{\partial \sigma_1^{-1}}{d z_1}(z_1)\) is strictly monotonic on a neighborhood of \(z^0_1\).

As a consequence, \(\left ( \frac{\partial \sigma_1^{-1}}{d z_1}(z_1) \right)^2r_{1i}r_{1j}\) is also strictly monotonic with respect to \(z_1\) on a neighborhood of \(z^0_1\) (where the other variables \((z_2, z_3, ..., z_n)\) are left constant), while the other terms in \ref{summation} are left constant because \(\sigma\) is an element-wise transformation. 

This leads to a contradiction as \ref{summation} (which should be satisfied for all \(\zb\)) cannot constantly stay zero as \(z_1\) varies within the neighborhood of \(z_1^0\).

Hence, our assumption that \(C_{\IMA}(f \circ \sigma^{-1} \circ \Rb, p_{\zb}) = 0\) cannot hold.

We conclude that \(C_{\IMA}(\fb \circ \sigma^{-1} \circ R, p_{\zb}) > 0\).

\end{proof}

\paragraph{Measure preserving automorphism based on the Darmois construction}\(\\\)
Following are helper lemmata to prove Theorem~\ref{thm:darmois}, which rules out a counterexample based on the Darmois construction.
\begin{lemma}
\label{obs:Jacobian_darmois}
Jacobian of the Darmois construction, \(\gb^D(\xb)\), in Definition~\ref{def:darmois} is lower triangular. 
\end{lemma}
\begin{proof}
On applying the recursive Darmois construction, we obtain latent variables \(\zb = \gb^D(\xb) \sim \text{Unif}([0, 1]^d)\). The Darmois construction is invertible since the (conditional) cumulative distribution function is injective. Consider the inverse of the Darmois construction, \(\fb^D\) such that \(\Xb = \fb^D(\zb)\). We observe from~\ref{def:darmois} that \(x_1\) is related to all the coordinates of \(\zb = (z_1, z_2, ... , z_d)\), \(z_2\) is related to \(\Zb_{\ge 2} = (z_2, z_3, ... , z_d)\) and so on. Hence, we take note of the observation that the Jacobian of \(\fb^D\) is lower-triangular.
\end{proof}
\begin{lemma}

\label{orth:lemma}
Consider a matrix \(\Tilde{\Ab} = \Ab\Ob\), where \(\Tilde{\Ab} \in \RR^{n \times d}, \Ab \in \RR^{n \times d}, \Ob \in \RR^{d \times d}\). \(\Ab\) is a tall matrix which has orthogonal columns with unit norm i.e. \(\Ab^{\top}\Ab = \II_d\), and \(\Ob\) is an orthonormal matrix. Then, \(\Tilde{\Ab}\) has orthogonal columns with unit norm i.e. \(\Tilde{\Ab}^{\top}\Tilde{\Ab} = \II_d\). 
\end{lemma}

\begin{proof}
\begin{align*}
    \Tilde{\Ab}^{\top}\Tilde{\Ab} &= (\Ab\Ob)^{\top}(\Ab\Ob)\\
                        &= \Ob^{\top}\Ab^{\top}\Ab\Ob\\
                        &= \Ob^{\top}\Ob = \II_d
\end{align*}
\end{proof}

\begin{lemma}
\label{tri:lemma}

Consider a matrix, \(\Tilde{\Ab} = \Ab\Tb\), where \(\Ab \in \RR^{n \times d}, \Tb \in \RR^{d \times d}\). \(\Ab\) has orthogonal columns with unit norm i.e. \(\Ab^{\top}\Ab = \II_d\), and \(\Tb\) is a lower-triangular matrix. \(\Tilde{\Ab}\) has orthogonal columns iff. \(\Tb\) is diagonal.

\end{lemma}

\begin{proof}

\(\Tb\) \textit{is diagonal.} \(\implies\) \(\Tilde{\Ab}\) \textit{has orthogonal columns.}

\begin{align*}
    \Tilde{\Ab}^{\top}\Tilde{\Ab} &= (\Ab\Tb)^{\top}(\Ab\Tb)\\
                        &= \Tb^{\top}\Ab^{\top}\Ab\Tb\\
                        &= \Tb^{\top}\Tb = \Tb^2
\end{align*}
\(\Tilde{\Ab}^{\top}\Tilde{\Ab}\) is a diagonal matrix, hence \(\Tilde{\Ab}\) has orthogonal columns.\
\
\(\Tilde{\Ab}\) \textit{has orthogonal columns.}  \(\implies\) \(\Tb\) \textit{is diagonal.}

We know that \(\Db = \Tilde{\Ab}^{\top}\Tilde{\Ab}\) is diagonal, by definition of orthogonality of the columns of \(\Tilde{\Ab}\).

\begin{align*}
    D = \Tilde{\Ab}^{\top}\Tilde{\Ab} &= \Tb^\Tb\Ab^{\top}\Ab\Tb\\
                        &= \Tb^{\top}\Tb
\end{align*}

Consider the determinant of \(\Db\), \(|\Db|\), and the determinant of \(\Tb^{\top}\Tb\), \(|\Tb^{\top}\Tb|\). \(|\Db| = \prod_{i=1}^d D_{ii} = \prod_{i=1}^d \|\Tb_i\|^2\). Also, \(|\Tb^{\top}\Tb| = |\Tb|^2 = \prod^d_{i = 1} T_{ii}^2\), since the determinant of a triangular matrix is the product of its diagonal elements.
\begin{align*}
    \Db = \Tb^{\top}\Tb, \;|\Db| = |\Tb^{\top}\Tb|\\
    \prod_{i=1}^d \|\Tb_i\|^2 = \prod^d_{i = 1} T_{ii}^2\\
    \implies \Tb \text{ is diagonal.}
\end{align*}
\end{proof}

\begin{lemma}

\label{ele:lemma}
A smooth function \(\fb: \RR^d \to \RR^d\) whose Jacobian is diagonal everywhere is an element-wise function, \(f(\sb) = (f_1(s_1), f_2(s_2), ..., f_d(s_d))\).  
\end{lemma}

\begin{proof}
Let \(\fb\) be a smooth function with a diagonal Jacobian everywhere.

Consider the function \(f_i(\sb)\) for any \(i \in {1, 2, ..., d}\). Suppose \textit{for a contradiction} that \(f_i\) depends on \(s_j\) for some \(j \ne i\). Then there must be at least one point \(\sb*\) such that \(\frac{\partial f_i}{s_j}(\sb*) \ne 0\). However, this contradicts the assumption that \(\Jb_\fb\) is diagonal everywhere (since \(\frac{\partial f_i}{s_j}\) is an off-diagonal element for \(i \ne j\) ). Hence, \(f_i\) can only depend on \(s_i\) for all \(i\). i.e. \(f\) is an element-wise function.  
\end{proof}

\darmois*

\begin{proof}

The theorem follows the following bijective maps:
\begin{align*}
    \xb \in \Xcal \xleftrightarrow[(i)]{} \sb \in \RR^d \xleftrightarrow[(ii)]{} \Tilde{\xb} \in \RR^d \xleftrightarrow[(iii)]{}\Tilde{\sb} \in \RR^d
\end{align*}
The bijective maps are described as follows:
\begin{enumerate}[(i)]
    \item \(\xb = \fb(\sb), \sb = \fb^{-1}(\xb)\)
    \item \(\Tilde{\xb} = \Ob\sb, \Ob \in \RR^{d \times d} \)
    \item \(\Tilde{\sb} = \gb^D(\Tilde{\xb})\) by the Darmois construction \ref{def:darmois}
\end{enumerate}

\(\Tilde{\sb}\) is mixed with respect to \(\sb\) since \(\gb^D \ne \Ob^{\top}\) as the Jacobians cannot match, \(\Jb_{\gb^D} \ne \Jb_{\Ob^{\top}}\) unless \(\gb^D\) is an element-wise transformation. \(\Jb_{\gb^D}\) is a triangular matrix by \ref{obs:Jacobian_darmois}, and \(\Jb_{\Ob^{\top}} = \Ob^{\top}\) is an orthonormal matrix.

We want to check if the solution, \((\Tilde{\fb}^D, p_\ub)\) satisfies IMA, i.e. \(C_{\IMA}(\Tilde{\fb}^D, p_\ub) = 0\). This is satisfied if \(J_{\Tilde{\fb}^D}(\Tilde{\sb})\) has orthogonal columns almost surely.

\begin{align*}
    \Jb_{\Tilde{\fb}^D}(\Tilde{\sb}) &= \Jb_{\fb \circ \Ob \circ {\gb^D}^{-1}}(\Tilde{\sb})\\
    &= \Jb_{\fb}(\Ob \circ {\gb^D}^{-1}(\Tilde{\sb}))\Jb_{\Ob}({\gb^D}^{-1}(\Tilde{\sb}))\Jb_{{\gb^D}^{-1}}(\Tilde{\sb})\\
    &= \Jb_{\fb}(\Ob \circ {\gb^D}^{-1}(\Tilde{\sb}))\Ob\Jb_{{\gb^D}^{-1}}(\Tilde{\sb})
\end{align*}

Consider \(\Ab = \Jb_{\fb}(\Ob \circ {\gb^D}^{-1}(\Tilde{\sb})), \Tb = \Jb_{{\gb^D}^{-1}}(\Tilde{\sb})\). Since \(\fb\) is a conformal map, \(\Ab\) has orthogonal columns with the same norm. Without loss of generality, we consider that the norm of the columns of \(\Ab\) is one. \(\Tb\) is a lower triangular matrix by \ref{obs:Jacobian_darmois}.

\begin{align*}
    \Jb_{\Tilde{\fb}^D}(\Tilde{\sb}) &= \Ab \Ob \Tb \\
                                    &= \Tilde{\Ab} \Tb && \Tilde{\Ab} \text{ has orthogonal columns with unit norm by \ref{orth:lemma}}
\end{align*}

For \(\Jb_{\Tilde{\fb}^D}(\Tilde{\sb})\) to have orthogonal columns, \(\Tb\) is diagonal (by \ref{tri:lemma}). 

Thus, for \(C_{\IMA}(\Tilde{\fb}^D, p_\ub) = 0\), \(\Tb\) has to be almost surely diagonal w.r.t \(p_{\Tilde{\sb}}\). 

Consider an off-diagonal element of \(\Tb = \Jb_{{\gb^D}^{-1}}(\Tilde{\sb}) = \frac{\partial {g^D}^{-1}_i}{\Tilde{\sb}_j}\) for \(i \ne j\), and because continuous functions which are zero almost everywhere must be zero everywhere, we conclude that \(\frac{\partial {g^D}^{-1}_i}{\Tilde{s}_j} = 0\) everywhere for \(i \ne j\), i.e. the Jacobian \(\Jb_{{\gb^D}^{-1}}(\Tilde{\sb})\) is \textit{diagonal everywhere}.

Hence, we conclude from Lemma \ref{ele:lemma} that \({\gb^D}^{-1}\) must be an element-wise function, \({\gb^D}^{-1}(\Tilde{\sb}) = ({g^D}^{-1}_1(\Tilde{s}_1), {g^D}^{-1}_2(\Tilde{s}_2), ..., {g^D}^{-1}_d(\Tilde{s}_d) )\).

Since \(\Tilde{\sb}\) has independent components by construction, it follows that \(\Tilde{x}_i = ({g^D}^{-1}_i(\Tilde{s}_i)\) and \(\Tilde{x}_j = ({g^D}^{-1}_j(\Tilde{s}_j)\) are independent for any \(i \ne j\). This implies that \(\Ob\) is a trivial matrix, i.e. a permutation or element-wise scaling. This is a contradiction to our theorem assumption.

We conclude that \(\Jb_{{\gb^D}^{-1}}(\Tilde{\sb})\) cannot be diagonal almost everwhere, and hence, \(C_{\IMA}(\Tilde{\fb}^D, p_\ub) > 0\).

Thus a solution satisfying \(C_{\IMA}(\fb, p_{\sb})\) can be distinguished from \((\Tilde{\fb}^D, p_\ub)\) based on the contrast \(C_{\IMA}\).

\end{proof}

\section{Genericity of IMA under the manifold hypothesis}
\label{app:c}

In this section, we provide a detailed exposition of the genericity arugument for IMA under the manifold hypothesis, presented in Section~\ref{sec:genericity} of the main text. 

\subsection{Levy's Lemma}
Genericity claims typically rely on high-dimensional concentration results~\citep{janzing2010causal, janzing2010telling}. In our work, we heavily use Levy's lemma, which is concentration result on smooth functions of vectors sampled from spherically symmetric priors around their mean. 

\begin{lemma}[Lévy's Lemma~\citep{janzing2010telling}]
    \label{levy}
    Let g: \(\mathbb{U}_m \to \RR\) be a \(L\)-Lipschitz continuous function on the \(m\)-dimensional sphere. If a point \(\ub\) on \(\DD_m\) is randomly chosen according to an \(\Ob(m)\)-invariant prior, it satisfies
    \begin{equation*}
        |g(\ub) - \bar{g}| \le \epsilon
    \end{equation*}
    
    with probability at least \(1 - \exp (-\kappa(m - 1)\epsilon^2/L^2)\) for some constant \(\kappa\) where \(\bar{g}\) can be interpreted as the median or average of \(g(\ub)\).
\end{lemma}

\subsection{Bound on the local IMA contrast}

Following are helper lemmata for proving Theorem~\ref{local:ima:thm}, which presents a high probability upper bound on the local IMA contrast, \(c_{\textsc{ima}}(\fb, \sb)\), on functions, \(\fb: \RR^d \to \RR^m\), sampled according to a statistical process which tries to emulate the IMA princicple, see section~\ref{subsec:local} in the main text. 

\begin{lemma}
\label{linind:sphind}
Consider a random matrix \(\Vb \in \RR^{m \times d}\) with columns \(\vb_1, \vb_2, ... \vb_d \overset{i.i.d}{\sim} p_\rb\) where \(p_\rb\) is a finite spherically symmetric distribution (\ref{sphdist}) on the Lebesgue measure over \(\RR^m\). Then \(\vb_1, \vb_2, ..., \vb_d\) are non-zero linearly independent with probability 1.
\end{lemma}

\begin{proof}
We prove the statement by induction. \(\vb \sim p_\rb \ne 0\) with probability 1 since the probability mass of \(p_\rb\) at \(\vb = 0\) is infinitesimally small.
By the induction hypothesis, \(\vb_1, \vb_2, ..., \vb_k \overset{i.i.d}{\sim} p_\rb\) are linearly independent, i.e. they span a \(k\)-dimensional subspace in \(\RR^m\), \(\DD_k\).
Consider \(\vb_{k + 1} \sim p_\rb\). The probability that \(\vb_{k+1} \in \DD_k\) is infinitesimally small. In fact, %
\(p_\rb\) is finite at all points and the volume of a \(k\)-dimensional linear subspace with respect to the Lebesgue measure defined on \(\RR^m\) is infinitesimally small. Thus, the vectors \(\vb_1, \vb_2, ..., \vb_k, \vb_{k+1}\) are linearly independent and span a \((k+1)\)-dimensional linear subspace in \(\RR^m\).
Hence, we conclude that if \(\vb_1, \vb_2, ... \vb_d \overset{i.i.d}{\sim} p_\rb\), \(\vb_1, \vb_2, ..., \vb_d\) are linearly independent with probability 1.
\end{proof}

\begin{lemma}[Spherically invariant distributions~\citep{lodzimierz1995normal}]
    \label{sph:thm}
    Suppose \(\Xb \) is an $m$-dimensional random vector with spherically symmetric distribution. Then, \(\Xb = R\Ub\), where the random variable \(\Ub \sim \text{Unif}(\UU_m)\) is uniformly distributed on the unit sphere \(\UU_m\subset\RR^m\), \(R \triangleq \|\Xb\| \ge 0\) is real valued, and random variables \((R, \Ub)\) are statistically independent.
\end{lemma}

\begin{lemma}
    \label{cong}
    Consider \(\Wb_i, \Wb_j \sim \text{Unif}(\UU_m)\) where \(\UU_m\) is the unit sphere in \(\RR^m\). Then, \(\left < \Wb_i, \Wb_j\right > \triangleq \left < \wb_i, \Wb_j\right >\) for any \(\wb_i \in \UU_m\) where \(\triangleq\) represents being congruent in distribution.
\end{lemma}

\begin{proof}
First we show that \(\left < \wb_i, \Wb_j\right >\) has the same distribution for all \(\wb_i \in \UU_m\) i.e. \(\left < \wb_i, \Wb_j\right > \sim p_d \; \forall \wb_i \in \UU_m\).

For any orthonormal matrix \(\Ob \in \RR^{m \times m}\),
\begin{align*}
    \left < \wb_i, \Wb_j\right > &\triangleq \left < \wb_i, \Ob\Wb_j\right > && \Wb_j \sim \text{Unif}(\UU_m) \text{ which is a spherical invariant distribution}\\ 
    &\cong \left < \Ob^{\top}\wb_i, \Wb_j\right>
\end{align*}
Since any \(\wb_i, \wb_k \in \UU_m\) are related through an orthonormal transformation, \(\left < \wb_i, \Wb_j\right > \sim p_d \; \forall \wb_i \in \UU_m\).

To show \(\left < \Wb_i, \Wb_j\right > \cong \left < \wb_i, \Wb_j\right >\) for any \(\wb_i \in \UU_m\):

\begin{align*}
    \PP(\left < \Wb_i, \Wb_j\right > = c) &= \int_{\wb_i} \PP(\left < \Wb_i = \wb_i, \Wb_j\right > = c)\PP(\Wb_i = \wb_i)d\wb_i\\
    &= \PP(\left < \wb_i, \Wb_j\right > = c )\int_{\wb_i}\PP(\Wb_i = \wb_i)d\wb_i \quad \text{for any }\wb_i \in \UU_m\\
    & \qquad \qquad \qquad \qquad \qquad \qquad \qquad \qquad \text{By } \left < \wb_i, \Wb_j\right > \sim p_d \; \forall \wb_i \in \UU_m\\
    &= \PP(\left < \wb_i, \Wb_j\right > = c)
\end{align*}

Hence, \(\left < \Wb_i, \Wb_j\right > \triangleq \left < \wb_i, \Wb_j\right >\) for any \(\wb_i \in \UU_m\).

\end{proof}

\begin{lemma}[Lower bounds on determinants of matrices, Corollary 3 in \citep{brent2014bounds}]
\label{det:bound}
If \(\Ab = \II - \Eb \in \RR^{n \times n}\), \(|E_{ij}| \le \epsilon\) for \(1 \le i, j \le n\),\(E_{ii} = 0\) for \(1 \le i \le n\), and \((n - 1)\epsilon \le 1\), then 
\begin{equation*}
    |\Ab| \ge (1 - (n - 1)\epsilon)(1 + \epsilon)^{n - 1}
\end{equation*}
and the inequality is sharp.
A non-sharp lower bound is as follows:
\begin{equation}
    \label{non:sharp:lb}
    |\Ab| \ge 1 - n\epsilon
\end{equation}
Note that the non-sharp bound in (\ref{non:sharp:lb}) holds when the diagonal elements of \(\Eb\), \(E_{ii}\) are non-zero.

\end{lemma}

\localimabound*

\textbf{Remark:} A sharper lower bound for the local IMA contrast is as follows, %
\begin{equation*}   %
    c_{\textsc{ima}}(\fb, \sb) \le \frac{1}{2}(-\log (1 - (d - 1)\epsilon) - (d - 1)\log(1 + \epsilon))
\end{equation*}  
with (high) probability \(\ge 1 - \min \left\{1, \exp(2\log d - \kappa(m -1)\epsilon^2)\right \}\) for \(m \gg d\) and \(\epsilon > 0\).

\begin{proof}

Computing \(c_{\textsc{ima}}(\fb, \sb)\)(~\ref{contrast}) relies on \(\Jb_\fb(\sb)\) being full column-rank. This holds by Lemma~\ref{linind:sphind}. Further, \(|\Jb^{\top}_\fb(\sb)\Jb_\fb(\sb)|\) is finite as all the columns of \(\Jb_\fb(\sb)\), \(\Jb_{\fb,1}(\sb), \Jb_{\fb,2}(\sb), ..., \Jb_{\fb,d}(\sb)\) are non-zero.

By Lemma~\ref{sph:thm}, \(\Jb_\fb(\sb) = \Wb\Db\), where \(\Db = diag(\|\Jb_{\fb,1}(\sb)\|, \|\Jb_{\fb,2}(\sb)\|, ..., \|\Jb_{\fb,d}(\sb)\|)\) and the columns of \(\Wb\), \(\wb_1, \wb_2, ..., \wb_d \sim \text{Unif}(\UU_m)\) where \(\UU_m\) is the unit sphere in \(\RR^m\).

Hence, \(\Wb^{\top}\Wb = \II_d + \Eb\) where \(E_{ii} = 0\). Consider the off-diagonal elements \(E_{ij} = \left < \Wb_i, \Wb_j \right>\) for \(i \ne j\). \(\left < \Wb_i, \Wb_j\right >\) is congruent in distribution to \(\left < \wb_i, \Wb_j\right >\) for any \(\wb_i \in \UU_m\) by Lemma~\ref{cong}.

Consider \(g(\Wb_j) = \left < \wb_i, \Wb_j\right >\) for a given \(\wb_i \in \UU_m\). \(\EE_{\Wb_j}(g(\Wb_j)) = 0\) since \(\Wb_j\) comes from a spherically invariant distribution centered at 0. Further, \(g(.)\) is Lipschitz with \(L = 1\) since \(\|\wb_i\| = 1\).

By Lévy's Lemma~\ref{levy},
\begin{align*}
    \PP(|g(\Wb_j)| \le \epsilon) \ge 1 - \exp(-\kappa(m-1)\epsilon^2) & \text{ for arbitrarily small }\epsilon\,.
\end{align*}

Since \(E_{ij} \triangleq g(\Wb_j)\), 
\begin{align}
\label{levy:subs}
    \PP(|E_{ij}| \le \epsilon) \ge 1 - \exp(-\kappa(m-1)\epsilon^2) & \text{ for arbitrarily small }\epsilon\,.
\end{align}

\begin{align}
    \PP\left (\bigcap_{i, j \in [d], i \ne j}E_{ij} \le \epsilon\right) &= 1 - \PP\left(\bigcup_{i, j \in [d], i \ne j}E_{ij} \ge \epsilon\right) \nonumber\\
    &\ge 1 - \min \left \{1, \Sigma_{i, j \in [d], i \ne j}\PP(E_{ij} \ge \epsilon)\right \} & \text{By union bound of probability} \nonumber\\
    &\ge 1 - \min \left \{1, d^2e^{-\kappa(m-1)\epsilon^2}\right \} & \text{By \ref{levy:subs}}\\ 
    &= 1 - \min \left \{1, e^{2\log d-\kappa(m-1)\epsilon^2}\right \} \label{levy:union}
\end{align}

Hence, for \(m \gg d\), \(\bigcap_{i, j \in [d], i \ne j}E_{ij} \le \epsilon\) with \textit{high} probability.

We write the local IMA contrast \(c_{\textsc{ima}}(\fb, \sb)\) as a function of \(\Wb^{\top}\Wb\) and the column norms of the Jacobian, \(\Jb_\fb(\sb)\) so that we can bound it.

\end{proof}

\subsection{Bound on the global IMA contrast}

\paragraph{Defining non-linear functions as composition of \textit{two} affine functions}

We consider the initial scenario of partitioning the domain of the map, \(\fb: \RR^d \to \RR^m\), into two half-spaces \(\PP^{(0)}\) and \(\PP^{(1)}\) defined as follows. %
Given a non-zero vector $\wb \in \RR^d$ and \(c \in \RR\), 
\[  
\PP^{(0)} = \{\sb\in \RR^d ,\wb^{\top}\sb \le c\} \quad \text{and}\quad 
\PP^{(1)} = \RR^d\setminus \PP^{(0)} =\{\sb \in \RR^d, \wb^{\top}\sb > c\}\,.
\]
To define \(\fb\), we glue together two affine maps across the partition boundary. The affine maps are defined by the following local Jacobians,
\[\Jb_\fb(\sb) = \begin{cases}
\Jb^{(0)}, & \sb \in \PP^{(0)}\\
\Jb^{(1)}, & \sb \in \PP^{(1)}
\end{cases}.\]

As previously mentioned, we locally retain the sampling procedure defined in the previous section---i.e., locally, the columns of the Jacobian of \(\fb, \Jb_\fb(\sb)\) are sampled from a spherically invariant distribution. Let us denote $\Jb_{:,k}$ the $k$-th column of $\Jb$. We sample i.i.d. the columns of matrices \(\Jb^{(0)}\) and \(\Jb^{(0)}\in \RR^{m \times d}\) as follows
\begin{align*}
\Jb^{(0)}_1(\sb), \Jb^{(0)}_2(\sb), ..., \Jb^{(0)}_d(\sb) &\overset{i.i.d}{\sim} p_\rb\,, \\
\Jb^{(1)}_1(\sb), \Jb^{(1)}_2(\sb), ..., \Jb^{(1)}_d(\sb) &\overset{i.i.d}{\sim} p_\rb\,.
\end{align*}
where \(p_\rb\) is a spherically symmetric distribution in \(\RR^m\).

We thereby consider the resulting maps of the form, 
\begin{equation*}
\fb(\sb) = \left. \begin{cases}
\fb^{(0)}(\sb) = \Jb^{(0)}\sb\,, & \wb^{\top}\sb \le c\,, \\
\fb^{(1)}(\sb) = \Jb^{(1)}\sb + \cb_1\,, & \wb^{\top}\sb > c\,.
\end{cases} \right.
\end{equation*}

Before we derive an upper bound on the global IMA contrast, \(C_{\textsc{ima}}(\fb, p_\sb)\), we introduce lemmas to ensure that the function is well-defined and well-behaved at every point in the domain:
\begin{enumerate}
    \item First, we provide the conditions for defining a continuous map \(\fb: \RR^d \to \RR^m, m \gg d\) by composing two affine maps.
    
    \item Second, we discuss how to ensures injectivity of the composition of two affine maps.
    
    \item Finally, we define a smooth approximation to the composition of two affine maps, and show that such approximation is continuously differentiable (in addition to being continuous and injective). 
\end{enumerate}

In the following lemma, we derive the condition required for local bases across the partition boundary, \(\wb^{\top}\sb = c\), given by the Jacobians, \(\Jb^{(0)}, \Jb^{(1)} \in \RR^{m \times d}\), to be able to define a continuous map.

\begin{lemma}[Conditions on local bases for a continuous map]
\label{resampling}
Consider a map \(\fb: \RR^d \to \RR^m, m \gg d\), with full column rank $m\times d$ matrices \(\Jb^{(0)}\) and \(\Jb^{(1)}\) such that
\begin{equation*}
    \fb(\sb) = \left. \begin{cases}
    \fb^{(0)}(\sb) = \Jb^{(0)}\sb\,, & \wb^{\top}\sb \le c\,, \\
    \fb^{(1)}(\sb) = \Jb^{(1)}\sb + \cb_1\,, & \wb^{\top}\sb > c\,,
    \end{cases} \right.
\end{equation*}
where \(\wb \in \RR^d, c \in \RR, \cb_1 \in \RR^m\) are given. The local bases of \(\fb^{(0)}\) and \(\fb^{(1)}\), i.e. the columns of \(\Jb^{(0)}\) and \(\Jb^{(1)}\), are sampled from a spherically invariant distribution, the columns of \(\Jb^{(0)}\) and \(\Jb^{(1)}\) are mutually independent. \(\fb\) is continuous non-linear \textit{only if} \(\mbox{\normalfont colrank} \left [\Jb^{(0)} - \Jb^{(1)} \right ] = 1\).

Further, %
if \( \text{\normalfont colrank}\left [\Jb^{(0)} - \Jb^{(1)} \right ] = 1\), then \(\exists \wb \in \RR^d, c \in \RR, \cb_1 \in \RR^m\) such that \(\fb: \RR^d \to \RR^m, m \gg d\), 
\begin{equation*}
    \fb(\sb) = \left. \begin{cases}
    \fb^{(0)}(\sb) = \Jb^{(0)}\sb\,, & \wb^{\top}\sb \le c\,, \\
    \fb^{(1)}(\sb) = \Jb^{(1)}\sb\,, + \cb_1 & \wb^{\top}\sb > c\,,
    \end{cases} \right.
\end{equation*}
is a continuous non-linear function.

\end{lemma}

\begin{proof}

Consider \(\sb \in \RR^d\) such that \(\wb^{\top}\sb > c\). We define the intersection point \(i(\sb) = \lambda(\sb)\sb, \lambda(\sb) \in (0, 1)\), of the segment between \(\mathbf{0}\) and \(\sb\) with the partition boundary, \(\wb^{\top}\sb = c, c > 0\) such that \(\wb^{\top}(\lambda(\sb)\sb) = c\). Observe that the Jacobian of the map \(\fb: \RR^d \to \RR^m\) is as follows:
\begin{equation*}
    \Jb_\fb(\sb) = \left. \begin{cases}
    \Jb^{(0)}\,, & \wb^{\top}\sb \le c\,, \\
    \Jb^{(1)}\,, & \wb^{\top}\sb > c\,.
    \end{cases} \right.
\end{equation*}
If we assume $\fb$ is continuous, in addition to be by definition \textit{continuously differentiable} within each half-space, the following holds:
\begin{align*}
    \fb(\sb) &=  \int_{\mathbf{0}}^{i(\sb)}\Jb^{(0)}d\sb + \int_{i(\sb)}^{\sb}\Jb^{(1)}d\sb\\
    &= \Jb^{(0)}i(\sb) + \Jb^{(1)}(\sb - i(\sb))
\end{align*}
such that
\begin{equation}
    \label{jacobian:composite}
    \Jb_\fb(\sb) = \Jb^{(1)} + (\Jb^{(0)} - \Jb^{(1)})\frac{\partial i(\sb)}{\partial \sb}, \text{ whenever } \wb^{\top}\sb > c\,.
\end{equation}

The partition boundary in the domain of \(\fb\), is defined by \(\wb^{\top}\sb = c, c > 0\), and is thus a \((d - 1)\)-dimensinoal affine space whose associated vector space is the span (\((d - 1)\)-dimensional) of all \(\frac{\partial i(\sb)}{\partial \sb}\). To obtain the correct Jacobian, \(\Jb_\fb(\sb) = \Jb^{(1)}\), in the half-space \(\wb^{\top}\sb > c\,.\), we need
\begin{equation}
    \label{nullity}
    (\Jb^{(0)} - \Jb^{(1)})\frac{\partial i(\sb)}{\partial \sb} = 0
\end{equation}
Thus, to obtain a \textit{continuous} map \(\fb\), \(\mbox{dim}\left(\mbox{Null} \left [\Jb^{(0)} - \Jb^{(1)} \right ]\right) \geq (d - 1)\).
Moreover, to have a non-linear function, we need the two Jacobian values to be different, such that we require 
\(\mbox{dim}\left(\mbox{Null} \left [\Jb^{(0)} - \Jb^{(1)} \right ]\right) = (d - 1)\). 

By the Rank-Nullity Theorem~\citep{enwiki:1084122554}, 
\begin{equation*}
    \text{colrank}\left ( [ \Jb^{(0)} - \Jb^{(1)} ]\right ) + \text{dim}\left ( \text{Null} [ \Jb^{(0)} - \Jb^{(1)} ]\right) 
    = \# \text{cols. of }[ \Jb^{(0)} - \Jb^{(1)} ] %
\end{equation*}
Leading to 
    \(\text{colrank}\left ( [ \Jb^{(0)} - \Jb^{(1)} ]\right ) = d - (d - 1)
    = 1  %
    \)
Hence, \(\fb\) is continuous non-linear \textit{only if} colrank\( \left [\Jb^{(0)} - \Jb^{(1)} \right ] = 1\).

To show the reverse direction, consider the existence of \(\Jb^{(0)}, \Jb^{(1)} \in \RR^{m \times d}\) s.t. colrank\( \left [\Jb^{(0)} - \Jb^{(1)} \right ] = 1\). Consider the singular value decomposition of \(\left [\Jb^{(0)} - \Jb^{(1)} \right ] = \sigma \ub \wb^{\top}, \ub \in \RR^m, \wb \in \RR^d\). To construct a function, \(\fb: \RR^d \to \RR^m\), we start by partitioning the domain into two half space by a hyperplane normal to \(\wb\), say \(\wb^{\top}\sb = c\) for \(c \in \RR\), leading to the half-spaces definition
\[  
\PP^{(0)} = \{\sb\in \RR^d ,\wb^{\top}\sb \le c\} \quad \text{and}\quad 
\PP^{(1)} = \RR^d\setminus \PP^{(0)} =\{\sb \in \RR^d, \wb^{\top}\sb > c\}\,.
\]

For the partition boundary, \(\KK \coloneqq \left \{ \sb: \wb^{\top}\sb = c\right \}\), \(\left (\Jb^{(0)} - \Jb^{(1)} \right )\sb = \sigma \ub \wb^{\top}\sb = c\sigma\ub\). \(\cb_1 \coloneqq c\sigma\ub \) is a constant vector in \(\RR^m\) such that \(\forall \sb \in \KK, \Jb^{(0)}\sb = \Jb^{(1)}\sb + \cb_1\). \(\fb: \RR^d \to \RR^m, m \gg d\), 
\begin{equation*}
    \fb(\sb) = \left. \begin{cases}
    \fb^{(0)}(\sb) = \Jb^{(0)}\sb\,, & \wb^{\top}\sb \le c\,, \\
    \fb^{(1)}(\sb) = \Jb^{(1)}\sb\,, + \cb_1 & \wb^{\top}\sb > c\,,
    \end{cases} \right.
\end{equation*}
is a continuous function since, for the exhaustive cases for \(\sb \in \RR^d\):

\begin{enumerate}
    \item \(\wb^{\top}\sb = c\)
    
    \(\fb(\sb) = \Jb^{(0)}\sb\), and in limit also equal to \(\Jb^{(1)}\sb + \cb_1\). Since we have shown that the limit of the function is equal to the value assumed by the function, \(\fb: \RR^d \to \RR^m\) is continuous \(\forall \sb: \wb^{\top}\sb = c\) (Theorem 4.6, \citep{rudin1964principles}).

    \item \(\wb^{\top}\sb < c\)
    
    \(\fb(\sb) = \Jb^{(0)}\sb\) is affine in this region, and hence continuous.
    
    \item \(\wb^{\top}\sb > c\)
    
    \(\fb(\sb) = \Jb^{(1)}\sb + \cb_1\) is affine in this region, and hence continuous.
\end{enumerate}

Hence, if \(\exists \Jb^{(0)}, \Jb^{(1)} \in \RR^{m \times d}\) s.t. colrank\( \left [\Jb^{(0)} - \Jb^{(1)} \right ] = 1\), then \(\exists \wb \in \RR^d, c \in \RR, \cb_1 \in \RR^m\) s. t. \(\fb: \RR^d \to \RR^m, m \gg d\), 
\begin{equation*}
    \fb(\sb) = \left. \begin{cases}
    \fb^{(0)}(\sb) = \Jb^{(0)}\sb & \wb^{\top}\sb \le c \\
    \fb^{(1)}(\sb) = \Jb^{(1)}\sb + \cb_1 & \wb^{\top}\sb > c
    \end{cases} \right \}
\end{equation*}
is a continuous function.

\end{proof}
This result leads to one possible way to sample nonlinear mixings  approximately satisfying the global IMA principle.
\clearpage

\begin{mdframed}

\underline{\textbf{Sampling procedure for piecewise affine  continuous maps}}

\begin{observation}[Partition boundaries of continuous maps]
\label{obs:part:boundary}

For functions \(\fb: \RR^d \to \RR^m, m \gg d\), 
\begin{equation*}
    \fb(\sb) = \left. \begin{cases}
    \fb^{(0)}(\sb) = \Jb^{(0)}\sb\,, & \wb^{\top}\sb \le c\,, \\
    \fb^{(1)}(\sb) = \Jb^{(1)}\sb + \cb_1\,, & \wb^{\top}\sb > c\,,
    \end{cases} \right.,
\end{equation*} where \(\Jb^{(0)}, \Jb^{(1)} \in \RR^{m \times d}\), one way to achieve \(\text{\normalfont colrank} \left [\Jb^{(0)} - \Jb^{(1)} \right ] = 1\) is the following:
\begin{enumerate}
    \label{sampling}
    \item Sample the columns of \(\Jb^{(0)}\) independently from the mentioned spherically symmetric distribution \(p_\rb\), \(\Jb^{(0)}_{1}, \Jb^{(0)}_{2}, ..., \Jb^{(0)}_{d} \overset{i.i.d}{\sim} p_\rb\).
    
    \item To construct \(\Jb^{(1)}\), retain any \((d - 1)\) columns of \(\Jb^{(0)}\) and sample the remaining column, \(\Jb^{(1)}_{k} \sim p_\rb\).
\end{enumerate}

Notice that the sampling procedure described above is locally equivalent to the one described in Theorem~\ref{local:ima:thm}. We deliberately retain the same sampling procedure so that we can derive a similar upper bound to the IMA contrast in the case of non-affine functions defined by joining two contiguous affine maps (Definition~\ref{map:2}).

\end{observation}

We will now show that a consequence of this sampling procedure is that the boundary of the partition of the domain is constrained to be axis-aligned. The alignment of the partition boundary \(\wb^{\top}\sb = c\), given by \(\wb\), is determined by the choice of \(\Jb^{(0)}\) and \(\Jb^{(1)}\).

Without loss of generality, consider that the last column is newly sampled~(\ref{sampling}) in \(\Jb^{(1)}\). Consider \(\sb \in \RR^d\) such that \(\wb^{\top}\sb > c\). By ~(\ref{nullity}), for continuous \(\fb\),  \((\Jb^{(0)} - \Jb^{(1)})\frac{\partial i(\sb)}{\partial \sb} = 0\). We thus have the following constraints,
\vspace*{-.1cm}
\begin{equation}\vspace*{-.1cm}
    \label{constraints:axis:aligned}
    \wb^{\top}\left ( \frac{\partial i(\sb)}{\partial \sb} \right ) = \mathbf{0}; \; \left [ \frac{\partial i(\sb)}{\partial \sb}\right]_{d.} = \mathbf{0}
\end{equation}
(\ref{constraints:axis:aligned}) is achieved \textit{iff.} \(\wb\) defines an axis-aligned \((d - 1)\)-dimensional subspace normal to the canonical basis vector associated with the index of the column that changes from \(\Jb^{(0)}\) to \(\Jb^{(1)}\) (here, the last column), i.e.
\vspace*{-.3cm}
\begin{equation*}\vspace*{-.3cm}
    \wb = \alpha[0, 0, ..., 1]^{\top}
\end{equation*}
Thus, the aforementioned sampling procedure of local bases~(\ref{sampling}) (i.e. the columns of \(\Jb^{(0)}\) and \(\Jb^{(1)}\)) to ensure continuity of the resultant manifold, leads to a constraint on the partition of the domain of \(\fb: \RR^d \to \RR^m\), i.e. the partition can only be axis-aligned. 
\end{mdframed}

Later in this chapter (Section~\ref{grid:section}), when will extend this construction to more expressive maps where the partition of the input domain is defined as a grid.
We will show that axis-alignment of the partition still allows some degree of expressivity for the resulting class of maps. Those can approximate a large family of Riemannian manifolds embedded in \(\RR^m\)  isomorphic to the \(d\)-dimensional Euclidean space. 

Hence, consider the following definition of maps \(\fb: \RR^d \to \RR^m, m \gg d\) by composing two affine maps, incorporating the axis alignement contraint, 
\begin{definition}[Maps defined by composing two affine maps]
\label{map:2}
Consider the columns of matrices \(\Jb^{(0)}\) and \(\Jb^{(1)}\) are sampled by the following procedure:
\begin{enumerate}
    \item Sample the $d$ columns of \(\Jb^{(0)}\) independently from the mentioned spherically symmetric distributio, \(p_\rb\), \(\Jb^{(0)}_{:,1}, \Jb^{(0)}_{:,2}, ..., \Jb^{(0)}_{:,d} \overset{i.i.d}{\sim} p_\rb\).
    
    \item To construct \(\Jb^{(1)}\), retain any \((d - 1)\) columns of \(\Jb^{(0)}\) and sample the remaining $k$-th column, \(\Jb^{(1)}_{:,k} \sim p_\rb\). 
\end{enumerate}
Consider the map \(\fb: \RR^d \to \RR^m\), 
\begin{equation*}
    \fb(\sb) = \left. \begin{cases}
    \fb^{(0)}(\sb) = \Jb^{(0)}\sb\,, & s_k \le c\,, \\
    \fb^{(1)}(\sb) = \Jb^{(1)}\sb + \cb_1\,, & s_k > c\,,
    \end{cases} \right.
\end{equation*}
where \( c \in \RR\) is given and $\cb_1\in \RR^m$ is set by continuity at the boundary to \(\cb_1=c\left(\Jb^{(0)}_{:,k}-\Jb^{(1)}_{:,k}\right)\). 
\end{definition}
Note that since the partition boundary for the change of Jacobian is axis-aligned, the map \(\fb: \RR^d \to \RR^m\) can be alternatively written as a sum of coordinate-wise functions, \(\fb(\sb) = \sum^d_{i = 1}\fb_i(s_i)\) for \(\sb = \{s_1, s_2, ..., s_d\}\). Without loss of generality, we describe the case \(k=d\). The coordinate-wise functions \(\fb_i: \RR \to \RR^m\) are then defined as follows: 
\begin{enumerate}
    \item \(\fb_i(s_i) = \Jb^{(0)}_{:,i}s_i = \Jb^{(1)}_{:,i}s_i \forall i \in \{1, 2, ..., (d-1)\}\)
    \item \(\fb_d(s_d) = \begin{cases}
                            \Jb^{(0)}_{:,d}s_d & s_d \le t_d \\
                            \Jb^{(1)}_{:,d}(s_d - t_d) + \Jb^{(0)}_{:,d}t_d & s_d > t_d
                        \end{cases}\)
\end{enumerate}
for \(t_d \in \RR\) determined by \(\Jb^{(0)}, \Jb^{(1)} \in \RR^{m \times d}\).

Next, we show that the functions above defined ( Definition~\ref{map:2}) are injective. The idea is that  injectivity of coordinate-wise functions of \(\fb\), \(\fb_i: \RR \to \RR^m\) that are in direct sum entails of \(\fb(\sb) = \sum^d_{i=1}\fb_i(s_i)\). %
We first show that the subspaces spanned by the images of \(\fb_i: \RR \to \RR^d\) are linearly independent, and thereby are in (internal) direct sum (Definition ~\ref{direct:sum}) with respect to the image of \(\fb: \RR^d \to \RR^m\). Then, we show that the coordinate-wise functions, \(\fb_i: \RR \to \RR^m\) are injective, and conclude that \(\fb: \RR^d \to \RR^m\) is injective.

\begin{definition}[Internal Direct Sum of Subspaces, Chapter 1~\citep{roman2005advanced}]

\label{direct:sum}

Let \(\VV\) be a vector space. We say that \(\VV\) is an (internal) direct sum of the family \(\Fcal = \{\Sbb_i | i \in K \}\) of subspaces of \(\VV\) if every vector \(\vb \in \VV\) can be written in a unique way (except for order) as a finite sum of vectors from the subspaces in \(\Fcal\), that is, if for all \(\vb \in \VV\), 
\[\vb = \ub_1 + \ub_2 +... + \ub_n\]
for \(\ub_i \in \Sbb_i\) and furthermore, if 
\[\vb = \wb_1 + \wb_2 + ... + \wb_n\]
where \(\wb_i \in \Sbb_i\), then \(m = n\) and (after reindexing if necessary) \(\wb_i = \ub_i\) for all \(i = 1, 2, ..., n.\) If \(\Fcal = \{\Sbb_1, \Sbb_2, ..., \Sbb_n\} \) is a finite family, we write \(\VV = \Sbb_1 \bigoplus \Sbb_2 \bigoplus ... \Sbb_n\).

\end{definition}

\begin{lemma}

\label{direct:sum:subspace}
Consider \(\Jb^{(0)}, \Jb^{(1)} \in \RR^{m \times d}\) as sampled in Definition (\ref{map:2}).
With probability $1$, the vector space \(\VV = \text{span}\{\Jb^{(0)}_{:,1} = \Jb^{(1)}_{:,1}, \Jb^{(0)}_{:,2} = \Jb^{(1)}_{:,2}, ..., \Jb^{(0)}_{:,(d-1)} = \Jb^{(1)}_{:,(d-1)}, \Jb^{(0)}_{:,d}, \Jb^{(1)}_{:,d}\}\) is the direct sum of the family \(\Fcal = \{\Sbb_1= \text{span}\{\Jb^{(0)}_{:,1} = \Jb^{(1)}_{:,1}\}, \Sbb_2 = \text{span}\{\Jb^{(0)}_{:,2} = \Jb^{(1)}_{:,2}\}, ..., \Sbb_{d-1}= \text{span}\{\Jb^{(0)}_{:,d-1} = \Jb^{(1)}_{:,d-1}\} , \Sbb_d = \text{span}\{\Jb^{(0)}_{:,d}, \Jb^{(1)}_{:,d}\}\}\).

\end{lemma}

\begin{proof}

From Lemma ~\ref{linind:sphind}, for the scenario \(m \gg d\), the vectors \(\{\Jb^{(0)}_{:,1} = \Jb^{(1)}_{:,1}, \Jb^{(0)}_{:,2} = \Jb^{(1)}_{:,2}, ..., \Jb^{(0)}_{:,(d-1)} = \Jb^{(1)}_{:,(d-1)}, \Jb^{(0)}_{:,d}, \Jb^{(1)}_{:,d}\}\) are non-zero and linearly independent with probability 1. 

Consider \(\vb \in \VV, \ub_1, \wb_1 \in \Sbb_1, \ub_2, \wb_2 \in \Sbb_2, ..., \ub_d, \wb_d \in \Sbb_d\) such that 
\begin{equation*}
    \vb = \ub_1 + \ub_2 + ... + \ub_d, \vb = \wb_1 + \wb_2 + ... + \wb_d
\end{equation*}

By definition ~\ref{direct:sum}, to show \( \VV = \Sbb_1 \bigoplus \Sbb_2 \bigoplus ... \Sbb_d\), we need to show \(\ub_1 = \wb_1, \ub_2 = \wb_2, ..., \ub_d = \wb_d\). 

Let 
\begin{itemize}
    \item \(u_1 = c_1\Jb^{(0)}_{:,1}, w_1 = c'_1\Jb^{(0)}_{:,1}\)
    \item \(u_2 = c_2\Jb^{(0)}_{:,2}, w_2 = c'_2\Jb^{(0)}_{:,2}\) \\
    \(\vdots\)
    \item \(u_{d-1} = c_{d-1}\Jb^{(0)}_{:,d-1}, w_{d-1} = c'_{d-1}\Jb^{(0)}_{:,d-1}\)
    \item \(u_d = c^{(0)}_d\Jb^{(0)}_{:,d} + c^{(1)}_d\Jb^{(1)}_{:,d}, v_d = c^{(0)'}_d\Jb^{(0)}_{:,d} + c^{(1)'}_d\Jb^{(1)}_{:,d}\)
\end{itemize}

\begin{align*}
    &\vb = \ub_1 + \ub_2 + ... + \ub_d = \wb_1 + \wb_2 + ... + \wb_d\\
    &c_1\Jb^{(0)}_{:,1} + c_2\Jb^{(0)}_{:,2} + ...  + c_{d-1}\Jb^{(0)}_{:,d-1} + c^{(0)}_d\Jb^{(0)}_{:,d} + c^{(1)}_d\Jb^{(1)}_{:,d} = \\
    &c'_1\Jb^{(0)}_{:,1} + c'_2\Jb^{(0)}_{:,2} + ... + c_{d-1}\Jb^{(0)}_{:,d-1} + c^{(0)'}_d\Jb^{(0)}_{:,d} + c^{(1)'}_d\Jb^{(1)}_{:,d}\\
    &(c_1 - c'_1)\Jb^{(0)}_{:,1} + (c_2 - c'_2)\Jb^{(0)}_{:,2} + ... + (c_{d-1} - c'_{d-1})\Jb^{(0)}_{:,d-1} \\
    &+ (c^{(0)}_d - c^{(0)'}_d)\Jb^{(0)}_{:,d} + (c^{(1)}_d - c^{(1)'}_d)\Jb^{(1)}_{:,d} = \mathbf{0}
\end{align*}

Since \(\{\Jb^{(0)}_{:,1}, \Jb^{(0)}_{:,2}, ..., \Jb^{(0)}_{:,d-1}, \Jb^{(0)}_{:,d}, \Jb^{(1)}_{:,d}\}\) are nonzero and linearly independent with probability 1 (Lemma ~\ref{linind:sphind}),

\begin{align*}
    (c_1 - c'_1) = (c_2 - c'_2) = ... = (c_{d-1} - c'_{d-1}) = (c^{(0)}_d - c^{(0)'}_d) = (c^{(1)}_d - c^{(1)'}_d) = 0
\end{align*}

Hence, it follows that \(\ub_1 = \wb_1, \ub_2 = \wb_2, ..., \ub_d = \wb_d\) and \( \VV = \Sbb_1 \bigoplus \Sbb_2 \bigoplus ... \Sbb_d\).

\end{proof}

\begin{lemma}

\label{inj:coord:to:full}

Consider maps \(\fb: \RR^d \to \RR^m, m  \gg d\), where \(\fb\) can be written as sum of coordinate-wise functions, \(\fb(\sb) = \sum^d_{i = 1}\fb_i(s_i)\) for \(\sb = \{s_1, s_2, ..., s_d\}\). We define \(\VV = \text{span}(\text{Im}(\fb)), \Sbb_1 = \text{span}(\text{Im}(\fb_1)), \Sbb_2 = \text{span}(\text{Im}(\fb_2)), ..., \Sbb_d = \text{span}(\text{Im}(\fb_d))\), where \(\text{Im}(\fb), \text{Im}(\fb_i)\) denote the images of the functions \(\fb: \RR^d \to \RR^m, \fb_i: \RR \to \RR^m \quad \forall i \in [d]\). If \(\VV = \Sbb_1 \bigoplus \Sbb_2 \bigoplus ... \bigoplus \Sbb_d\), the injectivity of the coordinate-wise functions \(\fb_i: \RR \to \RR^m\) implies the injectivity of \(\fb: \RR^d \to \RR^m\).

\end{lemma}

\begin{proof}

To show injectivity of \(\fb: \RR^d \to \RR^m\), we need to show the following:
\begin{equation}
    \label{injectivity}
    \forall \sb_1, \sb_2 \in \RR^d: \fb(\sb_1) = \fb(\sb_2) \implies \sb_1 = \sb_2
\end{equation}

We show~(\ref{injectivity}) by contradiction. Let
\begin{equation}
    \label{contradiction}
    \exists \sb^{(1)} \ne \sb^{(2)} \in \RR^d \; \text{s.t.} \fb(\sb^{(1)}) = \fb(\sb^{(2)}) 
\end{equation}

Observe that \(\text{Im}(\fb) \subseteq \VV, \text{Im}(\fb_i) \subseteq \Sbb_i \forall i \in [d]\). By Lemma ~\ref{direct:sum:subspace}, the vector space \(\VV\) is the direct sum of the subspaces \(\Sbb_i, \forall i \in [d]\), i.e. \(\VV = \Sbb_1 \bigoplus \Sbb_2 \bigoplus ... \bigoplus \Sbb_d\).

Hence by definition of direct sum (Definition ~\ref{direct:sum}), it follows that 

\begin{align*}
    \fb(\sb^{(1)}) = \fb(\sb^{(1)}) \implies \fb_1(s^{(1)}_1) = \fb_1(s^{(2)}_1), \fb_1(s^{(1)}_2) = \fb_1(s^{(2)}_2), ..., \fb_1(s^{(1)}_d) = \fb_1(s^{(2)}_d)
\end{align*}

By injectivity of the coordinate-wise functions,

\begin{align*}
    \fb_1(s^{(1)}_1) = \fb_1(s^{(2)}_1) & \implies s^{(1)}_1 = s^{(2)}_1\\
    \fb_1(s^{(1)}_2) = \fb_1(s^{(2)}_2) & \implies s^{(1)}_2 = s^{(2)}_2\\
    \vdots
    \fb_1(s^{(1)}_d) = \fb_1(s^{(2)}_d) & \implies s^{(1)}_d = s^{(2)}_d\\
    \implies \sb^{(1)} = \sb^{(2)}
\end{align*}

We arrive at a contradiction to ~(\ref{contradiction}), hence (~\ref{injectivity}) holds. Injectivity of the coordinate-wise functions, \(\fb_1: \RR \to \RR^m,\fb_2: \RR \to \RR^m, ..., \fb_d: \RR \to \RR^m \) implies injectivity of \(\fb: \RR^d \to \RR^m\) as defined in ~\ref{map:2}.

\end{proof}

\begin{lemma}[Injectivity of maps defined as a composition of two affine spaces]

\label{inj:comp:2}

Maps \(\fb: \RR^d \to \RR^m\) defined in ~\ref{map:2} are continuous and injective with probability one.

\end{lemma}

\begin{proof}

\(\fb: \RR^d \to \RR^m\) is continuous by lemma~\ref{resampling}.

By definition (\ref{map:2}), \(\fb(\sb) = \sum^d_{i = 1}f_i(s_i)\) for \(\sb = \{s_1, s_2, ..., s_d\}\). Consider the coordinate-wise functions \(\fb_i: \RR \to \RR^m\):

\begin{enumerate}
    \item \(\fb_i(s_i) = \Jb^{(0)}_{:,i}s_i = \Jb^{(1)}_{:,i}s_i \forall i \in \{1, 2, ..., (d-1)\}\)
    
    \item \(\fb_d(s_d) = \begin{cases}
                            \Jb^{(0)}_{:,d}s_d & s_d \le t_d \\
                            \Jb^{(1)}_{:,d}(s_d - t_d) + \Jb^{(0)}_{:,d}t_d & s_d > t_d
                        \end{cases}\)
\end{enumerate}
for \(t_d \in \RR\).

\begin{enumerate}
    \item \(\fb_i: \RR \to \RR \forall i \in \{1, 2, ..., (d-1)\}\) are injective since they are affine.
    
    \item To show that \(\fb_d: \RR \to \RR^m\) is injective, we need to show the following:
    
    \begin{equation}
    \label{inj:coord}
        \forall s^{(1)}_d, s^{(2)}_d \in \RR: \; \fb_d(s^{(1)}_d) = \fb_d(s^{(2)}_d) \implies s^{(1)}_d = s^{(2)}_d
    \end{equation}
    
    As usual, we show (~\ref{inj:coord}) by contradiction. Let
    
    \begin{equation}
        \label{cont:coord}
        \exists  s^{(1)}_d \ne s^{(2)}_d \in \RR \text{ s.t. }\fb_d(s^{(1)}_d) = \fb_d(s^{(2)}_d)
    \end{equation}
    
    Consider the following cases,
    \begin{enumerate}
        \item \(s^{(1)}_d, s^{(2)}_d \le t_d\)
        Then,
        \begin{align*}
            \fb_d(s^{(1)}_d) &= \fb_d(s^{(2)}_d) \\
            \Jb^{(0)}_{:,d}s^{(1)}_d &= \Jb^{(0)}_{:,d}s^{(2)}_d\\
            s^{(1)}_d &= s^{(2)}_d & \because \Jb^{(0)}_{:,d} \in \RR^m \text{ is non-zero w.p. 1. (Lemma ~\ref{linind:sphind})}
        \end{align*}
        
        Hence, we arrive at a contradiction to (~\ref{cont:coord}) in this case.
        
        \item \(s^{(1)}_d \le t_d, s^{(2)}_d > t_d\)
        Then,
        \begin{align*}
            \fb_d(s^{(1)}_d) &= \fb_d(s^{(2)}_d) \\
            \Jb^{(0)}_{:,d}s^{(1)}_d &= \Jb^{(1)}_{:,d}(s^{(2)}_d - t_d) + \Jb^{(0)}_{:,d}t_d\\
            \Jb^{(0)}_{:,d}(s^{(1)}_d - t_d) &= \Jb^{(1)}_{:,d}(s^{(2)}_d - t_d)\\
            [\Jb^{(0)}_{:,d} -\Jb^{(1)}_{:,d}]\begin{bmatrix}
                s^{(1)}_d - t_d\\
                s^{(2)}_d - t_d
            \end{bmatrix} &= \begin{bmatrix} 0 \\
                                            0\end{bmatrix}\\
            \implies \begin{bmatrix} s^{(1)}_d - t_d\\
                s^{(2)}_d - t_d \end{bmatrix} &= \begin{bmatrix} 0 \\
                                            0\end{bmatrix} \\
                                            &\quad \because \Jb^{(0)}_{:,d}, \Jb^{(1)}_{:,d} \text{ are non-zero}\\
                                        &\quad   \text{ and linearly independent w. p. 1 (Lemma ~\ref{linind:sphind})}\\
            \implies s^{(1)}_d = s^{(2)}_d = t_d
        \end{align*}
        Thus, we arrive at a contradiction since \(s^{(2)}_d > t_d\).
        
        \item \(s^{(1)}_d, s^{(2)}_d > t_d\)
        Then,
        \begin{align*}
            \fb_d(s^{(1)}_d) &= \fb_d(s^{(2)}_d) \\
            \Jb^{(1)}_{:,d}(s^{(1)}_d - t_d) + \Jb^{(0)}_{:,d}t_d &= \Jb^{(1)}_{:,d}(s^{(2)}_d - t_d) + \Jb^{(0)}_{:,d}t_d\\
            \Jb^{(1)}_{:,d}s^{(1)}_d &= \Jb^{(1)}_{:,d}s^{(2)}_d\\
            s^{(1)}_d &= s^{(2)}_d \\
            &\because \Jb^{(1)}_{:,d} \in \RR^m \text{ is non-zero w.p. 1.}\\
            & \text{(Lemma ~\ref{linind:sphind})}
        \end{align*}
        
        Hence, we also arrive at a contradiction to (~\ref{cont:coord}) in this case.
        
    \end{enumerate}
    
    Since we arrive at a contradiction to (~\ref{cont:coord}) in the aforementioned exhaustive and mutually exclusive cases for \(s^{(1)}_d, s^{(2)}_d \in \RR\), we conclude that \(\fb_d: \RR \to \RR^m\) is injective.

\end{enumerate}

Hence, we have shown that the coordinate wise functions, \(f_i: \RR \to \RR^m\) are injective. By Lemma ~\ref{inj:coord:to:full}, \(\fb: \RR^d \to \RR^m\) is injective.

\end{proof}

We now develop a smooth approximation to the map defined in \ref{map:2}.

\begin{definition}[Smooth step function]
\label{smooth:step}

We define the smooth step function as \(\tilde{1}_{\epsilon}: \RR \to \RR\) as
\begin{equation*}
    \tilde{1}_{\epsilon}(s) = \begin{cases} 0\,, & s \le -\epsilon\,,\\
\frac{1}{2}\sin \left (\frac{\pi s}{2\epsilon} \right )  + \frac{1}{2}\,, & -\epsilon < s \le \epsilon\,,\\
1\,, & s > \epsilon\,.
\end{cases}
\end{equation*}
\end{definition}

\begin{definition}[Smoothing composition of affine maps]
\label{def:smooth:2}
Consider maps \(\fb: \RR^d \to \RR^m\) defined in ~\ref{map:2}. Such maps can be written as, 
\begin{equation*}
    \fb(\sb) = (\Jb^{(0)}\sb )1_{\wb^{\top}\sb \le c} + (\Jb^{(1)}\sb + \cb_1)1_{\wb^{\top}\sb > c}
\end{equation*}
where \(\wb \in \RR^d, c \in \RR, \cb_1 \in \RR^m\) are given. 

We define the smoothened version of \(\fb: \RR^d \to \RR^m\) as \(\tilde{\fb}_\epsilon: \RR^d \to \RR^m\) as,
\begin{equation*}
    \tilde{\fb}_\epsilon(\sb) = (\Jb^{(0)}\sb )\tilde{1}_{\epsilon}(c - \wb^{\top}\sb) + (\Jb^{(1)}\sb + \cb_1)\tilde{1}_{\epsilon}(\wb^{\top}\sb - c)
\end{equation*}

Note that since \(\wb \in \RR^d\) is an axis-aligned vector, without loss of generality for \(\wb = \eb_d = (0, 0, ..., 1) \in \RR^d\), the function \(\tilde{\fb}_\epsilon: \RR^d \to \RR^m\) can be defined as a sum of coordinate-wise functions, \(\tilde{\fb}_\epsilon(\sb) = \sum^d_{i=1}\tilde{\fb}_{\epsilon, i}(s_i)\) where \(\sb \in \RR^d = (s_1, s_2, ..., s_d)\):

\begin{enumerate}
    \item \(\tilde{\fb}_{\epsilon, i}(s_i) = \Jb^{(0)}_{:,i}s_i = \Jb^{(1)}_{:,i}s_i \forall i \in \{1, 2, ..., (d-1)\}\)
    
    \item \(\tilde{\fb}_{\epsilon, d}(s_d) = \Jb^{(0)}_{:,d}s_d\tilde{1}_\epsilon(t_d - s_d) + (\Jb^{(1)}_{:,d}(s_d - t_d) + \Jb^{(0)}_{:,d}t_d)\tilde{1}_\epsilon(s_d - t_d)\)
\end{enumerate}
for \(t_d \in \RR\) determined by \(\Jb^{(0)}, \Jb^{(1)} \in \RR^{m \times d}, \wb \in \RR^d, \cb_1 \in \RR^m\).

\end{definition}

Finally, before we present the theorem with the high probability bound on the global IMA contrast, \(C_{\textsc{ima}}(\tilde{\fb}_{\epsilon}, p_\sb)\) for any finite probability density, \(p_\sb\) on \(\RR^d\), we introduce a lemma to show that maps \(\tilde{\fb}_{\epsilon}: \RR^d \to \RR^m\) defined in ~\ref{def:smooth:2} are continuous, injective and continuously differentiable. The objective of the following lemma is to ensure that the Jacobian of \(\tilde{\fb}_{\epsilon}\), \(\Jb_{\tilde{\fb}_{\epsilon}} \in \RR^{m \times d}\) , is well-defined for at all points in the domain of \(\tilde{\fb}_{\epsilon}\) such that the IMA contrast, \(C_{\textsc{ima}}(\tilde{\fb}_{\epsilon}, p_{\sb})\) can be computed for maps, \(\tilde{\fb}_{\epsilon}: \RR^d \to \RR^m\), can be computed with respect to all finite distribuitions, \(p_{\sb}\) on \(\RR^d\).

\begin{lemma}
\label{lemma:smooth:2}

Functions \(\tilde{\fb}_{\epsilon}: \RR^d \to \RR^m\) defined in ~\ref{def:smooth:2} are continuously differentiable in \(\RR^d\), in addition to being continuous and injective, are continuously differentiable with \(\epsilon > 0\) arbitrarily small.

\end{lemma}

\begin{proof}

We show successively that \(\tilde{\fb}_\epsilon: \RR^d \to \RR^m\) is continuous, injective and continuously differentiable. 

\textbf{Continuity of }\(\tilde{\fb}_\epsilon: \RR^d \to \RR^m\)

Consider the coordinate-wise decomposition of \(\tilde{\fb}_\epsilon; \tilde{f}_{\epsilon, i}(\sb) = \Jb^{(0)}_{i,:}\sb \tilde{1}_{\epsilon}(c - \wb^{\top}\sb) + (\Jb^{(1)}_{i,:}\sb + c_{1, i}) \tilde{1}_{\epsilon}(\wb^{\top}\sb - c)  \forall i \in [m]\).

\(\Jb^{(0)}_{i,:}\sb, \Jb^{(1)}_{i,:}\sb + c_{1, i}\) are continuous in \(\sb \in \RR^d\) since they are affine.

Note that \(\tilde{1}_{\epsilon}: \RR \to \RR\) is continuous by definition. \(\tilde{1}_{\epsilon}(c - \wb^{\top}\sb), \tilde{1}_{\epsilon}(\wb^{\top}\sb - c)\) are compositions of a continuous function with affine functions (thereby continuous), and hence are continuous (Theorem 4.9, ~\citep{rudin1964principles}). 

\(\tilde{f}_{\epsilon, i}: \RR^d \to \RR\), being a sum of continuous functions, is continuous for all \(i \in [m]\) (Theorem 4.9,~\citep{rudin1964principles}).

Since the coordinate functions of \(\tilde{\fb}: \RR^d \to \RR^m\), \(\tilde{f}_i: \RR^d \to \RR\) are continuous, \(\tilde{\fb}\) is continuous (Theorem 4.10, ~\citep{rudin1964principles}).

\textbf{Injectivity of }\(\tilde{\fb}_\epsilon: \RR^d \to \RR^m\)

Consider the coordinate-wise functions as defined in ~\ref{def:smooth:2}.

\begin{enumerate}
    \item By Lemma ~\ref{inj:comp:2}, the functions \(\tilde{\fb}_{\epsilon, i} = \fb_i: \RR \to \RR^m \forall i \in \{1, 2, ..., d-1 \} \) are injective.
    
    \item We now show that \(\tilde{\fb}_{\epsilon, d}: \RR \to \RR^m\) is injective.

    \begin{align}
        \tilde{\fb}_{\epsilon, d}(s_d) &= \Jb^{(0)}_{:,d}s_d\tilde{1}_\epsilon(t_d - s_d) + (\Jb^{(1)}_{:,d}(s_d - t_d) + \Jb^{(0)}_{:,d}t_d)\tilde{1}_\epsilon(s_d - t_d) \nonumber \\
        &= \Jb^{(0)}_{:,d}s_d(1 -\tilde{1}_\epsilon(s_d - t_d)) + (\Jb^{(1)}_{:,d}(s_d - t_d) + \Jb^{(0)}_{:,d}t_d)\tilde{1}_\epsilon(s_d - t_d) \nonumber\\ 
        & \quad \because \tilde{1}_\epsilon(s_d) + \tilde{1}_\epsilon(-s_d) = 1 \nonumber \\
        &= \Jb^{(0)}_{:,d}(s_d - (s_d - t_d)\tilde{1}_\epsilon(s_d - t_d)) + \Jb^{(0)}_{:,d}(s_d - t_d)\tilde{1}_\epsilon(s_d - t_d) \nonumber \\
        &= [\Jb^{(0)}_{:,d} \;\Jb^{(0)}_{:,d}]\begin{bmatrix} 
            s_d - (s_d - t_d)\tilde{1}_\epsilon(s_d - t_d)\\
            (s_d - t_d)\tilde{1}_\epsilon(s_d - t_d)
        \end{bmatrix} \label{smooth:comp:fn}
    \end{align}
    
    Define \(\tb_d: \RR \to \RR^d\) such that \(\tb_d(s_d) = \begin{bmatrix} 
            s_d - (s_d - t_d)\tilde{1}_\epsilon(s_d - t_d)\\
            (s_d - t_d)\tilde{1}_\epsilon(s_d - t_d)
            \end{bmatrix}\). To show that \(\tilde{\fb}_{\epsilon, d}: \RR \to \RR^m\) is injective, we need to show the following:

    \begin{equation}
    \label{inj:smooth}
        \forall s^{(1)}_d, s^{(2)}_d \in \RR: \; \tilde{\fb}_{\epsilon, d}(s^{(1)}_d) = \tilde{\fb}_{\epsilon, d}(s^{(2)}_d) \implies s^{(1)}_d = s^{(2)}_d\\
    \end{equation}
    
    As usual, we show (~\ref{inj:smooth}) by contradiction. Let 
    \begin{equation}
        \label{cont:smooth}
        \exists  s^{(1)}_d \ne s^{(2)}_d \in \RR \text{ s.t. }\tilde{\fb}_{\epsilon, d}(s^{(1)}_d) = \tilde{\fb}_{\epsilon, d}(s^{(2)}_d)
    \end{equation}
    then we deduce
    \begin{align*}
        \tilde{\fb}_{\epsilon, d}(s^{(1)}_d) &= \tilde{\fb}_{\epsilon, d}(s^{(2)}_d)\\
        [\Jb^{(0)}_{:,d} \;\Jb^{(0)}_{:,d}]\tb(s^{(1)}_d) &= [\Jb^{(0)}_{:,d} \;\Jb^{(0)}_{:,d}]\tb(s^{(2)}_d)\\
        \implies \tb(s^{(1)}_d) &= \tb(s^{(2)}_d) \\
        &\because [\Jb^{(0)}_{:,d} \;\Jb^{(0)}_{:,d}] \text{ is full column rank, Lemma ~\ref{linind:sphind}}\\
        \begin{bmatrix} 
            s^{(1)}_d - (s^{(1)}_d - t_d)\tilde{1}_\epsilon(s^{(1)}_d - t_d)\\
            (s^{(1)}_d - t_d)\tilde{1}_\epsilon(s^{(1)}_d - t_d)
        \end{bmatrix} &= \begin{bmatrix} 
            s^{(2)}_d - (s^{(2)}_d - t_d)\tilde{1}_\epsilon(s^{(2)}_d - t_d)\\
            (s^{(2)}_d - t_d)\tilde{1}_\epsilon(s^{(2)}_d - t_d)
        \end{bmatrix}\\
        \implies s^{(1)}_d = s^{(2)}_d
    \end{align*}
    
    Hence, we arrive at a contradiction to (~\ref{cont:smooth}). Thereby, \(\tilde{\fb}_{\epsilon, d}: \RR \to \RR^m\) is injective. 
    
\end{enumerate}

Hence, we have shown that the coordinate-wise functions of \(\tilde{\fb}_\epsilon: \RR^d \to \RR^m, \quad \tilde{\fb}_\epsilon(\sb) = \sum^d_{i=1}\tilde{\fb}_{\epsilon, i}(s_i)\) where \(\sb \in \RR^d = (s_1, s_2, ..., s_d)\), given by \(\tilde{\fb}_{\epsilon, i}: \RR \to \RR^m\) are injective. We now show the above statement implies the injectivity of \(\fb: \RR^d \to \RR^m\), \(\fb(\sb)= \sum^d_{k=1}f_k(s_k) \). Observe that by definition ~\ref{map:2}

\begin{enumerate}
    \item \(\Sbb_1\) = span(Im(\(\fb_1\))) = span(\(\Jb^{(0)}_{:,1}\)) = span(\(\Jb^{(1)}_{:,1}\))
    \item \(\Sbb_2\) = span(Im(\(\fb_2\))) = span(\(\Jb^{(0)}_{:,2}\)) = span(\(\Jb^{(1)}_{:,2}\))\\
    \vdots
    \item \(\Sbb_{d-1}\) = span(Im(\(\fb_{d-1}\))) = span(\(\Jb^{(0)}_{:,d-1}\)) = span(\(\Jb^{(1)}_{:,d-1}\))
    \item \(\Sbb_d\) = span(Im(\(\fb_d\))) = span(\(\Jb^{(0)}_{:,1}, \Jb^{(1)}_{:,1}\))
\end{enumerate}

Consider \(\VV = \text{span}(\text{Im}(\fb))\). By Lemma ~\ref{direct:sum:subspace}, \(\VV = \Sbb_1 \bigoplus \Sbb_2 \bigoplus ... \Sbb_d\). Further, by Lemma ~\ref{inj:coord:to:full}, injectivity of \(\fb_i: \RR^d \to \RR^m \forall i \in [d]\) implies injectivity of \(\fb: \RR^d \to \RR^m\).

\textbf{Continuity of derivatives of }\(\tilde{\fb}_\epsilon: \RR^d \to \RR^m\)

Consider the derivatives of \(\tilde{\fb}_\epsilon(\sb)\) with respect to the coordinates of \(\sb = (s_1, s_2, ..., s_d)\).

\begin{enumerate}
    \item By definition ~\ref{def:smooth:2}, \(\tilde{\fb}_{\epsilon, i}(s_i) = \Jb^{(0)}_{:,i}s_i = \Jb^{(1)}_{:,i}s_i \forall i \in \{1, 2, ..., (d-1)\}\). Therefore, \(\frac{\partial \tilde{\fb}_{\epsilon}}{\partial s_i} = \Jb^{(0)}_{:,i} = \Jb^{(1)}_{:,i} \forall i \in \{1, 2, ..., (d-1)\). Thus, the derivatives \(\frac{\partial \tilde{f}_{\epsilon, j}}{\partial s_i} \forall j \in [m], i \in \{1, 2, ..., (d-1) \}\) are continuous. 
    
    \item Consider the derivative of \(\tilde{\fb}_\epsilon: \RR^d \to \RR^m\) with respect to \(s_d\). By (~\ref{smooth:comp:fn}),
    \begin{align*}
        \tilde{\fb}_{\epsilon, d}(s_d) &= [\Jb^{(0)}_{:,d} \;\Jb^{(0)}_{:,d}]\begin{bmatrix} 
            s_d - (s_d - t_d)\tilde{1}_\epsilon(s_d - t_d)\\
            (s_d - t_d)\tilde{1}_\epsilon(s_d - t_d)
        \end{bmatrix}\\
        \tilde{\fb'}_{\epsilon, d}(s_d) &= [\Jb^{(0)}_{:,d} \;\Jb^{(0)}_{:,d}]\begin{bmatrix}
            s_d - \tilde{1}_\epsilon(s_d - t_d) - (s_d - t_d)\tilde{1'}_\epsilon(s_d - t_d)\\
            \tilde{1}_\epsilon(s_d - t_d) + (s_d - t_d)\tilde{1'}_\epsilon(s_d - t_d)
        \end{bmatrix}
    \end{align*}
    where by definition ~\ref{smooth:step} \(\tilde{1'}_\epsilon(s) = \begin{cases} 0 & s \le -\epsilon\\
\frac{1}{2}\cos \left (\frac{\pi s}{2\epsilon} \right )\frac{\pi}{2\epsilon}   & -\epsilon < s \le \epsilon\\
 0 & s > \epsilon\end{cases}\). 
 
  Notice that \(\tilde{1'}_\epsilon: \RR \to \RR\) is continuous in \(\RR.\) \(\tilde{\fb'}_{\epsilon, d}(s_d)\) is continuous since it is composed by a sum and product of continuous functions (Theorem 4.9, ~\citep{rudin1964principles}). Notice also that the term \((s_d - t_d)\tilde{1'}_\epsilon(s_d - t_d) = \frac{1}{2}\cos \left (\frac{\pi (s_d - t_d)}{2\epsilon} \right )\frac{\pi}{2\epsilon}.(s_d - t_d)\) is non-zero only when \(-\epsilon < (s_d - t_d) \le \epsilon\), hence this term is finite even for \(\epsilon > 0\) arbitrarily small. The other terms in \(\tilde{\fb'}_{\epsilon, d}(s_d)\) are also finite be definition. Thus, the derivatives \(\frac{\partial \tilde{f}_{\epsilon, j}}{\partial s_d} \forall j \in [m]\) are continuous for \(\epsilon > 0\) arbitrarily small. 
\end{enumerate}

Since all the partial derivatives of \(\tilde{\fb}_\epsilon: \RR^d \to \RR^m\) are continuous, \(\tilde{\fb}\) is continuously differentiable (Theorem 9.21, ~\citep{rudin1964principles}).

\end{proof}

We now present the theorem that introduces a bound on the global IMA contrast for non-affine maps, \(\tilde{\fb}_\epsilon: \RR^d \to \RR^m, m \gg d\), composed by smoothly joining two affine maps with local bases sampled isotropically as defined here ~\ref{def:smooth:2}.

\begin{theorem}

\label{global:ima:thm:2}
Consider the map \(\tilde{\fb}_\epsilon: \RR^d \to \RR^m\) sampled randomly from the procedure ~\ref{def:smooth:2} and any finite probability density, \(p_\sb\), defined over \(\RR^d\).

Then, for \(\epsilon > 0\) arbitrarily small,  %
\(C_{\textsc{ima}}(\tilde{\fb}_\epsilon, p_\sb) \le \delta \) with (high) probability \(\ge 1 - \min \left\{1, \exp(2\log d - \kappa(m -1)\frac{\delta^2}{d^2})\right \}\) for \(m \gg d\) where \(\delta < \frac{1}{2}\) is arbitrarily small.

\end{theorem}

\begin{proof}

We show that the condition of Theorem ~\ref{local:ima:thm}, the columns of the Jacobian of \(\tilde{\fb}_\epsilon\) are locally sampled isotropically i.e. , is still satisfied for the domain of \(\tilde{\fb}_\epsilon\), i.e. \(\forall \sb \in \RR^d\) almost surely w.r.t finite probability measure, \(p_\sb\) over \(\RR^d\).

Following from ~\ref{def:smooth:2}, consider the following partition of the domain, 
\[ 
\begin{cases}
\sb \in \PP^{(0)} & \iff \wb^{\top}\sb \le c - \epsilon\\
\sb \in \PP^{(1)} & \iff \wb^{\top}\sb > c + \epsilon\\
\sb \in \BB & \iff c - \epsilon < \wb^{\top}\sb \le c + \epsilon
\end{cases} \,. 
\]

\begin{enumerate}
    \item \(\forall \sb \in  \PP^{(0)}, \Jb_\fb(\sb) = \Jb^{(0)},\mbox{ with } \Jb^{(0)}_{:,1}, \Jb^{(0)}_{:,2}, ..., \Jb^{(0)}_{:,d} \overset{i.i.d}{\sim} p_\rb\).
    \item \(\forall \sb \in  \PP^{(1)}, \Jb_\fb(\sb) = \Jb^{(1)}, \mbox{ with }\Jb^{(1)}_{:,1}, \Jb^{(1)}_{:,2}, ..., \Jb^{(1)}_{:,d} \overset{i.i.d}{\sim} p_\rb\). 
    \item The region \(\BB\) sandwiching the boundary of the partitions has arbitrarily small probability measure since:
    
\begin{enumerate}
    \item \(\BB\) is an \(\epsilon\)-sandwich of a \((d - 1)\)-dimensional subspace of a \(d\)-dimensional domain. The Lebesgue measure on \(\BB\) is equal to the volumne element associated with \(\BB\) (3.3, ~\citep{ccinlar2011probability}), thus, \(\lambda(\BB) = \Theta(\epsilon)\)\footnote{\label{theta-notation}Refer to big theta notation \(\Theta(.)\) \href{https://en.wikipedia.org/wiki/Big_O_notation\#Family_of_Bachmann.E2.80.93Landau_notations}{here}.} where \(\lambda(.)\) denotes the Lebesgue measure.
    
    \item \(p_\sb\) is finite at all points.
\end{enumerate}

    Hence, \(p(\BB) = \int_\BB p_\sb\lambda(\sb) = \Theta(\epsilon)\), is arbitrarily small for suitably chosen \(\epsilon\).
    
    To derive a bound on the global IMA contrast of \(\tilde{\fb_\epsilon}\), \(c_{\textsc{ima}}(\tilde{\fb_\epsilon}, p_\sb)\), we need that in region, \(\forall \sb \in \BB\), the value of the local IMA contrast  \(c_{\textsc{ima}}(\tilde{\fb_\epsilon}, \sb)\) is finite. This is equivalent to showing that the Jacobian, \(\Jb_{\tilde{\fb}_\epsilon}\) is full column-rank for all \(\sb \in \BB\). To show this, consider the alternate definition of \(\tilde{\fb}_\epsilon: \RR^d \to \RR^m\) in terms of coordinate-wise functions (Definition ~\ref{def:smooth:2}), \(\tilde{\fb}_\epsilon(\sb) = \sum^d_{i=1}\tilde{\fb}_{\epsilon, i}(s_i)\) where \(\sb \in \RR^d = (s_1, s_2, ..., s_d)\):

    \begin{enumerate}
        \item \(\tilde{\fb}_{\epsilon, i}(s_i) = \Jb^{(0)}_{:,i}s_i = \Jb^{(1)}_{:,i}s_i \forall i \in \{1, 2, ..., (d-1)\}\)
        
        \item \(\tilde{\fb}_{\epsilon, d}(s_d) = \Jb^{(0)}_{:,d}s_d\tilde{1}_\epsilon(t_d - s_d) + (\Jb^{(1)}_{:,d}(s_d - t_d) + \Jb^{(0)}_{:,d}t_d)\tilde{1}_\epsilon(s_d - t_d)\)
    \end{enumerate}
    for \(t_d \in \RR\)

    From the above definition, we see that the first \((d-1)\) columns of the Jacobian, \(\Jb_{\tilde{\fb}_\epsilon}\) are defined as \(\Jb_{\tilde{\fb}_\epsilon, :, i} = \frac{\partial \tilde{\fb}_\epsilon(\sb)}{\partial s_i} = \Jb^{(0)}_{:, i} = \Jb^{(0)}_{:, i}\) for \( i \in \{1, 2, ..., d-1\}\). By Lemma ~\ref{linind:sphind}, the first \((d-1)\) columns of \(\Jb_{\tilde{\fb}_\epsilon}\) are nonzero and linearly independent with probability 1. Consider the \(d\)-th column of \(\Jb_{\tilde{\fb}_\epsilon}\).
    
    \begin{multline*}
        \Jb_{\tilde{\fb}_\epsilon, :, d} = \Jb^{(0)}_{:, d}\tilde{1}_\epsilon(t_d - s_d) 
        + \Jb^{(1)}_{:, d}\tilde{1}_\epsilon(s_d - t_d) -  \Jb^{(0)}_{:, d}s_d\tilde{1'}_\epsilon(t_d - s_d)\\
        + (\Jb^{(1)}_{:, d}(s_d - t_d) + \Jb^{(0)}_{:, d}t_d)\tilde{1'}_\epsilon(s_d - t_d)
        \end{multline*}
        thus
        \begin{multline*}
        \Jb_{\tilde{\fb}_\epsilon, :, d} = \Jb^{(0)}_{:, d}((t_d - s_d)\tilde{1'}_\epsilon(t_d - s_d) \\
        + \tilde{1}_\epsilon(t_d - s_d)) + \Jb^{(1)}_{:, d}((s_d - t_d)\tilde{1'}_\epsilon(s_d - t_d) \\
        + \tilde{1}_\epsilon(s_d - t_d))
    \end{multline*}
    
    Observe that \(\Jb_{\tilde{\fb}_\epsilon, :, d}\) is a linear combination of \(\Jb^{(0)}_{:, d}\) and \(\Jb^{(1)}_{:, d}\). Since by Lemma ~\ref{linind:sphind}, \(\Jb^{(0)}_{:, d}, \Jb^{(1)}_{:, d}\) are nonzero and linearly independent with respect to each other and \(\Jb^{(0)}_{:, i}, \Jb^{(1)}_{:, i} \; \forall i \in \{1, 2, ..., (d-1)\}\), the only possibility for \(\Jb_{\tilde{\fb}_\epsilon}\) to not be full column-rank is 
    
    \begin{align*}
        &\Jb_{\tilde{\fb}_\epsilon, :, d} = \mathbf{0}\\
        &\implies (t_d - s_d)\tilde{1'}_\epsilon(t_d - s_d) + \tilde{1}_\epsilon(t_d - s_d) = (s_d - t_d)\tilde{1'}_\epsilon(s_d - t_d) + \tilde{1}_\epsilon(s_d - t_d) = 0\\
        & \quad \quad \quad \because \Jb^{(0)}_{:, d}, \Jb^{(1)}_{:, d} \text{ are linearly independent.}
    \end{align*}
    
    Consider the function, \(q: \RR \to \RR\) such that \(q(s) = s\tilde{1'}_\epsilon(s) + \tilde{1}_\epsilon(s)\). Observe that \(q(s) \ge 0\) for \(s \ge 0\). Thus, for \( (t_d - s_d)\tilde{1'}_\epsilon(t_d - s_d) + \tilde{1}_\epsilon(t_d - s_d) = (s_d - t_d)\tilde{1'}_\epsilon(s_d - t_d) + \tilde{1}_\epsilon(s_d - t_d) = 0\), we need that \(s_d = t_d\). At \(s_d = t_d\), \(q(s_d - t_d) = q(t_d - s_d) = \frac{1}{2} \ne 0\). Hence, we have shown that \(\Jb_{\tilde{\fb}_\epsilon, :, d} \ne \mathbf{0}\), thereby \(\Jb_{\tilde{\fb}_\epsilon}\) is full column-rank and \(c_{\textsc{ima}}(\tilde{\fb}_\epsilon, \sb)\) is finite for all \(\sb \in \BB\).

\end{enumerate}

Hence, 
\begin{align*}
    C_{\textsc{ima}}(\tilde{\fb}_\epsilon, p_\sb) &= \int_{\RR^d}c_{\textsc{ima}}(\tilde{\fb}_\epsilon, \sb)\;p_\sb d\sb \\
    &= \int_{\PP^{(0)}}c_{\textsc{ima}}(\tilde{\fb}_\epsilon, \sb)\;p_\sb d\sb + \int_{\PP^{(1)}}c_{\textsc{ima}}(\tilde{\fb}_\epsilon, \sb)\;p_\sb d\sb + \int_{\BB}c_{\textsc{ima}}\;p_\sb d\sb\\
    &= \int_{\PP^{(0)}}c_{\textsc{ima}}(\tilde{\fb}_\epsilon, \sb)\;p_\sb d\sb + \int_{\PP^{(1)}}c_{\textsc{ima}}(\tilde{\fb}_\epsilon, \sb)\;p_\sb d\sb + \Theta(\epsilon)\\
    &\approx \int_{\PP^{(0)}}c_{\textsc{ima}}(\tilde{\fb}_\epsilon, \sb)\;p_\sb d\sb + \int_{\PP^{(1)}}c_{\textsc{ima}}(\tilde{\fb}_\epsilon, \sb)\;p_\sb d\sb \quad \text{for }\epsilon\text{ arbitrarily small.}\\
    &\le \max_{\sb \in \RR^d}c_{\textsc{ima}}(\tilde{\fb}_\epsilon, \sb)\int_{\RR^d}p_\sb d\sb\\
    &\le \max_{\sb \in \RR^d}c_{\textsc{ima}}(\tilde{\fb}_\epsilon, \sb) \le \delta  \quad \text{w. p. }\ge 1 - \min \left\{1, \exp(2log d - \kappa(m -1)\frac{\delta^2}{d^2})\right \}\, 
\end{align*}
 by Theorem~\ref{local:ima:thm}.
 
Thus, \(C_{\textsc{ima}}(\tilde{\fb}_\epsilon, p_\sb) \le \delta \) for \(\tilde{\fb}_\epsilon: \RR^d \to \RR^m\) defined in \ref{def:smooth:2} with (high) probability \(\ge 1 - \min \left\{1, \exp(2log d - \kappa(m -1)\frac{\delta^2}{d^2})\right \}\) for \(m \gg d\) where \(\delta < \frac{1}{2}\) is arbitrarily small.

\end{proof}

\paragraph{Defining non-linear functions by gluing \(2^d\) affine functions}

We generalize the above construct for the map, \(\fb: \RR^d \to \RR^m\), where \(\fb\) was constructed by stitching together \textit{two} affine functions. In the following construct, the domain \(\RR^d\) is split into \(2^d\) axis-aligned parts using one split point per coordinate. The map \(\fb: \RR^d \to \RR^m\) is now constructed by stitching together \(2^d\) affine functions defined on each part of the domain, mapping the domain to more complex manifolds embedded in the observation space $\RR^m$.  %

\begin{definition}[Maps defined as a composition of \(2^d\) affine maps on an axis-aligned domain partition]

\label{map:definition:2d}

The maps \(\fb: \RR^d \to \RR^m\) we consider are defined as follows:

\begin{enumerate}
    
    \item Consider the map, \(\fb: \RR^d \to \RR^m\) applied to \(\sb \in \RR^d\).
    
    \item For any \(\tb = (t_1, t_2, ..., t_d) \in \RR^d\),a  partition of the domain of \(\fb\) is defined by the binary vector, \(\bb: \RR^d \to \left \{0, 1 \right \}^d\) where \(\bb_k(\sb) \coloneqq 1_{s_k > t_k} \), \(\RR^d = \PP_{\RR^d} = \bigcup_{\bb \in \{0, 1\}^d}\PP^{(\bb)}\), where \( \PP^{(\bb)} \coloneqq \{\sb \quad | \quad \bb(\sb) = \bb\} \). Note that the partition defined is axis-aligned to the canonical basis in \(\RR^d\). This follows to extend the continuity argument from the two-partition case in Lemma ~\ref{resampling}, observation ~\ref{obs:part:boundary}. 
    
    \item Consider the two matrices, \(\Jb^{(0)}, \Jb^{(1)} \in \RR^{m \times d}\), used to define the Jacobian in each part, \(\PP^{(\bb)} \; \forall \bb \in \{0, 1\}^d\). Sample the columns of \(\Jb^{(0)}, \Jb^{(1)} \) independently from the mentioned spherically symmetric distribution~(\ref{local:ima:thm}) \(p_\rb\), \(\Jb^{(0)}_{1}, \Jb^{(0)}_{2}, ..., \Jb^{(0)}_{d} \overset{i.i.d}{\sim} p_\rb\), \(\Jb^{(1)}_{1}, \Jb^{(1)}_{2}, ..., \Jb^{(1)}_{d} \overset{i.i.d}{\sim} p_\rb\).
    
    \item For \(\sb \in \RR^d\) with \(\bb(\sb) = \bb \in \left \{0, 1 \right \}^d\), \(\Jb_\fb(\sb) = \Jb^{(\bb)}\) such that  \( \left. \begin{cases}
    \Jb^{(\bb)}_{:, k} = \Jb^{(0)}_{:, k} & b_k = 0\\
    \Jb^{(\bb)}_{:, k} = \Jb^{(1)}_{:, k} & b_k = 1
    \end{cases} \right \} \).
    
    Note that this corresponds to the observation~\ref{obs:part:boundary} where changing one column of \(\Jb_\fb(\sb)\) across a partition of the domain results in axis-aligned partitions.
    
    \item \(\fb: \RR^d \to \RR^m\) is defined as: \( \left. \begin{cases}
    \fb(\sb) = \Jb^{(\bb)}(\sb) + \cb^{(\bb)} | \quad  \bb(\sb) = \bb
    \end{cases} \right \} \), where \(\cb^{(\bb)} \in \RR^m \; \forall \; \bb \in \left \{0, 1 \right \}^d\).
    
    \item We show that owing to axis-alignment of chosen partition, \(\PP_{\RR^d}\) and resampling \textit{exactly one} column of the Jacobian, \(\Jb_{\fb}(\sb)\) at the boundary between two parts, \(\fb\) can be written as a product of submanifolds. Consider the functions \( \fb_k(s_k) = \left. \begin{cases}
    \Jb^{(0)}_{:, k}s_k & s_k \le t_k\\
    \Jb^{(1)}_{:, k}(s_k - t_k) + \Jb^{(0)}_{:, k}t_k & s_k > t_k
    \end{cases} \right \} \) which only act on one coordinate of the input, \(\sb = (s_1, s_2, ..., s_d) \in \RR^d\).
    
     Consider again \( \left. \begin{cases}
    \fb(\sb) = \Jb^{(\bb)}(\sb) + \cb^{(\bb)} | \quad \bb(\sb) = \bb
    \end{cases} \right \} \), where \(\cb^{(\bb)} \in \RR^m \; \forall \; \bb \in \left \{0, 1 \right \}^d\), \(\cb^{(\mathbf{0})} = \mathbf{0}, \cb^{(\bb)}, \bb \ne \mathbf{0}\) are completely specified by \(\tb\). \(\fb(\sb)\) can be equivalently written as \(\fb(\sb) = \sum^d_{k=1}\fb_k(s_k)\). Here, we write \(\fb: \RR^d \to \RR^m\) as a product of the submanifolds, or also referred to in latter parts of the note as a sum of the coordinate-wise functions, \(\fb_k: \RR \to \RR^m\).

\end{enumerate}

\end{definition}
In the following results, we show that the Jacobian of maps defined by stitching together \(2^d\) affine functions is well-defined at all points in the domain. We start by showing that maps defined in ~\ref{map:definition:2d} are continuous and injective. We then define a smooth approximation to~\ref{map:definition:2d}, so that the map is differentiable also at the partition boundary given by the partition of domain defined in \(\PP_{\RR^d}\). Futher, we show that the smooth approximation to~\ref{map:definition:2d} is continuous, injective and continuously differentiable, which ensures that the Jacobian is well-defined at all points in the domain and the IMA contrast can be computed. Finally, we present a theorem which bounds the IMA contrast with high probability as the dimension of the observed space grows.

\begin{lemma}[Continuity of composition of \(2^d\) affine maps]
\label{cont:2d}
Consider maps \(\fb: \RR^d \to \RR^m\) as defined in ~\ref{map:definition:2d}. Such a map \(\fb\) is continuous.

\end{lemma}

\begin{proof}
Consider \(\fb: \RR^d \to \RR^m\), \(\fb(\sb)= \sum^d_{k=1}f_k(s_k) \), where \[ \fb_k(s_k) =  \begin{cases}
    \Jb^{(0)}_{:, k}s_k\,, & s_k \le t_k\,,\\
    \Jb^{(1)}_{:, k}(s_k - t_k) + \Jb^{(0)}_{:, k}t_k\,, & s_k > t_k\,.
    \end{cases} \]
    
We show continuity of \(\fb_k: \RR \to \RR^m \; \forall k \in [d]\). For a particular \(k\), consider the cases:
\begin{enumerate}
    \item \(s_k < t_k\):

    \(\fb_k(s_k) = \Jb^{(0)}_{:, k}s_k \) is affine in the entire region and hence, is continuous.
    
    \item \(s_k = t_k\):
    
    \(\fb_k(s_k) = \Jb^{(0)}_{:, k}s_k\), and in limit equal to \(\Jb^{(1)}_{:, k}(s_k - t_k) + \Jb^{(0)}_{:, k}t_k\). For \(s_k = t_k\), we need to show that \(\Jb^{(1)}_{:, k}(s_k - t_k) + \Jb^{(0)}_{:, k}t_k\) converges to \(\Jb^{(0)}_{:, k}s_k\). This is easily seen by substituting \(s_k = t_k\), \(\Jb^{(0)}_{:, k}t_k = \Jb^{(1)}_{:, k}(t_k - t_k) + \Jb^{(0)}_{:, k}t_k\). \(\fb_k(s_k)\) is continuous at \(s_k = t_k\) (Theorem 4.6, \citep{rudin1964principles}).
    
    \item \(s_k > t_k\):
    
    \(\fb_k(s_k) = \Jb^{(1)}_{:, k}(s_k - t_k) + \Jb^{(0)}_{:, k}t_k \) is affine in the entire region and hence, is continuous.
\end{enumerate}

\(\fb: \RR^d \to \RR^m, \fb(\sb)= \sum^d_{k=1}f_k(s_k)\) is continuous since the sum of continuous functions is continuous (Theorem 4.9, \citep{rudin1964principles}).

\end{proof}

To show injectivity of \(\fb: \RR^d \to \RR^m\) defined in ~\ref{map:definition:2d}, we follow a similar approach as in the case of maps defined by joining two affine functions (Definition ~\ref{map:2}, Lemma ~\ref{inj:comp:2}). We start by showing that the images of the coordinate-wise functions, \(\fb_k: \RR \to \RR^m\) which compose \(\fb: \RR^d \to \RR^m, \fb(\sb) = \sum^d_{k=1}\fb_k(s_k)\), are in direct sum with respect to the image of \(\fb\). Then, we show that the coordinate-wise functinos, \(\fb_k\), are injective. Finally, we use that in this scenario the injectivity of the coordinate-wise functions, \(\fb_k\), implies the injectiivity of \(\fb\) to conclude that \(\fb\) is injective (Lemma ~\ref{inj:coord:to:full}).

\begin{lemma}

\label{direct:sum:subspace:2d}
Consider \(\Jb^{(0)}, \Jb^{(1)} \in \RR^{m \times d}\) as sampled in Definition (\ref{map:definition:2d}).
The vector space \(\VV = \text{span}\{\Jb^{(0)}_{:,1}, \Jb^{(1)}_{:,1}, \Jb^{(0)}_{:,2} , \Jb^{(1)}_{:,2}, ..., \Jb^{(0)}_{:,(d-1)}, \Jb^{(1)}_{:,(d-1)}, \Jb^{(0)}_{:,d}, \Jb^{(1)}_{:,d}\}\) is the direct sum of the family \(\Fcal = \{\Sbb_1= \text{span}\{\Jb^{(0)}_{:,1}, \Jb^{(1)}_{:,1}\}, \Sbb_2 = \text{span}\{\Jb^{(0)}_{:,2}, \Jb^{(1)}_{:,2}\}, ..., \Sbb_{d-1}= \text{span}\{\Jb^{(0)}_{:,d-1}, \Jb^{(1)}_{:,d-1}\} , \Sbb_d = \text{span}\{\Jb^{(0)}_{:,d}, \Jb^{(1)}_{:,d}\}\}\).

\end{lemma}

\begin{proof}

From Lemma ~\ref{linind:sphind}, for the scenario \(m \gg d\) (here it is sufficient to have (\(m > 2d\))), the vectors \(\{\Jb^{(0)}_{:,1}, \Jb^{(1)}_{:,1}, \Jb^{(0)}_{:,2}, \Jb^{(1)}_{:,2}, ..., \Jb^{(0)}_{:,(d-1)} = \Jb^{(1)}_{:,(d-1)}, \Jb^{(0)}_{:,d}, \Jb^{(1)}_{:,d}\}\) are non-zero and linearly independent with probability 1. 

Consider \(\vb \in \VV, \ub_1, \wb_1 \in \Sbb_1, \ub_2, \wb_2 \in \Sbb_2, ..., \ub_d, \wb_d \in \Sbb_d\) such that 
\begin{equation*}
    \vb = \ub_1 + \ub_2 + ... + \ub_d, \vb = \wb_1 + \wb_2 + ... + \wb_d
\end{equation*}

By definition ~\ref{direct:sum}, to show \( \VV = \Sbb_1 \bigoplus \Sbb_2 \bigoplus ... \Sbb_d\), we need to show \(\ub_1 = \wb_1, \ub_2 = \wb_2, ..., \ub_d = \wb_d\). 

Let 
\begin{itemize}
    \item \(\ub_1 = c^{(0)}_1\Jb^{(0)}_{:,1} + c^{(1)}_1\Jb^{(1)}_{:,1}, \wb_1 = c^{(0)'}_1\Jb^{(0)}_{:,1} + c^{(1)'}_1\Jb^{(1)}_{:,1}\)
    \item \(\ub_2 = c^{(0)}_2\Jb^{(0)}_{:,2} + c^{(1)}_2\Jb^{(1)}_{:,2}, \wb_2 = c^{(0)'}_2\Jb^{(0)}_{:,2} + c^{(1)'}_2\Jb^{(1)}_{:,2}\) \\
    \(\vdots\)
    \item \(\ub_1 = c^{(0)}_{d-1}\Jb^{(0)}_{:,d-1} + c^{(1)}_{d-1}\Jb^{(1)}_{:,d-1}, \wb_1 = c^{(0)'}_{d-1}\Jb^{(0)}_{:,d-1} + c^{(1)'}_{d-1}\Jb^{(1)}_{:,d-1}\)
    \item \(\ub_d = c^{(0)}_d\Jb^{(0)}_{:,d} + c^{(1)}_d\Jb^{(1)}_{:,d}, \wb_d = c^{(0)'}_d\Jb^{(0)}_{:,d} + c^{(1)'}_d\Jb^{(1)}_{:,d}\)
\end{itemize}

\begin{align*}
    &\vb = \ub_1 + \ub_2 + ... + \ub_d = \wb_1 + \wb_2 + ... + \wb_d\\
    &(c^{(0)}_1\Jb^{(0)}_{:,1} + c^{(1)}_1\Jb^{(1)}_{:,1}) + (c^{(0)}_2\Jb^{(0)}_{:,2} + c^{(1)}_2\Jb^{(1)}_{:,2}) = (c^{(0)'}_1\Jb^{(0)}_{:,1} + c^{(1)'}_1\Jb^{(1)}_{:,1}) + (c^{(0)'}_2\Jb^{(0)}_{:,2} + c^{(1)'}_2\Jb^{(1)}_{:,2})\\
    &+ ...   + (c^{(0)}_d\Jb^{(0)}_{:,d} + c^{(1)}_d\Jb^{(1)}_{:,d}) \quad \quad \quad \quad \quad \quad \quad \quad+ ...   + (c^{(0)'}_d\Jb^{(0)}_{:,d} + c^{(1)'}_d\Jb^{(1)}_{:,d})\\
    &(c^{(0)}_1 - c^{(0)'}_1)\Jb^{(0)}_{:,1} + (c^{(0)}_2 - c^{(0)'}_2)\Jb^{(0)}_{:,2} + ... + (c^{(0)}_d - c^{(0)'}_d)\Jb^{(0)}_{:,d}\\
    & + (c^{(1)}_1 - c^{(1)'}_1)\Jb^{(1)}_{:,1} + (c^{(1)}_2 - c^{(1)'}_2)\Jb^{(1)}_{:,2} + ... + (c^{(1)}_d - c^{(1)'}_d)\Jb^{(1)}_{:,d} = \mathbf{0}
\end{align*}

Since \(\{\Jb^{(0)}_{:,1}, \Jb^{(0)}_{:,2}, ..., \Jb^{(0)}_{:,d}, \Jb^{(1)}_{:,1}, \Jb^{(1)}_{:,2}, ..., \Jb^{(1)}_{:,d}\}\) are nonzero and linearly independent with probability 1 (Lemma ~\ref{linind:sphind}),

\begin{align*}
    (c^{(0)}_1 - c^{(0)'}_1) = (c^{(0)}_2 - c^{(0)'}_2) = ... = (c^{(0)}_d - c^{(0)'}_d)\\
    = (c^{(1)}_1 - c^{(1)'}_1) = (c^{(1)}_2 - c^{(1)'}_2) = ... = (c^{(1)}_d - c^{(1)'}_d) = 0
\end{align*}

Hence, it follows that \(\ub_1 = \wb_1, \ub_2 = \wb_2, ..., \ub_d = \wb_d\) and \( \VV = \Sbb_1 \bigoplus \Sbb_2 \bigoplus ... \Sbb_d\).

\end{proof}

\begin{lemma}[Injectivity of composition of \(2^d\) affine maps]
\label{inj:2d}
Consider maps \(\fb: \RR^d \to \RR^m\) as defined in ~\ref{map:definition:2d}. Such a map \(\fb\) is injective.

\end{lemma}

\begin{proof}

Consider \(\fb: \RR^d \to \RR^m\) written as a sum of coordinate-wise functions (Definition ~\ref{map:definition:2d}), \(\fb(\sb)= \sum^d_{k=1}f_k(s_k) \), where \[ \fb_k(s_k) =  \begin{cases}
    \Jb^{(0)}_{:, k}s_k\,, & s_k \le t_k\,,\\
    \Jb^{(1)}_{:, k}(s_k - t_k) + \Jb^{(0)}_{:, k}t_k\,, & s_k > t_k\,.
    \end{cases}  \]
    
In the following proof, we show injectivity of the coordinate-wise functions \(\fb_k: \RR \to \RR^m \; \forall k \in [d]\) and conclude that \(\fb: \RR^d \to \RR^m\) is injective by Lemma ~\ref{direct:sum:subspace:2d} and Lemma ~\ref{inj:coord:to:full}. For a particular \(k\), to show that \(\fb_k: \RR \to \RR^m\) is injective, we need to show the following:

\begin{equation}
    \label{inj:2d}
    \forall s^{(1)}_k, s^{(2)}_k \in \RR : \quad \fb_k(s^{(1)}_k) = \fb_k(s^{(2)}_k) \implies s^{(1)}_k = s^{(2)}_k
\end{equation}

As usual, we show (~\ref{inj:2d}) by contradiction. Let
\begin{equation}
    \label{contradiction:2d}
    \exists s_k^{(1)} \ne s_k^{(2)} \in \RR s.t. \fb_k(s_k^{(1)}) = \fb_k(s_k^{(2)}) 
\end{equation}

Consider the mutually exclusive, and exhaustive cases:

\begin{enumerate}
    \item \(s_k^{(1)} \ne s_k^{(2)} \in \RR, s_k^{(1)}, s_k^{(2)} \le t_k\)
    
    \begin{align*}
        \fb_k(s_k^{(1)}) &= \fb_k(s_k^{(2)}) \\
        \Jb^{(0)}_{:, k}s_k^{(1)} &= \Jb^{(0)}_{:, k}s_k^{(2)}\\
        \implies s_k^{(1)} &= s_k^{(2)} & \Jb^{(0)} \ne \mathbf{0} \text{ w. p. }1
    \end{align*}
    
    Thus, for \(s_k^{(1)}, s_k^{(2)} \le t_k\), \(\fb_k(s_k^{(1)}) = \fb_k(s_k^{(2)}) \implies s_k^{(1)} = s_k^{(2)}\), which contradicts (\ref{contradiction:2d}).
    
    \item \(s_k^{(1)} \ne s_k^{(2)} \in \RR, s_k^{(1)}, s_k^{(2)} > t_k\)
    
    \begin{align*}
        \fb_k(s_k^{(1)}) &= \fb_k(s_k^{(2)}) \\
        \Jb^{(1)}_{:, k}(s_k^{(1)} - t_k) + \Jb^{(0)}_{:, k}t_k &= \Jb^{(1)}_{:, k}(s_k^{(2)} - t_k) + \Jb^{(0)}_{:, k}t_k\\
        \Jb^{(1)}_{:, k}s_k^{(1)} &= \Jb^{(1)}_{:, k}s_k^{(2)}\\
        \implies s_k^{(1)} &= s_k^{(2)} & \Jb^{(1)} \ne \mathbf{0} \text{ w. p. }1
    \end{align*}
    
    Thus, for \(s_k^{(1)}, s_k^{(2)} > t_k\), \(\fb_k(s_k^{(1)}) = \fb_k(s_k^{(2)}) \implies s_k^{(1)} = s_k^{(2)}\), which contradicts (\ref{contradiction:2d}).
    
    \item \(s_k^{(1)} \ne s_k^{(2)} \in \RR, s_k^{(1)} \le t_k, s_k^{(2)} > t_k \)
    
    \begin{align*}
        \fb_k(s_k^{(1)}) &= \fb_k(s_k^{(2)}) \\
        \Jb^{(0)}_{:, k}s_k^{(1)} &= \Jb^{(1)}_{:, k}(s_k^{(2)} - t_k) + \Jb^{(0)}_{:, k}t_k\\
        \left [\Jb^{(0)}_{:, k}  -\Jb^{(1)}_{:, k}\right ]\begin{bmatrix} s_k^{(1)} - t_k \\
        s_k^{(2)} - t_k\end{bmatrix} &= \mathbf{0} \\
        \implies s_k^{(1)} = s_k^{(2)} = t_k & \quad \because \left [\Jb^{(0)}_{:, k}  -\Jb^{(1)}_{:, k}\right ] \text{ is full column-rank w.p. 1}
    \end{align*}
    We arrive at a contradiction since \(s_k^{(2)} > t_k\).
    
\end{enumerate}

\(\fb_k: \RR \to \RR^m \; \forall k\) is injective for the exhaustive cases for \(s_k^{(1)}, s_k^{(2)} \in \RR\) in the above-mentioned points and hence, is injective.

We now show that the injectivity of the coordinate-wise functions \(\fb_k: \RR \to \RR^d\) implies the injectivity of \(\fb: \RR^d \to \RR^m\), \(\fb(\sb)= \sum^d_{k=1}\fb_k(s_k) \). Observe that by definition ~\ref{map:definition:2d}

\begin{enumerate}
    \item \(\Sbb_1\) = span(Im(\(\fb_1\))) = span(\(\Jb^{(0)}_{:,1}, \Jb^{(1)}_{:,1}\))
    \item \(\Sbb_2\) = span(Im(\(\fb_2\))) = span(\(\Jb^{(0)}_{:,1}, \Jb^{(1)}_{:,1}\))\\
    \vdots
    \item \(\Sbb_d\) = span(Im(\(\fb_d\))) = span(\(\Jb^{(0)}_{:,1}, \Jb^{(1)}_{:,1}\))
\end{enumerate}

Consider \(\VV = \text{span}(\text{Im}(\fb))\). By Lemma ~\ref{direct:sum:subspace:2d}, \(\VV = \Sbb_1 \bigoplus \Sbb_2 \bigoplus ... \Sbb_d\). Further, by Lemma ~\ref{inj:coord:to:full}, injectivity of \(\fb_k: \RR^d \to \RR^m \forall k \in [d]\) implies injectivity of \(\fb: \RR^d \to \RR^m\).

\end{proof}

We proceed to define a smooth approximation to the map, \(\fb: \RR^d \to \RR^m\) defined in ~\ref{map:definition:2d}.

\begin{definition}
    \label{def:smooth:2d}
    Consider the decomposition of \(\fb: \RR^d \to \RR^m\) as a sum of coordinate-wise functions, \(\fb(\sb)= \sum^d_{k=1}f_k(s_k) \), \(\sb = (s_1, s_2, ..., s_d) \in \RR^d\) where \( \fb_k(s_k) = \left. \begin{cases}
    \Jb^{(0)}_{:, k}s_k & s_k \le t_k\\
    \Jb^{(1)}_{:, k}(s_k - t_k) + \Jb^{(0)}_{:, k}t_k & s_k > t_k
    \end{cases} \right \} \).
    \(\fb_k: \RR \to \RR^m\) can be alternatively written as:
    
    \begin{equation*}
        \fb_k(s_k) = (\Jb^{(0)}_{:, k}s_k)1_{s_k \le t_k} + (\Jb^{(1)}_{:, k}(s_k - t_k) + \Jb^{(0)}_{:, k}t_k)1_{s_k > t_k}
    \end{equation*}
    
    We define the smoothened version of \(\fb: \RR^d \to \RR^m\) as \(\tilde{\fb}_\epsilon(\sb) = \sum_{k=1}^d\tilde{\fb}_{\epsilon, k}(s_k)\) for \(\epsilon > 0\) arbitrarily small where
    
    \begin{equation*}
        \tilde{\fb}_{\epsilon, k}(s_k) = (\Jb^{(0)}_{:, k}s_k)\tilde{1}_{\epsilon}(t_k -  s_k) + (\Jb^{(1)}_{:, k}(s_k - t_k) + \Jb^{(0)}_{:, k}t_k)\tilde{1}_{\epsilon}(s_k -  t_k)
    \end{equation*}
\end{definition}

As with the case of defining maps \(\fb: \RR^d \to \RR^m\) by smoothly joining \textit{two} affine maps (Definition ~\ref{def:smooth:2}, Lemma ~\ref{lemma:smooth:2}), before we present the theorem with the high probability bound on the global IMA contrast, \(C_{\textsc{ima}}(\tilde{\fb}_{\epsilon}, p_\sb)\) for any finite probability density, \(p_\sb\) on \(\RR^d\), we introduce a lemma to show that maps \(\tilde{\fb}_{\epsilon}: \RR^d \to \RR^m\) defined in ~\ref{def:smooth:2} are continuous, injective and continuously differentiable. The objective of the following lemma is to ensure that the Jacobian of \(\tilde{\fb}_{\epsilon}\), \(\Jb_{\tilde{\fb}_{\epsilon}} \in \RR^{m \times d}\) , is well-defined for at all points in the domain of \(\tilde{\fb}_{\epsilon}\) such that the IMA contrast, \(C_{\textsc{ima}}(\tilde{\fb}_{\epsilon}, p_{\sb})\) can be computed for maps, \(\tilde{\fb}_{\epsilon}: \RR^d \to \RR^m\), can be computed with respect to all finite distribuitions, \(p_{\sb}\) on \(\RR^d\). 

\begin{lemma}

\label{lemma:smooth:2d}
Functions \(\tilde{\fb}_{\epsilon}: \RR^d \to \RR^m\) defined in ~\ref{def:smooth:2d} are continuously differentiable in \(\RR^d\), in addition to being continuous and injective, with \(\epsilon > 0\) arbitrarily small.

\end{lemma}

\begin{proof}

We show that \(\tilde{\fb}_\epsilon: \RR^d \to \RR^m\) defined in ~\ref{def:smooth:2d} is continuous, injective and continuously differentiable. 

\textbf{Continuity of }\(\tilde{\fb}_\epsilon: \RR^d \to \RR^m\)

Consider the coordinate-wise decomposition of \(\tilde{\fb}_\epsilon; \tilde{f}_{\epsilon, i}(\sb) = \sum^d_{k=1}\tilde{f}_{\epsilon, (i,k)}(s_k)\), where \(\tilde{f}_{\epsilon, (i,k)}: \RR \to \RR\) is defined as \(\tilde{f}_{\epsilon, (i,k)}(s_k) \coloneqq (\Jb^{(0)}_{i, k}s_k)\tilde{1}_{\epsilon}(t_k -  s_k) + (\Jb^{(1)}_{i, k}(s_k - t_k) + \Jb^{(0)}_{i, k}t_k)\tilde{1}_{\epsilon}(s_k -  t_k) \)

\((\Jb^{(0)}_{m, k}s_k), (\Jb^{(1)}_{m, k}(s_k - t_k) + \Jb^{(0)}_{m, k}t_k)\) are continuous in \(s_k \in \RR \; \forall k \in [d]\) since they are affine.

Note that \(\tilde{1}: \RR \to \RR\) is continuous by definition. \(\tilde{1}_{\epsilon}(t_k -  s_k), \tilde{1}_{\epsilon}(s_k -  t_k)\) are compositions of a continuous function with affine functions (thereby continuous), and hence are continuous (Theorem 4.9, ~\citep{rudin1964principles}). 

\(\tilde{f}_{\epsilon, i}: \RR^d \to \RR\), being a sum of continuous functions, is continuous for all \(i \in [m]\) (Theorem 4.9, ~\citep{rudin1964principles}).

Since the coordinate functions of \(\tilde{\fb}: \RR^d \to \RR^m\), \(\tilde{f}_i: \RR^d \to \RR\) are continuous, \(\tilde{\fb}\) is continuous (Theorem 4.10, ~\citep{rudin1964principles}).

\textbf{Injectivity of }\(\tilde{\fb}_\epsilon: \RR^d \to \RR^m\)

Consider the coordinate-wise functions as defined in ~\ref{def:smooth:2d}, \\ \(\tilde{\fb}_\epsilon(\sb) = \sum_{k=1}^d\tilde{\fb}_{\epsilon, k}(s_k), \; \tilde{\fb}_{\epsilon, k}(s_k) = (\Jb^{(0)}_{:, k}s_k)\tilde{1}_{\epsilon}(t_k -  s_k) + (\Jb^{(1)}_{:, k}(s_k - t_k) + \Jb^{(0)}_{:, k}t_k)\tilde{1}_{\epsilon}(s_k -  t_k)\). We now show that all the coordinate-wise functions, \(\tilde{\fb}_{\epsilon, k}: \RR \to \RR^m \; \forall k \in [d]\) are injective. The following proof is the same as the proof of injectivity of \(\tilde{\fb}_{\epsilon, d}: \RR \to \RR^m\) in Lemma ~\ref{lemma:smooth:2}.

 \begin{align}
        \tilde{\fb}_{\epsilon, k}(s_k) &= \Jb^{(0)}_{:,k}s_k\tilde{1}_\epsilon(t_k - s_k) + (\Jb^{(1)}_{:,k}(s_k - t_k) + \Jb^{(0)}_{:,k}t_k)\tilde{1}_\epsilon(s_k - t_k) \nonumber \\
        &= \Jb^{(0)}_{:,k}s_k(1 -\tilde{1}_\epsilon(s_k - t_k)) + (\Jb^{(1)}_{:,k}(s_k - t_k) + \Jb^{(0)}_{:,k}t_k)\tilde{1}_\epsilon(s_k - t_k) \nonumber \\
        & \because \tilde{1}_\epsilon(s_k) + \tilde{1}_\epsilon(-s_k) = 1 \nonumber \\
        &= \Jb^{(0)}_{:,k}(s_k - (s_k - t_k)\tilde{1}_\epsilon(s_k - t_k)) + \Jb^{(0)}_{:,k}(s_k - t_k)\tilde{1}_\epsilon(s_k - t_k) \nonumber \\
        &= [\Jb^{(0)}_{:,k} \;\Jb^{(0)}_{:,k}]\begin{bmatrix} 
            s_k - (s_k - t_k)\tilde{1}_\epsilon(s_k - t_k)\\
            (s_k - t_k)\tilde{1}_\epsilon(s_k - t_k)
        \end{bmatrix} \label{smooth:comp:fn:2d}
    \end{align}
    
    Define \(\tb_k: \RR \to \RR^d\) such that \(\tb_k(s_k) = \begin{bmatrix} 
            s_k - (s_k - t_k)\tilde{1}_\epsilon(s_k - t_k)\\
            (s_k - t_k)\tilde{1}_\epsilon(s_k - t_k)
            \end{bmatrix}\). To show that \(\tilde{\fb}_{\epsilon, k}: \RR \to \RR^m\) is injective, we need to show the following:

    \begin{equation}
    \label{inj:smooth:2d}
        \forall s^{(1)}_k, s^{(2)}_k \in \RR: \; \tilde{\fb}_{\epsilon, k}(s^{(1)}_k) = \tilde{\fb}_{\epsilon, k}(s^{(2)}_k) \implies s^{(1)}_k = s^{(2)}_k\\
    \end{equation}
    
    As usual, we show (~\ref{inj:smooth:2d}) by contradiction. Let 
    \begin{equation}
        \label{cont:smooth:2d}
        \exists  s^{(1)}_k \ne s^{(2)}_k \in \RR \text{ s.t. }\tilde{\fb}_{\epsilon, d}(s^{(1)}_k) = \tilde{\fb}_{\epsilon, k}(s^{(2)}_k)
    \end{equation}
    Then we obtain
    \begin{align*}
        \tilde{\fb}_{\epsilon, k}(s^{(1)}_k) &= \tilde{\fb}_{\epsilon, k}(s^{(2)}_k)\\
        [\Jb^{(0)}_{:,k} \;\Jb^{(0)}_{:,k}]\tb(s^{(1)}_k) &= [\Jb^{(0)}_{:,k} \;\Jb^{(0)}_{:,k}]\tb(s^{(2)}_k)\\
        \implies \tb(s^{(1)}_k) &= \tb(s^{(2)}_k) \\
        & \quad \mbox{ because } [\Jb^{(0)}_{:,k} \;\Jb^{(0)}_{:,k}] \text{ is full column rank (Lemma ~\ref{linind:sphind}})\\
        \begin{bmatrix} 
            s^{(1)}_k - (s^{(1)}_k - t_k)\tilde{1}_\epsilon(s^{(1)}_k - t_k)\\
            (s^{(1)}_k - t_k)\tilde{1}_\epsilon(s^{(1)}_k - t_k)
        \end{bmatrix} &= \begin{bmatrix} 
            s^{(2)}_k - (s^{(2)}_k - t_k)\tilde{1}_\epsilon(s^{(2)}_k - t_k)\\
            (s^{(2)}_k - t_k)\tilde{1}_\epsilon(s^{(2)}_k - t_k)
        \end{bmatrix}\\
        \implies s^{(1)}_k = s^{(2)}_k
    \end{align*}
    
    Hence, we arrive at a contradiction to (~\ref{cont:smooth:2d}). Thereby, \(\tilde{\fb}_{\epsilon, k}: \RR \to \RR^m\) is injective \(\forall k \in [d]\).

    We now show the above statement implies that the injectivity of the coordinate-wise functions \(\tilde{\fb}_{\epsilon, k}: \RR \to \RR^m\) implies the injectivity of \(\tilde{\fb}_{\epsilon}: \RR^d \to \RR^m\), \(\tilde{\fb}_{\epsilon}(\sb)= \sum^d_{k=1}\tilde{\fb}_{\epsilon, k}(s_k) \). Observe that by definition ~\ref{def:smooth:2d}

    \begin{enumerate}
        \item \(\Sbb_1\) = span(Im(\(\tilde{\fb}_{\epsilon, 1}\))) = span(\(\Jb^{(0)}_{:,1}, \Jb^{(1)}_{:,1}\))
        \item \(\Sbb_2\) = span(Im(\(\tilde{\fb}_{\epsilon, 2}\))) = span(\(\Jb^{(0)}_{:,1}, \Jb^{(1)}_{:,1}\))\\
        \vdots
        \item \(\Sbb_d\) = span(Im(\(\tilde{\fb}_{\epsilon, d}\))) = span(\(\Jb^{(0)}_{:,1}, \Jb^{(1)}_{:,1}\))
    \end{enumerate}
    
    Consider \(\VV = \text{span}(\text{Im}(\tilde{\fb}_{\epsilon}))\). By Lemma ~\ref{direct:sum:subspace:2d}, \(\VV = \Sbb_1 \bigoplus \Sbb_2 \bigoplus ... \Sbb_d\). Further, by Lemma ~\ref{inj:coord:to:full}, injectivity of \(\tilde{\fb}_{\epsilon, k}: \RR \to \RR^m \forall k \in [d]\) implies injectivity of \(\tilde{\fb}_{\epsilon}: \RR^d \to \RR^m\).

    \textbf{Continuity of derivatives of }\(\tilde{\fb}_\epsilon: \RR^d \to \RR^m\)

Consider the derivatives of \(\tilde{\fb}_\epsilon(\sb)\) with respect to the coordinates of \(\sb = (s_1, s_2, ..., s_d)\). Since \(\tilde{\fb}_{\epsilon}(\sb)= \sum^d_{k=1}\tilde{\fb}_{\epsilon, k}(s_k)\), \(\frac{\partial \tilde{\fb}_{\epsilon}(\sb)}{\partial s_k} = \frac{d \tilde{\fb}_{\epsilon, k}(s_k)}{d s_k } = \tilde{\fb'}_{\epsilon, k}(s_k)\). The proof of continuity of \(\tilde{\fb'}_{\epsilon, k}(s_k)\) is the same as the proof of continuity of \(\tilde{\fb'}_{\epsilon, d}(s_d)\) as in Lemma ~\ref{lemma:smooth:2} and is rewritten here for easy readability.

 By (~\ref{smooth:comp:fn:2d}),
    
    \begin{align*}
        \tilde{\fb}_{\epsilon, k}(s_k) &= [\Jb^{(0)}_{:,k} \;\Jb^{(1)}_{:,k}]\begin{bmatrix} 
            s_k - (s_k - t_k)\tilde{1}_\epsilon(s_k - t_k)\\
            (s_k - t_k)\tilde{1}_\epsilon(s_k - t_k)
        \end{bmatrix}\\
        \tilde{\fb'}_{\epsilon, k}(s_k) &= [\Jb^{(0)}_{:,k} \;\Jb^{(1)}_{:,k}]\begin{bmatrix} 
            1 - \tilde{1}_\epsilon(s_k - t_k) - (s_k - t_k)\tilde{1'}_\epsilon(s_k - t_k)\\
            \tilde{1}_\epsilon(s_k - t_k) + (s_k - t_k)\tilde{1'}_\epsilon(s_k - t_k)
        \end{bmatrix}\\
    \end{align*}
    
    where by definition ~\ref{smooth:step} \(\tilde{1'}_\epsilon(s) = \begin{cases} 0 & s \le -\epsilon\\
                                          \frac{1}{2}\cos \left (\frac{\pi s}{2\epsilon} \right )\frac{\pi}{2\epsilon}   & -\epsilon < s \le \epsilon\\
                                          0 & s > \epsilon\end{cases}\). 
  Notice that \(\tilde{1'}_\epsilon: \RR \to \RR\) is continuous in \(\RR.\) \(\tilde{\fb'}_{\epsilon, k}(s_k)\) is continuous since it is composed by a sum and product of continuous functions (Theorem 4.9, ~\citep{rudin1964principles}). Notice also that the term \((s_k - t_k)\tilde{1'}_\epsilon(s_k - t_k) = \frac{1}{2}\cos \left (\frac{\pi (s_k - t_k)}{2\epsilon} \right )\frac{\pi}{2\epsilon}.(s_k - t_k)\) is non-zero only when \(-\epsilon < (s_k - t_k) \le \epsilon\), hence this term is finite even for \(\epsilon > 0\) arbitrarily small. The other terms in \(\tilde{\fb'}_{\epsilon, k}(s_k)\) are also finite be definition. Thus, the derivatives \(\frac{\partial \tilde{f}_{\epsilon, i}}{\partial s_k} \forall i \in [m], k \in [d]\) are continuous for \(\epsilon > 0\) arbitrarily small.

Since all the partial derivatives of \(\tilde{\fb}_\epsilon: \RR^d \to \RR^m\) are continuous, \(\tilde{\fb}\) is continuously differentiable (Theorem 9.21, ~\citep{rudin1964principles}).

\end{proof}

We now present the theorem that introduces a bound on the global IMA contrast for non-affine maps, \(\tilde{\fb}_\epsilon: \RR^d \to \RR^m, m \gg d\), composed by smoothly joining \(2^d\) affine maps with local bases sampled isotropically as defined here ~\ref{def:smooth:2d}.

\begin{theorem}

\label{global:ima:thm:2d}
Consider the map \(\tilde{\fb}_\epsilon: \RR^d \to \RR^m\) sampled randomly from the procedure ~\ref{def:smooth:2d}. 

Then, the map \(\tilde{\fb}_\epsilon: \RR^d \to \RR^m\), for \(\epsilon > 0\) arbitrarily small and any finite probability density, \(p_\sb\), defined over \(\RR^d\) satisfies the following bound on the global IMA contrast \(C_{\textsc{ima}}(\tilde{\fb}_\epsilon, p_\sb)\), \(C_{\textsc{ima}}(\tilde{\fb}_\epsilon, p_\sb) \le \delta \) with (high) probability \(\ge 1 - \min \left\{1, \exp(2\log d - \kappa(m -1)\frac{\delta^2}{d^2})\right \}\) for \(m \gg d\) where \(\delta < \frac{1}{2}\) is arbitrarily small.

\end{theorem}

\begin{proof}

We show that the condition of Theorem ~\ref{local:ima:thm}, the columns of the Jacobian of \(\tilde{\fb}_\epsilon\) defined in ~\ref{def:smooth:2d} are locally sampled isotropically i.e. , is still satisfied for the domain of \(\tilde{\fb}_\epsilon\), i.e. \(\forall \sb \in \RR^d\) almost surely w.r.t finite probability measure, \(p_\sb\) over \(\RR^d\).

Following from Definition ~\ref{map:definition:2d} and Definition ~\ref{def:smooth:2d}, consider the partition of the domain of \(\tilde{\fb}_\epsilon: \RR^d \to \RR^m\), \(\RR^d\) into the following regions, 

\begin{enumerate}
    \item \(\text{int}(\PP^{(\bb)}) \coloneqq \{\sb \; | \; \bb(\sb) = \bb, s_i \notin (t_i - \epsilon, t_i + \epsilon ] \; \forall i \in [d]\} \; \forall \bb \in \{0, 1\}^d\)
    
    Notice that by Definition ~\ref{map:definition:2d} and Definition ~\ref{def:smooth:2d}, \(\forall \sb \in \text{int}(\PP^{(\bb)}), \; \forall \bb \in \{0, 1\}^d\), \(\Jb_{\tilde{\fb}_\epsilon}(\sb) = \Jb^{(\bb)}\) (defined in ~\ref{map:definition:2d}), where \(\Jb^{(\bb)}_{1}, \Jb^{(\bb)}_{2}, ..., \Jb^{(\bb)}_{d} \overset{i.i.d}{\sim} p_\rb\). Thus, the condition of Theorem ~\ref{local:ima:thm}, the columns of the Jacobian of \(\tilde{\fb}_\epsilon\) are locally sampled isotropically, is still satisfied for these regions.
    
    \item \(\BB \coloneqq \RR^d - \bigcup_{\bb \in \{0, 1\}^d}\text{int}(\PP^{(\bb)})\)
    
    As in the case where the map, \(\tilde{\fb}_\epsilon: \RR^d \to \RR^m\) was defined as the smooth connection of \textit{two} affine maps (Theorem ~\ref{global:ima:thm:2}), the region \(\BB\) sandwiching the boundary of the partitions has arbitrarily small probability measure since: 
    
    \begin{enumerate}
    \item \(\BB\) is an \(\epsilon\)-sandwich of a \((d - 1)\)-dimensional region of a \(d\)-dimensional domain. The Lebesgue measure on \(\BB\) is equal to the volumne element associated with \(\BB\) (3.3, ~\citep{ccinlar2011probability}), thus, \(\lambda(\BB) = \Theta(\epsilon)\)\footnoteref{theta-notation} where \(\lambda(.)\) denotes the Lebesgue measure.
    
    \item \(p_\sb\) is finite at all points.
\end{enumerate}
    Hence, \(p(\BB) = \int_\BB p_\sb\lambda(\sb) = \Theta(\epsilon)\), is arbitrarily small for suitably chosen \(\epsilon\).
    
    Like in Theorem ~\ref{global:ima:thm:2}, to derive a bound on the global IMA contrast of \(\tilde{\fb_\epsilon}\), \(c_{\textsc{ima}}(\tilde{\fb_\epsilon}, p_\sb)\), we need that in region, \(\forall \sb \in \BB\), the value of the local IMA contrast  \(c_{\textsc{ima}}(\tilde{\fb_\epsilon}, \sb)\) is finite. This is equivalent to showing that the Jacobian, \(\Jb_{\tilde{\fb}_\epsilon}\) is full column-rank for all \(\sb \in \BB\). Consider the definition of \(\tilde{\fb}_\epsilon: \RR^d \to \RR^m\) (Definition ~\ref{def:smooth:grid}) in terms of coordinate-wise functions, \(\fb_{\epsilon, k}: \RR \to \RR^m, \forall k \in [d]\).
    \begin{equation*}
        \tilde{\fb}_{\epsilon, k}(s_k) = \Jb^{(0)}_{:,k}s_k\tilde{1}_\epsilon(t_k - s_k) + (\Jb^{(1)}_{:,k}(s_k - t_k) + \Jb^{(0)}_{:,k}t_k)\tilde{1}_\epsilon(s_k - t_k) \quad \forall k \in [d] 
    \end{equation*}

     Consider the \(k\)-th column of \(\Jb_{\tilde{\fb}_\epsilon}\) for any \(k \in [d]\).
    
    \begin{align*}
        \Jb_{\tilde{\fb}_\epsilon, :, k} &= \Jb^{(0)}_{:, k}\tilde{1}_\epsilon(t_k - s_k) + \Jb^{(1)}_{:, k}\tilde{1}_\epsilon(s_k - t_k) -  \Jb^{(0)}_{:, k}s_k\tilde{1'}_\epsilon(t_k - s_k) + (\Jb^{(1)}_{:, k}(s_k - t_k) + \Jb^{(0)}_{:, k}t_k)\tilde{1'}_\epsilon(s_k - t_k)\\
        &= \Jb^{(0)}_{:, k}((t_k - s_k)\tilde{1'}_\epsilon(t_k - s_k) + \tilde{1}_\epsilon(t_k - s_k)) + \Jb^{(1)}_{:, k}((s_k - t_k)\tilde{1'}_\epsilon(s_k - t_k) + \tilde{1}_\epsilon(s_k - t_k))
    \end{align*}
     Observe that \(\Jb_{\tilde{\fb}_\epsilon, :, k}\) is a linear combination of \(\Jb^{(0)}_{:, k}\) and \(\Jb^{(1)}_{:, k} \quad \forall k \in [d]\). Since by Lemma ~\ref{linind:sphind}, \(\Jb^{(0)}_{:, 1}, \Jb^{(1)}_{:, 1}, \Jb^{(0)}_{:, 2}, \Jb^{(1)}_{:, 2}, ..., \Jb^{(0)}_{:, d}, \Jb^{(1)}_{:, d}\) are all nonzero and linearly independent with respect to each other with probability 1, the only possibility for \(\Jb_{\tilde{\fb}_\epsilon}\) to not be full column-rank is for \(k \in [d]\),
     \begin{align*}
        &\Jb_{\tilde{\fb}_\epsilon, :, k} = \mathbf{0}\\
        &\implies (t_k - s_k)\tilde{1'}_\epsilon(t_k - s_k) + \tilde{1}_\epsilon(t_k - s_k) = (s_k - t_k)\tilde{1'}_\epsilon(s_k - t_k) + \tilde{1}_\epsilon(s_k - t_k) = 0\\
        & \quad \quad \quad \because \Jb^{(0)}_{:, k}, \Jb^{(1)}_{:, k} \text{ are linearly independent.}
    \end{align*}
    Consider the function, \(q: \RR \to \RR\) such that \(q(s) = s\tilde{1'}_\epsilon(s) + \tilde{1}_\epsilon(s)\). Observe that \(q(s) \ge 0\) for \(s \ge 0\). Thus, for \( (t_k - s_k)\tilde{1'}_\epsilon(t_k - s_k) + \tilde{1}_\epsilon(t_k - s_k) = (s_k - t_k)\tilde{1'}_\epsilon(s_k - t_k) + \tilde{1}_\epsilon(s_k - t_k) = 0\), we need that \(s_k = t_k\). At \(s_k = t_k\), \(q(s_k - t_k) = q(t_k - s_k) = \frac{1}{2} \ne 0\). Hence, we have shown that \(\Jb_{\tilde{\fb}_\epsilon, :, k} \ne \mathbf{0} \forall k \in [d]\), thereby \(\Jb_{\tilde{\fb}_\epsilon}\) is full column-rank and \(c_{\textsc{ima}}(\tilde{\fb}_\epsilon, \sb)\) is finite for all \(\sb \in \BB\).

\end{enumerate}

Hence, 
\begin{align*}
    C_{\textsc{ima}}(\tilde{\fb}_\epsilon, p_\sb) &= \int_{\RR^d}c_{\textsc{ima}}(\tilde{\fb}_\epsilon, \sb)\;p_\sb d\sb \\
    &= \bigcup_{\bb \in \{0, 1\}^d}\int_{\text{int}(\PP^{(\bb)})}c_{\textsc{ima}}(\tilde{\fb}_\epsilon, \sb)\;p_\sb d\sb + \int_{\BB}c_{\textsc{ima}}(\tilde{\fb}_\epsilon, \sb)\;p_\sb d\sb\\
    &= \bigcup_{\bb \in \{0, 1\}^d}\int_{\text{int}(\PP^{(\bb)})}c_{\textsc{ima}}(\tilde{\fb}_\epsilon, \sb)\;p_\sb d\sb + \Theta(\epsilon)\\
    &\approx \bigcup_{\bb \in \{0, 1\}^d}\int_{\text{int}(\PP^{(\bb)})}c_{\textsc{ima}}(\tilde{\fb}_\epsilon, \sb)\;p_\sb d\sb \quad \text{for }\epsilon\text{ arbitrarily small.}\\
    &\le \max_{\sb \in \RR^d}c_{\textsc{ima}}(\tilde{\fb}_\epsilon, \sb)\int_{\RR^d}p_\sb d\sb\\
    &\le \max_{\sb \in \RR^d}c_{\textsc{ima}}(\tilde{\fb}_\epsilon, \sb) \le \delta  \quad \text{w. p. }\ge 1 - \min \left\{1, \exp(2log d - \kappa(m -1)\frac{\delta^2}{d^2})\right \}\,, 
\end{align*}
by Theorem~\ref{local:ima:thm}.

Thus, \(C_{\textsc{ima}}(\tilde{\fb}_\epsilon, p_\sb) \le \delta \) for \(\tilde{\fb}_\epsilon: \RR^d \to \RR^m\) defined in \ref{def:smooth:2} with (high) probability at least \( 1 - \min \left\{1, \exp(2\log d - \kappa(m -1)\frac{\delta^2}{d^2})\right \}\) for \(m \gg d\) where \(\delta < \frac{1}{2}\) is arbitrarily small.

\end{proof}

\paragraph{Defining non-linear functions as grid-wise affine functions}

\label{grid:section}

In the previous subsection, we defined smooth nonlinear functions, by smoothly approximating affine functions defined on orthants across a given point (Definition~\ref{map:definition:2d},~\ref{def:smooth:2d}). In this subsection, we extend the previous sampling process for functions 
to consider functions which smoothly approximate functions which are piecewise affine across a grid-like partition of the domain. However, unlike the previous function sampling processes, we now restrict ourselves, in the current sampling process, to define functions on a bounded subset of the \(d\)-dimensional Euclidean space and without loss of generality, we choose our domain to be \([0, 1]^d\). We make this restriction since upon defining a regular grid (with fixed grid width) on an unbounded domain, we would no longer be able to argue that the columns of the Jacobian of the to-be defined sampling process of functions (definition~\ref{map:definition:grid}), \(\fb: \RR^d \to \RR^m\), defining the local bases, would no longer be linearly independent with probability 1, since there are infinitely many of them, and we can no longer apply Lemma ~\ref{linind:sphind}. Recall that we needed the linear independence of the local bases of the function with respect to one another to show injectivity of \(\fb\) (Lemmata ~\ref{inj:comp:2}, ~\ref{lemma:smooth:2}, ~\ref{inj:2d}, ~\ref{lemma:smooth:2d}). In the context of signal processing, the assumption of bounded domain for the sources is often satisfied.

We proceed in the usual fashion, we first use grid-wise affine function and show their continuity and injectvity.  Then, we define a smooth approximation to the grid-wise affine function and show that it is continuous, injective and continuously differentiable, and hence, the Jacobian of the function is defined at all points and the global IMA contrast can be computed. Then, we derive a high probability bound on the global IMA contrast in this scenario such the bound holds with growing probability as the dimensionality of the observed space increases.

Following is the definition of the grid-wise affine function.
\begin{definition}
\label{map:definition:grid}
The maps \(\fb: \RR^d \to \RR^m\) we consider are defined as follows:
\begin{enumerate}
    \item We define a function, \(\fb: [0, 1]^d \to \RR^m\) as a grid-wise affine function, applied to say \(\sb \in [0, 1]^d\).
    
    \item Consider a partition of the domain \([0, 1]^d\) as a regular grid, with grid width \(1 \ge \delta > \epsilon > 0, \delta, \epsilon \in \RR\). The number of grid parts along a dimension, \(k \in [d]\), of the domain \([0, 1]^d\) is therefore equal to \(p = \ceil{\frac{1}{\delta}}{} + 1\) where \(\ceil{.}{}\) is the ceiling function.
    
    \item To define partitions of the domain, \([0, 1]^d\), consider the vector, \(\bb ; [0, 1]^d \to [p]^d, [p] = \{1, 2, ..., p\}\), where \(\bb(\sb) = (b_1(\sb), b_2(\sb), ..., b_d(\sb))\) such that \(b_k(\sb) \coloneqq \ceil{\frac{s_k}{\delta}}{}\). The partition of the domain, \([0, 1]^d\), we consider is defined as \([0, 1]^d = \PP_{[0, 1]^d} = \bigcup_{\bb \in [p]^d}\PP^{(\bb)}\) where \(\PP^{(\bb)} \coloneqq \{\sb \quad | \quad \bb(\sb) = \bb\}\). Note that the partition defined is axis-aligned to the canonical basis in \(\RR^d\). This follows to extend the continuity argument from the two-partition case in Lemma ~\ref{resampling}, observation ~\ref{obs:part:boundary}. 
    
    \item Consider the matrices, \(\Jb^{(1)}, \Jb^{(2)}, ..., \Jb^{(p)} \in \RR^{m \times d}\), used to define the Jacobian in each part, \(\PP^{(\bb)} \; \forall \bb \in [p]^d\). The columns of \(\Jb^{(1)}, \Jb^{(2)}, ..., \Jb^{(p)}\) are sampled from a spherically symmetric distribution, \(p_\rb\), \(\Jb^{(i)}_{1}, \Jb^{(i)}_{2}, ..., \Jb^{(i)}_{d} \overset{i.i.d}{\sim} p_\rb \forall i \in [p]\), so that the pre-condition for Theorem ~\ref{local:ima:thm} holds almost everywhere.
    
    \item For \(\sb \in \RR^d\) with \(\bb(\sb) = \bb \in [p]^d\), \(\Jb_\fb(\sb) = \Jb^{(\bb)}\) such that  \[ \left. \begin{cases}
    \Jb^{(\bb)}_{:, k} = \Jb^{(i)}_{:, k}\,, & b_k = i\,, \; \forall \; k \in [d], i \in [p]
    \end{cases} \right \} \]
    
    Note that this corresponds to Observation~\ref{obs:part:boundary} where changing one column of \(\Jb_\fb(\sb)\) across a partition of the domain results in axis-aligned partitions, also akin to the definition in \ref{map:definition:2d}, except we now have a regular grid where each dimension is split into \(p\) segments rather than orthants across a given point in the domain. 
    
    \item \(\fb: \RR^d \to \RR^m\) is defined as: \( \left. \begin{cases}
    \fb(\sb) = \Jb^{(\bb)}(\sb) + \cb^{(\bb)} | \quad  \bb(\sb) = \bb
    \end{cases} \right \} \), where \(\cb^{(\bb)} \in \RR^m \; \forall \; \bb \in [p]^d\).
    
    \item Owing to the axis-alignment of the chosen partition, \(\PP_[0, 1]^d\), \(\fb: \RR^d \to \RR^m\) can be decomposed as a sum of functions acting upon individual coordinates (also called as coordinate-wise functions), \(\fb(\sb) = \sum^d_{k=1}\fb_k(s_k) \; \forall \; k \in [d]\) where \(\fb_k: [0, 1] \to \RR^m\) 
    is defined as,

    \begin{equation}
        \label{coord:wise:grid}
        \fb_k(s_k) \coloneqq  \sum^p_{t=1}(\Jb^{(t)}_{:, k}(s_k - (t-1)\delta) + \sum^{t-1}_{i=1}\Jb^{(i)}_{:,k}\delta)1_{s_k \in ((t-1)\delta, t\delta]}
    \end{equation}

\end{enumerate}

\end{definition}

We now show that maps, \(\fb: \RR^d \to \RR^m\) defined in ~\ref{map:definition:grid} are continuous and injective.

\begin{lemma}[Continuity of grid-wise affine functions]
\label{cont:grid}
Consider maps \(\fb: \RR^d \to \RR^m\) as defined in ~\ref{map:definition:grid}. Such a map \(\fb\) is continuous.

\end{lemma}

\begin{proof}
Consider \(\fb: [0, 1]^d \to \RR^m\), \(\fb(\sb)= \sum^d_{k=1}f_k(s_k) \), where \(\fb_k(s_k) =  \Jb^{b_k(\sb)}_{:,k}(s_k - (b_k(\sb) - 1)\delta) +  \sum^{b_k(\sb)-1}_{i=1}\Jb^{(i)}_{:, k}(\delta) \; \forall \; k \in [d]\).

We show continuity of \(\fb_k: \RR \to \RR^m \; \forall k \in [d]\). For a particular \(k\), consider the cases:

\begin{enumerate}
    \item \(\BB \coloneqq \{ s_k = t\delta, \; t \in [p-1]\}\)
    
    Let \(s_k = t\delta, \; t \in [p-1], b_k(\sb) = t\), by (~\ref{coord:wise:grid}), \(\fb_k(s_k) = \sum^{t}_{i=1}\Jb^{(i)}_{:,k}\delta\), which is also equal to \(\fb_k(s_k)\) in the left limit. To show that \(\fb_k: [0, 1] \to \RR^m \) is continuous at \(s_k = t\delta\), it remains to show that the right limit of \(\fb_k(s_k)\) is equal to the value of the function at \(s_k = t\delta\). By (~\ref{coord:wise:grid}), at the right limit of \(s_k = t\delta\), \(\fb_k(s_k) = (\Jb^{(t+1)}_{:, k}(t\delta - t\delta) + \sum^{t}_{i=1}\Jb^{(i)}_{:,k}\delta) = \sum^{t}_{i=1}\Jb^{(i)}_{:,k}\delta = \fb_k(s_k)\). Hence, \(\fb_k: [0, 1] \to \RR^m\) is continuous as \(s_k = t\delta, t \in [p-1]\).

    \item \(s_k \in [0, 1] - \BB\)
    
    For these values of \(s_k\), \(\fb_k: [0, 1] \to \RR^m\) is affine. Let \(\ceil{\frac{s_k}{\delta}}{} = t\), \(\fb_k(s_k) = \Jb^{(t)}_{:, k}(s_k - (t-1)\delta) + \sum^{t-1}_{i=1}\Jb^{(i)}_{:,k}\delta\). Since, \(\fb_k(s_k)\) is affine for \(s_k\) in this region, \(\fb_k(s_k)\) is continuous in this region.

\end{enumerate}

\(\fb: \RR^d \to \RR^m, \fb(\sb)= \sum^d_{k=1}f_k(s_k)\) is continuous since the sum of continuous functions is continuous (Theorem 4.9, \citep{rudin1964principles}).
\end{proof}

To show that \(\fb: [0, 1]^d \to \RR^m\) as defined in ~\ref{map:definition:grid} is injective, we follow an analogous approach to the previous cases where non-linear functions were defined as a composition of two affine functions and \(2^d\) affine functions respectively. We first show that the images of the coordinate-wise functions of \(\fb\), \(\fb_k: [0, 1] \to \RR^m\) are in direct sum with respect to the image of \(\fb: [0, 1] \to \RR^m\). Then, we show that the coordinate-wise functions, \(\fb_k\) are injective and use the direct sum property to conclude that injectivity of \(\fb\) is implied.

\begin{lemma}
\label{direct:sum:subspace:grid}

Consider \(\Jb^{(1)}, \Jb^{(2)}, ..., \Jb^{(p)} \in \RR^{m \times d}\) as sampled in Definition ~\ref{map:definition:grid}. The vector space \(\VV = \text{span}\{\bigcup^p_{i=1}\text{cols}(\Jb^{(i)})\}\) is the direct sum of the family \begin{multline*}
    \Fcal = \{\Sbb_1= \text{span}\{\Jb^{(1)}_{:,1}, \Jb^{(2)}_{:,1}, ..., \Jb^{(p)}_{:,1}\}, \Sbb_2 = \text{span}\{\Jb^{(1)}_{:,2}, \Jb^{(2)}_{:,2}, ..., \Jb^{(p)}_{:,2}\}, \\..., \Sbb_{d-1}= \text{span}\{\Jb^{(1)}_{:,d-)}, \Jb^{(2)}_{:,d-1}, ..., \Jb^{(p)}_{:,d-1}\} , \Sbb_d = \text{span}\{\Jb^{(1)}_{:,d}, \Jb^{(2)}_{:,d}, ..., \Jb^{(p)}_{:,d}\}\}\,,
\end{multline*}
where cols(.) denotes the set of columns of a given matrix.

\end{lemma}

\begin{proof}

From Lemma ~\ref{linind:sphind}, for the scenario \(m \gg d\) (here it is sufficient to have (\(m > p.d\))), the set of vectors \(\{\bigcup^p_{i=1}\text{cols}(\Jb^{(i)})\}\) are non-zero and linearly independent with probability 1. 

Consider \(\vb \in \VV, \ub_1, \wb_1 \in \Sbb_1, \ub_2, \wb_2 \in \Sbb_2, ..., \ub_d, \wb_d \in \Sbb_d\) such that 
\begin{equation*}
    \vb = \ub_1 + \ub_2 + ... + \ub_d, \vb = \wb_1 + \wb_2 + ... + \wb_d
\end{equation*}

By definition ~\ref{direct:sum}, to show \( \VV = \Sbb_1 \bigoplus \Sbb_2 \bigoplus ... \Sbb_d\), we need to show \(\ub_1 = \wb_1, \ub_2 = \wb_2, ..., \ub_d = \wb_d\). 

Let 
\begin{itemize}
    \item \(\ub_1 =  \sum^p_{i=1}c^{(i)}_1\Jb^{(i)}_{:,1}, \; \wb_1 =  \sum^p_{i=1}c^{(i)'}_1\Jb^{(i)}_{:,1}\)
    \item \(\ub_2 =  \sum^p_{i=1}c^{(i)}_2\Jb^{(i)}_{:,2}, \; \wb_2 =  \sum^p_{i=1}c^{(i)'}_2\Jb^{(i)}_{:,2}\)\\
    \(\vdots\)
    \item \(\ub_{d-1} =  \sum^p_{i=1}c^{(i)}_{d-1}\Jb^{(i)}_{:,d-1}, \; \wb_2 =  \sum^p_{i=1}c^{(i)'}_{d-1}\Jb^{(i)}_{:,d-1}\)
    \item \(\ub_{d} =  \sum^p_{i=1}c^{(i)}_{d}\Jb^{(i)}_{:,d}, \; \wb_2 =  \sum^p_{i=1}c^{(i)'}_{d}\Jb^{(i)}_{:,d}\)
\end{itemize}

\begin{align*}
    &\vb = \ub_1 + \ub_2 + ... + \ub_d = \wb_1 + \wb_2 + ... + \wb_d\\
    &\sum^p_{i=1}c^{(i)}_1\Jb^{(i)}_{:,1} + \sum^p_{i=1}c^{(i)}_2\Jb^{(i)}_{:,2} + ... + \sum^p_{i=1}c^{(i)}_{d}\Jb^{(i)}_{:,d} = \sum^p_{i=1}c^{(i)'}_1\Jb^{(i)}_{:,1} + \sum^p_{i=1}c^{(i)'}_2\Jb^{(i)}_{:,2} + ... + \sum^p_{i=1}c^{(i)'}_{d}\Jb^{(i)}_{:,d}\\
    &\sum^p_{i=1}(c^{(i)}_1 - c^{(i)'}_1)\Jb^{(i)}_{:,1} + \sum^p_{i=1}(c^{(i)}_2 - c^{(i)'}_2)\Jb^{(i)}_{:,2} + ... + \sum^p_{i=1}(c^{(i)}_d - c^{(i)'}_d)\Jb^{(i)}_{:,d} = \mathbf{0}
\end{align*}

Since the set of vectors \(\bigcup^p_{i=1}\text{cols}(\Jb^{(i)})\) are nonzero and linearly independent with probability 1 (Lemma ~\ref{linind:sphind}),

\begin{align*}
    (c^{(i)}_1 - c^{(i)'}_1) = (c^{(i)}_2 - c^{(i)'}_2) = ... = (c^{(i)}_d - c^{(i)'}_d) = 0\quad \forall i \in [p]
\end{align*}

Hence, it follows that \(\ub_1 = \wb_1, \ub_2 = \wb_2, ..., \ub_d = \wb_d\) and \( \VV = \Sbb_1 \bigoplus \Sbb_2 \bigoplus ... \Sbb_d\).

\end{proof}

\begin{lemma}[Injectivity of grid-wise affine maps]
\label{inj:2d}
Consider maps \(\fb: \RR^d \to \RR^m\) as defined in ~\ref{map:definition:grid}. Such a map \(\fb\) is injective.

\end{lemma}

\begin{proof}

Consider \(\fb: [0, 1]^d \to \RR^m\) written as a sum of coordinate-wise functions (Definition ~\ref{map:definition:grid}), \(\fb(\sb)= \sum^d_{k=1}f_k(s_k) \), where \( \fb_k(s_k) =   \sum^p_{t=1}(\Jb^{(t)}_{:, k}(s_k - (t-1)\delta) + \sum^{t-1}_{i=1}\Jb^{(i)}_{:,k}\delta)1_{s_k \in ((t-1)\delta, t\delta]}\).

In the following proof, we show injectivity of the coordinate-wise functions \(\fb_k: [0, 1] \to \RR^m \; \forall k \in [d]\) and conclude that \(\fb: [0, 1]^d \to \RR^m\) is injective by Lemma ~\ref{direct:sum:subspace:grid} and Lemma ~\ref{inj:coord:to:full}. For a particular \(k\), to show that \(\fb_k: [0, 1] \to \RR^m\) is injective, we need to show the following:

\begin{equation}
    \label{inj:grid}
    \forall s^{(1)}_k, s^{(2)}_k \in [0, 1] : \quad \fb_k(s^{(1)}_k) = \fb_k(s^{(2)}_k) \implies s^{(1)}_k = s^{(2)}_k
\end{equation}

As usual, we show (~\ref{inj:grid}) by contradiction. Let
\begin{equation}
    \label{contradiction:grid}
    \exists s_k^{(1)} \ne s_k^{(2)} \in [0, 1] s.t. \fb_k(s_k^{(1)}) = \fb_k(s_k^{(2)}) 
\end{equation}
Observe that \(\fb_k(s_k) \in \Sbb_k = \text{span}(\Jb^{(1)}_{:,k}, \Jb^{(2)}_{:,k}, ..., \Jb^{(p)}_{:,k})\). We define the coefficient vector for a given \(s_k \in \RR, \ceil{\frac{s_k}{\delta}}{} = t+1\), \(\tb: [0, 1] \to [0, 1]^p, \tb(s_k) \coloneqq \sum^t_{i=1}\delta\eb_i + (s_k - t\delta)\eb_{t+1} \), \(\eb_i\) denotes the \(i\)-th canonical orthonormal basis vector in \(\RR^p\).
Then we get
\begin{align*}
    \fb_k(s_k^{(1)}) &= \fb_k(s_k^{(2)})\\
    [\Jb^{(1)}_{:,k} \; \Jb^{(2)}_{:,k} \; ... \;\Jb^{(p)}_{:,k}]\tb(s_k^{(1)}) &= [\Jb^{(1)}_{:,k} \; \Jb^{(2)}_{:,k} \; ... \;\Jb^{(p)}_{:,k}]\tb(s_k^{(2)})\\
    \implies \tb(s_k^{(1)}) &= \tb(s_k^{(2)})
    \end{align*}
    because \(\Jb^{(1)}_{:,k}, \Jb^{(2)}_{:,k}, ..., \Jb^{(p)}_{:,k}\)  are linearly independent w. p. 1, \text{Lemma ~\ref{linind:sphind}}
    This implies
    \begin{align*}
\ceil{\frac{s^{(1)}_k}{\delta}}{} = \ceil{\frac{s^{(2)}_k}{\delta}}{} = t+1 &, s^{(1)}_k - t\delta = s^{(2)}_k - t\delta \quad \quad \text{ for some }t \in \NN\\
    \implies  s^{(1)}_k = s^{(2)}_k\,.
\end{align*}

Thus, we arrive at a contradiction to (~\ref{contradiction:grid}), and therefore (~\ref{inj:grid}) holds, and \(\fb_k: [0, 1] \to \RR^m\) defined as in ~\ref{map:definition:grid} is injective \(\forall k \in [d]\).

We now show that the injectivity of the coordinate-wise functions \(\fb_k: [0, 1] \to \RR^d\) implies the injectivity of \(\fb: [0, 1]^d \to \RR^m\), \(\fb(\sb)= \sum^d_{k=1}\fb_k(s_k) \). Observe that by definition ~\ref{map:definition:grid}

\begin{enumerate}
    \item \(\Sbb_1\) = span(Im(\(\fb_1\))) = span(\(\Jb^{(1)}_{:,1}, \Jb^{(2)}_{:,1}, ..., \Jb^{(p)}_{:,1}\))
    \item \(\Sbb_2\) = span(Im(\(\fb_2\))) = span(\(\Jb^{(1)}_{:,2}, \Jb^{(2)}_{:,2}, ..., \Jb^{(p)}_{:,2}\))\\
    \vdots
    \item \(\Sbb_d\) = span(Im(\(\fb_d\))) = span(\(\Jb^{(1)}_{:,d}, \Jb^{(2)}_{:,d}, ..., \Jb^{(p)}_{:,d}\))
\end{enumerate}

Consider \(\VV = \text{span}(\text{Im}(\fb))\). By Lemma ~\ref{direct:sum:subspace:grid}, \(\VV = \Sbb_1 \bigoplus \Sbb_2 \bigoplus ... \Sbb_d\). Further, by Lemma ~\ref{inj:coord:to:full}, injectivity of \(\fb_k: [0, 1]^d \to \RR^m \forall k \in [d]\) implies injectivity of \(\fb: [0, 1]^d \to \RR^m\). 

\end{proof}
We proceed to define a smooth approximation to \(\fb: [0, 1]^d \to \RR^m\) defined in ~\ref{map:definition:grid}.
\begin{definition}[Smooth approximation to grid-wise affine maps]
 \label{def:smooth:grid}
    Consider the decomposition of \(\fb: [0, 1]^d \to \RR^m\) as a sum of coordinate-wise functions, \(\fb(\sb)= \sum^d_{k=1}f_k(s_k) \), \(\sb = (s_1, s_2, ..., s_d) \in [0, 1]^d\) where 
    \begin{equation*}
        \fb_k(s_k) =  \sum^p_{t=1}(\Jb^{(t)}_{:, k}(s_k - (t-1)\delta) + \sum^{t-1}_{i=1}\Jb^{(i)}_{:,k}\delta)1_{s_k \in ((t-1)\delta, t\delta]}\,.
    \end{equation*}
    We define the smoothened version of \(\fb: [0, 1]^d \to \RR^m\) as \(\tilde{\fb}_\epsilon(\sb) = \sum_{k=1}^d\tilde{\fb}_{\epsilon, k}(s_k)\) for \(\epsilon > 0\) arbitrarily small where
    \begin{equation*}
        \tilde{\fb}_{\epsilon, k}(s_k) \coloneqq \sum^p_{t=1}\left (\Jb^{(t)}_{:, k}(s_k - (t-1)\delta) + \sum^{t-1}_{i=1}\Jb^{(i)}_{:,k}\delta \right )(\tilde{1}_\epsilon(s_k - (t-1)\delta ) - \tilde{1}_\epsilon(s_k - t\delta))\,.
    \end{equation*}
\end{definition}

\begin{lemma}

\label{lemma:smooth:grid}
Functions \(\tilde{\fb}_{\epsilon}: [0, 1]^d \to \RR^m\) defined in ~\ref{def:smooth:grid} are continuously differentiable in \([0, 1]^d\), in addition to being continuous and injective, with \(\epsilon > 0\) arbitrarily small.

\end{lemma}

\begin{proof}
We show that \(\tilde{\fb}_\epsilon: [0, 1]^d \to \RR^m\) defined in ~\ref{def:smooth:grid} is continuous, injective and continuously differentiable. \\

\textbf{Continuity of }\(\tilde{\fb}_\epsilon: [0, 1]^d \to \RR^m\)

Consider the coordinate-wise decomposition of \(\tilde{\fb}_\epsilon; \tilde{f}_{\epsilon, j}(\sb) = \sum^d_{k=1}\tilde{f}_{\epsilon, (j,k)}(s_k)\), where \(\tilde{f}_{\epsilon, (j,k)}: [0, 1] \to \RR\) is defined as \[\tilde{f}_{\epsilon, (j,k)}(s_k) \coloneqq  \sum^p_{t=1}\left (\Jb^{(t)}_{j, k}(s_k - (t-1)\delta) + \sum^{t-1}_{i=1}\Jb^{(i)}_{j,k}\delta \right )(\tilde{1}_\epsilon(s_k - (t-1)\delta ) - \tilde{1}_\epsilon(s_k - t\delta))\,. \]

We note that \((\Jb^{(t)}_{j, k}(s_k - (t-1)\delta) + \sum^{t-1}_{i=1}\Jb^{(i)}_{j,k}\delta)\) is continuous in \(s_k \in [0, 1] \; \forall  t \in [p], k \in [d]\) since this is affine. 
Moreover \(\tilde{1}: \RR \to \RR\) is continuous by definition. \(\tilde{1}_{\epsilon}(s_k - (t-1)\delta ), \tilde{1}_{\epsilon}(s_k - t\delta)\) are compositions of a continuous function with affine functions (thereby continuous), and hence are contiinuous (Theorem 4.9, ~\citep{rudin1964principles}). 

\(\tilde{f}_{\epsilon, j}: [0, 1]^d \to \RR\), being a sum of continuous functions, is continuous for all \(j \in [m]\) (Theorem 4.9, ~\citep{rudin1964principles}).

Since the coordinate functions of \(\tilde{\fb}: [0, 1]^d \to \RR^m\), \(\tilde{f}_j: [0, 1]^d \to \RR \; \forall j \in [m]\) are continuous, \(\tilde{\fb}_\epsilon\) is continuous (Theorem 4.10, ~\citep{rudin1964principles}).

\textbf{Injectivity of }\(\tilde{\fb}_\epsilon: [0, 1]^d \to \RR^m\)

We show injectivity of the coordinate-wise functions \(\tilde{\fb}_{\epsilon, k}: [0, 1] \to \RR^m \; \forall k \in [d]\) and conclude that \(\tilde{\fb}_\epsilon: [0, 1]^d \to \RR^m\) is injective by Lemma ~\ref{direct:sum:subspace:grid} and Lemma ~\ref{inj:coord:to:full}. For a particular \(k\), to show that \(\tilde{\fb}_{\epsilon, k}: [0, 1] \to \RR^m\) is injective, we need to show the following:

\begin{equation}
    \label{inj:smooth:grid}
    \forall s^{(1)}_k, s^{(2)}_k \in [0, 1] : \quad \fb_k(s^{(1)}_k) = \fb_k(s^{(2)}_k) \implies s^{(1)}_k = s^{(2)}_k
\end{equation}

As usual, we show (~\ref{inj:smooth:grid}) by contradiction. Let
\begin{equation}
    \label{contradiction:smooth:grid}
    \exists s_k^{(1)} \ne s_k^{(2)} \in [0, 1] s.t. \fb_k(s_k^{(1)}) = \fb_k(s_k^{(2)}) 
\end{equation}
Observe that \(\fb_k(s_k) \in \Sbb_k = \text{span}(\Jb^{(1)}_{:,k}, \Jb^{(2)}_{:,k}, ..., \Jb^{(p)}_{:,k})\). We define the coefficient vector for a given \(s_k \in \RR\), \[\tb: [0, 1] \to [0, 1]^p, \tb(s_k) \coloneqq \sum^p_{t=1} \left ( \sum^{t-1}_{i=1}\delta\eb_i + (s_k - (t-1)\delta)\eb_t \right )(\tilde{1}_\epsilon(s_k - (t-1)\delta ) - \tilde{1}_\epsilon(s_k - t\delta))\,, \] where \(\eb_i\) denotes the \(i\)-th canonical orthonormal basis vector in \(\RR^p\). Then we get
\begin{align*}
    \fb_k(s_k^{(1)}) &= \fb_k(s_k^{(2)})\\
    [\Jb^{(1)}_{:,k} \; \Jb^{(2)}_{:,k} \; ... \;\Jb^{(p)}_{:,k}]\tb(s_k^{(1)}) &= [\Jb^{(1)}_{:,k} \; \Jb^{(2)}_{:,k} \; ... \;\Jb^{(p)}_{:,k}]\tb(s_k^{(2)})\\
    \implies \tb(s_k^{(1)}) &= \tb(s_k^{(2)})\,, 
\end{align*}
because \(\Jb^{(1)}_{:,k}, \Jb^{(2)}_{:,k}, ..., \Jb^{(p)}_{:,k}\)  are linearly independent with probability one (Lemma ~\ref{linind:sphind}).

Observe that the number of non-zero entries in \(\tb(s_k)\) is determined by the value of \(s_k\). The number of nonzero entries in \(\tb(s_k)\) for \(s_k \in [0, 1]\) is as follows:

\begin{itemize}
    \item \(s_k \in \PP^{(1)}_{[0, 1]} \coloneqq [0, \delta - \epsilon)\): No. of nonzero entries in \(\tb(s_k) = 1\)
    \item \(s_k \in \PP^{(2)}_{[0, 1]} \coloneqq  [\delta - \epsilon, 2\delta - \epsilon)\): No. of nonzero entries in \(\tb(s_k) = 2\)\\
    \(\vdots\)
    \item \(s_k \in \PP^{(p-1)}_{[0, 1]} \coloneqq  [(p-2)\delta - \epsilon, (p-1)\delta - \epsilon)\): No. of nonzero entries in \(\tb(s_k) = p-1\)
    \item \(s_k \PP^{(p)}_{[0, 1]} \coloneqq  \in [(p-1)\delta - \epsilon, 1]\): No. of nonzero entries in \(\tb(s_k) = p\)
\end{itemize}

Therefore, \(\tb(s_k^{(1)}) = \tb(s_k^{(2)}) \implies s_k^{(1)}, s_k^{(2)} \in \PP^{(r)}_{[0, 1]}\) for \(r \in [p]\). 

\begin{enumerate}
    \item Consider \(s_k^{(1)}, s_k^{(2)} \in \PP^{(r)}_{[0, 1]}\) for \(r \in \{2, 3, ..., p\}\). 
    
    For \(s_k \in \PP^{(r)}_{[0, 1]}, r \in \{2, 3, ..., p\}, \tb(s_k) = \sum^r_{t=r-1} \left ( \sum^{t-1}_{i=1}\delta\eb_i + (s_k - (t-1)\delta)\eb_t \right )(\tilde{1}_\epsilon(s_k - (t-1)\delta ) - \tilde{1}_\epsilon(s_k - t\delta))\). Observe that the coefficient of \(\eb_r\) is equal to \(t_r(s_k) = (s_k - (r-1)\delta)\tilde{1}_\epsilon(s_k - (r-1)\delta)\). \(t_r(s_k)\) is montonic in \(s_k\).

    \(\tb(s_k^{(1)}) = \tb(s_k^{(2)}) \implies t_r(s_k^{(1)}) = t_r(s_k^{(2)})\). Thus, \(s_k^{(1)} = s_k^{(2)}\), since \(t_r(.)\) is monotonic.
    
    \item Consider \(s_k^{(1)}, s_k^{(2)} \in \PP^{(1)}_{[0, 1]}\). 
    
     For \(s_k \in \PP^{(1)}_{[0, 1]}\), \(\tb(s_k) = s_k\eb_1\). Thus, \(\tb(s_k^{(1)}) = \tb(s_k^{(2)}) \implies s_k^{(1)} = s_k^{(2)}\).
    
\end{enumerate}

Thus, we see that, \(\tb(s_k^{(1)}) = \tb(s_k^{(2)}) \implies s_k^{(1)} = s_k^{(2)}\), \(\fb_k(s_k^{(1)}) = \fb_k(s_k^{(2)}) \implies s_k^{(1)} = s_k^{(2)}\), thereby, \(\fb_k: [0, 1] \to \RR^m\) is injective \(\forall k \in [d]\).

We now show that the injectivity of the coordinate-wise functions \(\tilde{\fb}_{\epsilon, k}: [0, 1] \to \RR^d\) implies the injectivity of \(\tilde{\fb}_\epsilon: [0, 1]^d \to \RR^m\), \(\tilde{\fb}_\epsilon(\sb)= \sum^d_{k=1}\tilde{\fb}_{\epsilon, k}(s_k) \). Observe that by definition ~\ref{def:smooth:grid}

\begin{enumerate}
    \item \(\Sbb_1\) = span(Im(\(\tilde{\fb}_{\epsilon, 1}\))) = span(\(\Jb^{(1)}_{:,1}, \Jb^{(2)}_{:,1}, ..., \Jb^{(p)}_{:,1}\))
    \item \(\Sbb_2\) = span(Im(\(\tilde{\fb}_{\epsilon, 2}\))) = span(\(\Jb^{(1)}_{:,2}, \Jb^{(2)}_{:,2}, ..., \Jb^{(p)}_{:,2}\))\\
    \vdots
    \item \(\Sbb_d\) = span(Im(\(\tilde{\fb}_{\epsilon, d}\))) = span(\(\Jb^{(1)}_{:,d}, \Jb^{(2)}_{:,d}, ..., \Jb^{(p)}_{:,d}\))
\end{enumerate}

Consider \(\VV = \text{span}(\text{Im}(\tilde{\fb}_\epsilon))\). By Lemma ~\ref{direct:sum:subspace:grid}, \(\VV = \Sbb_1 \bigoplus \Sbb_2 \bigoplus ... \Sbb_d\). Further, by Lemma ~\ref{inj:coord:to:full}, injectivity of \(\tilde{\fb}_{\epsilon, k}: [0, 1]^d \to \RR^m \forall k \in [d]\) implies injectivity of \(\tilde{\fb}_\epsilon: [0, 1]^d \to \RR^m\).

\textbf{Continuity of derivatives of }\(\tilde{\fb}_\epsilon: [0, 1]^d \to \RR^m\)

Consider the derivatives of \(\tilde{\fb}_\epsilon(\sb)\) with respect to the coordinates of \(\sb = (s_1, s_2, ..., s_d)\). Since \(\tilde{\fb}_{\epsilon}(\sb)= \sum^d_{k=1}\tilde{\fb}_{\epsilon, k}(s_k)\), \(\frac{\partial \tilde{\fb}_{\epsilon}(\sb)}{\partial s_k} = \frac{d \tilde{\fb}_{\epsilon, k}(s_k)}{d s_k } = \tilde{\fb'}_{\epsilon, k}(s_k)\). By Definition ~\ref{def:smooth:grid},

\begin{align}
    \tilde{\fb}_{\epsilon, k}(s_k) &=\sum^p_{t=1}\left (\Jb^{(t)}_{:, k}(s_k - (t-1)\delta) + \sum^{t-1}_{i=1}\Jb^{(i)}_{:,k}\delta \right )(\tilde{1}_\epsilon(s_k - (t-1)\delta ) - \tilde{1}_\epsilon(s_k - t\delta)) \nonumber  \\
    \tilde{\fb'}_{\epsilon, k}(s_k) &= \sum^p_{t=1}\Jb^{(t)}_{:, k}(\tilde{1}_\epsilon(s_k - (t-1)\delta ) - \tilde{1}_\epsilon(s_k - t\delta)) \nonumber \\
    &+ \left (\Jb^{(t)}_{:, k}(s_k - (t-1)\delta) + \sum^{t-1}_{i=1}\Jb^{(i)}_{:,k}\delta \right )(\tilde{1'}_\epsilon(s_k - (t-1)\delta ) - \tilde{1'}_\epsilon(s_k - t\delta)) \label{jaco:col:grid}
\end{align}

 where by definition ~\ref{smooth:step} \(\tilde{1'}_\epsilon(s) = \begin{cases} 0 & s \le -\epsilon\\
                                          \frac{1}{2}\cos \left (\frac{\pi s}{2\epsilon} \right )\frac{\pi}{2\epsilon}   & -\epsilon < s \le \epsilon\\
                                          0 & s > \epsilon\end{cases}\). 
  Notice that \(\tilde{1'}_\epsilon: \RR \to \RR\) is continuous in \(\RR.\) \(\tilde{\fb'}_{\epsilon, k}(s_k)\) is continuous since it is composed by a sum and product of continuous functions (Theorem 4.9, ~\citep{rudin1964principles}). 
  
  To show that the derivative, \(\tilde{\fb'}_{\epsilon, k}(s_k)\), is well-defined for \(s_k \in [0, 1]\), we remark on terms containing \(\tilde{1'}_\epsilon(s_k - t\delta) \; \forall t \in [p-1]\) since \(\tilde{1'}_\epsilon(.)\) can be very large for small \(\epsilon\). Notice that the term \(\tilde{1'}_\epsilon(s_k - p\delta)\) is always equal to zero for \(s_k \in [0, 1]\) since \(p\delta > 1, \epsilon > 0\) is arbitrarily small; and \(\tilde{1'}_\epsilon(s_k - t\delta)\) is nonzero for \(t\delta - \epsilon < s_k \le t\delta + \epsilon\). The coefficient multiplied to \(\tilde{1'}_\epsilon(s_k - t\delta)\) is equal to \(\left (\Jb^{(t+1)}_{:, k}(s_k - t\delta) + \sum^{t}_{i=1}\Jb^{(i)}_{:,k}\delta - \Jb^{(t)}_{:, k}(s_k - (t-1)\delta) - \sum^{t-1}_{i=1}\Jb^{(i)}_{:,k}\delta \right ) = (\Jb^{(t+1)}_{:, k} - \Jb^{t}_{:, k})(s_k - t\delta)\). Thus for \(t\delta - \epsilon < s_k \le t\delta + \epsilon\), the term \((\Jb^{(t+1)}_{:, k} - \Jb^{t}_{:, k})(s_k - t\delta)\tilde{1'}_\epsilon(s_k - t\delta) = (\Jb^{(t+1)}_{:, k} - \Jb^{t}_{:, k})(s_k - t\delta)\frac{1}{2}\cos \left (\frac{\pi s}{2\epsilon} \right )\) is well-defined since \((s_k - t\delta) = \Theta(\epsilon)\)\footnoteref{theta-notation}.

  We have shown that \(\tilde{\fb'}_{\epsilon, k}(s_k)\) is continuous and well-defined \(\forall k \in [d]\). Thus, the derivatives \(\frac{\partial \tilde{f}_{\epsilon, i}}{\partial s_k} \forall i \in [m], k \in [d]\) are continuous for \(\epsilon > 0\) arbitrarily small. 
  
  Since all the partial derivatives of \(\tilde{\fb}_\epsilon: [0, 1]^d \to \RR^m\) are continuous, \(\tilde{\fb}\) is continuously differentiable (Theorem 9.21, ~\citep{rudin1964principles}).

 \end{proof}

We now present the theorem that introduces a bound on the global IMA contrast for non-affine maps, \(\tilde{\fb}_\epsilon: \RR^d \to \RR^m, m \gg d\), defined as a smooth approximation to grid-wise affine maps ~\ref{def:smooth:grid}. 

\begin{theorem}

\label{global:ima:thm:grid}
Consider the map \(\tilde{\fb}_\epsilon: [0, 1]^d \to \RR^m\) sampled randomly from the procedure ~\ref{def:smooth:grid}. 

Then, the map \(\tilde{\fb}_\epsilon: \RR^d \to \RR^m\), for \(\epsilon > 0\) arbitrarily small and any finite probability density, \(p_\sb\), defined over \([0, 1]^d\) satisfies the following bound on the global IMA contrast \(C_{\textsc{ima}}(\tilde{\fb}_\epsilon, p_\sb)\), \(C_{\textsc{ima}}(\tilde{\fb}_\epsilon, p_\sb) \le \delta \) with (high) probability \(\ge 1 - \min \left\{1, \exp(2log d - \kappa(m -1)\frac{\delta^2}{d^2})\right \}\) for \(m \gg d\) where \(\delta < \frac{1}{2}\) is arbitrarily small.

\end{theorem}

\begin{proof}
We show that the condition of Theorem ~\ref{local:ima:thm}, the columns of the Jacobian of \(\tilde{\fb}_\epsilon\) defined in ~\ref{def:smooth:2d} are locally sampled isotropically i.e. , is still satisfied for the domain of \(\tilde{\fb}_\epsilon\), i.e. \(\forall \sb \in [0, 1]^d\) almost surely w.r.t finite probability measure, \(p_\sb\) over \([0, 1]^d\).

Following from Definition ~\ref{map:definition:grid} and Definition ~\ref{def:smooth:grid}, consider the partition of the domain of \(\tilde{\fb}_\epsilon: [0, 1]^d \to \RR^m\), \([0, 1]^d\) into the following regions, 

\begin{enumerate}
    \item \(\II \coloneqq \{s_k - t\delta \ne (-\epsilon, \epsilon] \; \forall t \in [p-1], k \in [d]\}\)

    Notice that by Definition ~\ref{map:definition:2d} and Definition ~\ref{def:smooth:2d}, \(\forall \sb \in \II, \bb \in \{0, 1\}^d\), such that \(\bb(\sb) = \bb\), \(\Jb_{\tilde{\fb}_\epsilon}(\sb) = \Jb^{(\bb)}\). Recall that \(\bb(\sb)\) is defined such that, \(b_k(\sb) = \ceil{\frac{s_k}{\delta}}{} \; \forall k \in [d]\). The columns of \(\Jb^{(\bb)}\) are sampled independently from a spherically invariant distribution in \(\RR^m\), i.e. \(\Jb^{(\bb)}_{1}, \Jb^{(\bb)}_{2}, ..., \Jb^{(\bb)}_{d} \overset{i.i.d}{\sim} p_\rb\). Thus, the condition of Theorem ~\ref{local:ima:thm}, the columns of the Jacobian of \(\tilde{\fb}_\epsilon\) are locally sampled isotropically, is still satisfied for these regions.

    \item \(\BB \coloneqq [0, 1]^d - \II\)
    
    As in the case where the map, \(\tilde{\fb}_\epsilon: \RR^d \to \RR^m\) was defined as the smooth connection of \textit{two} affine maps (Theorem ~\ref{global:ima:thm:2}), the region \(\BB\) sandwiching the boundary of the partitions has arbitrarily small probability measure since: 
    
    \begin{enumerate}
    \item \(\BB\) is an \(\epsilon\)-sandwich of a \((d - 1)\)-dimensional region of a \(d\)-dimensional domain. The Lebesgue measure on \(\BB\) is equal to the volumne element associated with \(\BB\) (3.3, ~\citep{ccinlar2011probability}), thus, \(\lambda(\BB) = \Theta(\epsilon)\)\footnoteref{theta-notation} where \(\lambda(.)\) denotes the Lebesgue measure.
    
    \item \(p_\sb\) is finite at all points.
\end{enumerate}
    Hence, \(p(\BB) = \int_\BB p_\sb\lambda(\sb) = \Theta(\epsilon)\), is arbitrarily small for suitably chosen \(\epsilon\).
    
    Like in Theorem ~\ref{global:ima:thm:2}, to derive a bound on the global IMA contrast of \(\tilde{\fb_\epsilon}\), \(c_{\textsc{ima}}(\tilde{\fb_\epsilon}, p_\sb)\), we need that in region, \(\forall \sb \in \BB\), the value of the local IMA contrast  \(c_{\textsc{ima}}(\tilde{\fb_\epsilon}, \sb)\) is finite. This is equivalent to showing that the Jacobian, \(\Jb_{\tilde{\fb}_\epsilon}\) is full column-rank for all \(\sb \in \BB\). Consider the definition of \(\tilde{\fb}_\epsilon: \RR^d \to \RR^m\) (Definition ~\ref{def:smooth:grid}) in terms of coordinate-wise functions, \(\fb_{\epsilon, k}: \RR \to \RR^m, \forall k \in [d]\).
    
    \begin{equation*}
        \tilde{\fb}_{\epsilon, k}(s_k) \coloneqq \sum^p_{t=1}\left (\Jb^{(t)}_{:, k}(s_k - (t-1)\delta) + \sum^{t-1}_{i=1}\Jb^{(i)}_{:,k}\delta \right )(\tilde{1}_\epsilon(s_k - (t-1)\delta ) - \tilde{1}_\epsilon(s_k - t\delta))
    \end{equation*}
    
    By (\ref{jaco:col:grid}), the \(k\)-th column of \(\Jb_{\tilde{\fb}_\epsilon}\) for any \(k \in [d]\) is given by:
    
    \begin{align*}
        \Jb_{\tilde{\fb}_\epsilon, :, k}  &= \sum^p_{t=1}\Jb^{(t)}_{:, k}(\tilde{1}_\epsilon(s_k - (t-1)\delta ) - \tilde{1}_\epsilon(s_k - t\delta))  \\
    &+ \left (\Jb^{(t)}_{:, k}(s_k - (t-1)\delta) + \sum^{t-1}_{i=1}\Jb^{(i)}_{:,k}\delta \right )(\tilde{1'}_\epsilon(s_k - (t-1)\delta ) - \tilde{1'}_\epsilon(s_k - t\delta))
    \end{align*}
    
    For \(s_k \in (t - \epsilon, t + \epsilon]\), 
    
    \begin{align*}
        \Jb_{\tilde{\fb}_\epsilon, :, k}  &= \Jb^{(t)}_{:, k}(\tilde{1}_\epsilon(s_k - (t-1)\delta ) - \tilde{1}_\epsilon(s_k - t\delta)) + \Jb^{(t+1)}_{:, k}(\tilde{1}_\epsilon(s_k - t\delta ) - \tilde{1}_\epsilon(s_k - (t + 1)\delta))\\
        &+ (\Jb^{(t+1)}_{:, k} - \Jb^{t}_{:, k})(s_k - t\delta)\tilde{1'}_\epsilon(s_k - t\delta)\\
        &= \Jb^{(t)}_{:, k}(1 - \tilde{1}_\epsilon(s_k - t\delta) - (s_k - t\delta)\tilde{1'}_\epsilon(s_k - t\delta)) + \Jb^{(t+1)}_{:, k}(\tilde{1}_\epsilon(s_k - t\delta ) + (s_k - t\delta)\tilde{1'}_\epsilon(s_k - t\delta))\\
        & \quad \because \tilde{1}_\epsilon(s_k - (t-1)\delta ) = 1, \tilde{1}_\epsilon(s_k - (t + 1)\delta) = 0\\
        &= \Jb^{(t)}_{:, k}(\tilde{1}_\epsilon(t\delta - s_k) + (t\delta - s_k)\tilde{1'}_\epsilon(t\delta - s_k)) + \Jb^{(t+1)}_{:, k}(\tilde{1}_\epsilon(s_k - t\delta ) + (s_k - t\delta)\tilde{1'}_\epsilon(s_k - t\delta))\\
        & \quad \because \tilde{1}_\epsilon(s) + \tilde{1}_\epsilon(-s) = 1, \tilde{1'}_\epsilon(s) = \tilde{1'}_\epsilon(-s)
    \end{align*}
    
     Observe that \(\Jb_{\tilde{\fb}_\epsilon, :, k}\) is a linear combination of \(\Jb^{(1)}_{:, k}, \Jb^{(2)}_{:, k}, ..., \Jb^{(p)}_{:, k} \quad \forall k \in [d]\). Since by Lemma ~\ref{linind:sphind}, \(\bigcup^p_{i=1}\text{cols}(\Jb^{(i)})\) are all nonzero and linearly independent with respect to each other with probability 1, the columns of \(\Jb_{\tilde{\fb}_\epsilon}\) are linearly independent as long as they are all non-zero. Hence, the only possibility for \(\Jb_{\tilde{\fb}_\epsilon}\) to not be full column-rank is for \(k \in [d]\),

   \begin{align*}
        &\Jb_{\tilde{\fb}_\epsilon, :, k} = \mathbf{0}\\
        &\implies (\tilde{1}_\epsilon(t\delta - s_k ) + ( t\delta - s_k)\tilde{1'}_\epsilon(t\delta - s_k)) = 0, (\tilde{1}_\epsilon(s_k - t\delta ) + (s_k - t\delta)\tilde{1'}_\epsilon(s_k - t\delta)) = 0
        \\
        & \quad \because \Jb^{(t)}_{:, k}, \Jb^{(t+1)}_{:, k} \text{ are linearly independent.}
    \end{align*}
    
    As in Theorem~\ref{global:ima:thm:2d}, consider the function, \(q: \RR \to \RR\) such that \(q(s) = \tilde{1}_\epsilon(s) + s\tilde{1'}_\epsilon(s)\). Observe that \(q(s) \ge 0\) for \(s \ge 0\). Thus, for \( (\tilde{1}_\epsilon(t\delta - s_k ) + ( t\delta - s_k)\tilde{1'}_\epsilon(t\delta - s_k)) = (\tilde{1}_\epsilon(s_k - t\delta ) + (s_k - t\delta)\tilde{1'}_\epsilon(s_k - t\delta)) = 0\), we need that \(s_k = t\delta\).  At \(s_k = t\delta\), \(q(s_k - t\delta) = q(t\delta - s_k) = \frac{1}{2} \ne 0\). Hence, we have shown that \(\Jb_{\tilde{\fb}_\epsilon, :, k} \ne \mathbf{0} \forall k \in [d]\), thereby \(\Jb_{\tilde{\fb}_\epsilon}\) is full column-rank and \(c_{\textsc{ima}}(\tilde{\fb}_\epsilon, \sb)\) is finite for all \(\sb \in \BB\).
    
\end{enumerate}

Hence, 
\begin{align*}
    C_{\textsc{ima}}(\tilde{\fb}_\epsilon, p_\sb) &= \int_{\RR^d}c_{\textsc{ima}}\;p_\sb d\sb \\
    &= \int_{\II}c_{\textsc{ima}}(\tilde{\fb}_\epsilon, \sb)\;p_\sb d\sb + \int_{\BB}c_{\textsc{ima}}(\tilde{\fb}_\epsilon, \sb)\;p_\sb d\sb\\
    &= \int_{\II}c_{\textsc{ima}}(\tilde{\fb}_\epsilon, \sb)\;p_\sb d\sb + \Theta(\epsilon)\\
    &\approx \int_{\II}c_{\textsc{ima}}(\tilde{\fb}_\epsilon, \sb)\;p_\sb d\sb \quad \text{for }\epsilon\text{ arbitrarily small.}\\
    &\le \max_{\sb \in \RR^d}c_{\textsc{ima}}(\tilde{\fb}_\epsilon, \sb)\int_{\RR^d}p_\sb d\sb\\
    &\le \max_{\sb \in \RR^d}c_{\textsc{ima}}(\tilde{\fb}_\epsilon, \sb) \le \delta \,,
\end{align*}
with probability at least \(1 - \min \left\{1, \exp(2\log d - \kappa(m -1)\frac{\delta^2}{d^2})\right \}\) by Theorem~\ref{local:ima:thm}. 
Thus, \(C_{\textsc{ima}}(\tilde{\fb}_\epsilon, p_\sb) \le \delta \) for \(\tilde{\fb}_\epsilon: \RR^d \to \RR^m\) defined in \ref{def:smooth:grid} with (high) probability \(\ge 1 - \min \left\{1, \exp(2log d - \kappa(m -1)\frac{\delta^2}{d^2})\right \}\) for \(m \gg d\) where \(\delta < \frac{1}{2}\) is arbitrarily small.

\end{proof}

We have shown that for smoothened grid-wise affine maps, \(\tilde{\fb}_\epsilon: \RR^d \to \RR^m\) defined in ~\ref{def:smooth:grid} where locally the columns of the Jacobian, \(\Jb_{\tilde{\fb}_\epsilon}(\sb)\), are sampled independently from a spherically invariant distribution (statistical notion of independent influences), \(\Jb_{\tilde{\fb}_\epsilon, 1}(\sb), \Jb_{\tilde{\fb}_\epsilon, 2}(\sb), ..., \Jb_{\tilde{\fb}_\epsilon, d}(\sb) \sim p_\sb\), the IMA function class which formalizes the non-statistical notion of independent influences is "typical", i.e. the columns of \(\Jb_{\tilde{\fb}_\epsilon}(\sb)\) are close to orthogonal with high probability.

\end{document}